%% file: arxiv.tex
\documentclass{article} % For LaTeX2e
\PassOptionsToPackage{table}{xcolor}
\usepackage{iclr2027_conference,times}

\input{math_commands.tex}

\usepackage{hyperref}
\usepackage{url}

\usepackage{booktabs}
\usepackage{amsthm}
\theoremstyle{definition}

\newtheorem{theorem}{Theorem}
\usepackage{amsmath}
\usepackage{tikz}
\usetikzlibrary{positioning}
\usetikzlibrary{patterns}

\usepackage[table]{xcolor}
\usepackage[inline, shortlabels]{enumitem}
\usepackage{placeins}
\definecolor{c1}{HTML}{E57373} % 0.0 - 0.2 (Red)
\definecolor{c2}{HTML}{FFB74D} % 0.2 - 0.4 (Orange)
\definecolor{c3}{HTML}{FFF176} % 0.4 - 0.6 (Yellow)
\definecolor{c4}{HTML}{AED581} % 0.6 - 0.8 (Light Green)
\definecolor{c5}{HTML}{4CAF50} % 0.8 - 1.0 (Green)
\usepackage{amssymb}

\usepackage{pgfplots}
\pgfplotsset{width=10cm,compat=1.9,each nth point=1}

\usepackage{titletoc}
\usepackage{wrapfig}
\usepackage{subcaption}
 \usepackage{floatflt}
 
\newcommand{\benchmark}[0]{WOMBAT}
\newcommand{\real}[0]{\mathbb{R}}

\usepackage{multirow}

\usepackage{pifont}
\newcommand\greencheckmark {\textcolor[RGB]{13, 120, 24}{\ding{52}}}

\newcommand\redcross {\textcolor[RGB]{173, 21, 14}{\ding{55}}}

\title{WOMBAT: Whitebox Oracle for Molecular Benchmarking and Attribution Testing}

\author{Dominik Matuszek\(^{1, 2, 3}\) \\ 
\And
Bartosz Zieliński\(^{1, 3}\)\\
\And
Tomasz Danel\(^{4, *}\) \\
\And
Dawid Rymarczyk\(^{1, 5, *}\) \\
}

\footnotetext[1]{Jagiellonian University, Faculty of Mathematics and Computer Science;\(^{2}\)Jagiellonian
University, Doctoral School of Exact and Natural Sciences;\(^{3}\)Jagiellonian
University, Jagiellonian Center for Artificial Intelligence; \(^{4}\)Jagiellonian University, Faculty of Chemistry; \(^{5}\)Ardigen SA; \(^{*}\)Equal advisorship; \\Correspondence to: \texttt{tomasz.danel@uj.edu.pl}, \texttt{dawid.rymarczyk@uj.edu.pl}}

\iclrfinalcopy % Uncomment for camera-ready version, but NOT for submission.

\begin{document}

\maketitle
\lhead{Preprint.}

\begin{abstract}
    When a graph neural network (GNN) explainer produces an unexpected attribution on a molecule, the attribution alone cannot reveal whether the explainer has failed or the model has learned a shortcut. We introduce WOMBAT, a benchmark of 14 whitebox GNNs, each with message-passing weights set by hand to detect a specific SMARTS motif. Each model’s decision rule is known by construction, providing attribution ground truth against which explainer errors can be identified and studied. We validate the models on millions of PubChem molecules and evaluate post-hoc explainers including GNNExplainer, PGExplainer, and Integrated Gradients. Guided by our qualitative analysis, we construct a model that causes Integrated Gradients to spread attribution across the graph, even though the model reliably detects the intended motif. We release the dataset, models, and evaluation code to help researchers in the development of newer XAI tools for GNNs.

\end{abstract}

\section{Introduction}

Graph Neural Networks (GNNs) have emerged as the foundational architecture for modern cheminformatics, enabling progress in molecular property prediction, virtual screening, and the design of novel therapeutics~\citep{zhang2025graph}. As such, understanding GNN predictions is critical to ensure models rely on true chemical principles rather than spurious shortcuts~\citep{proszewska2025b, wang2022extending}. This is why post-hoc XAI methods~\citep{ying2019gnnexplainer, sundararajan2017axiomatic, luo2020parameterized} are used to identify important graph features. However, verifying their reliability is currently limited by a lack of ground truth explanations for real-world molecular data.

Two common evaluation approaches use task-defined motifs as attribution ground truth~\citep{ying2019gnnexplainer, proszewska2025b} or perturbation-based metrics such as fidelity~\citep{pope2019explainability, zheng2024towards}. The former assumes that the model uses the intended motif~\citep{comparingtogroundtruthfaber}; the latter may evaluate predictions on out-of-distribution inputs~\citep{hooker2019benchmark}.

\FloatBarrier
\begin{figure}[t]
    \centering
    \includegraphics[width=\linewidth]{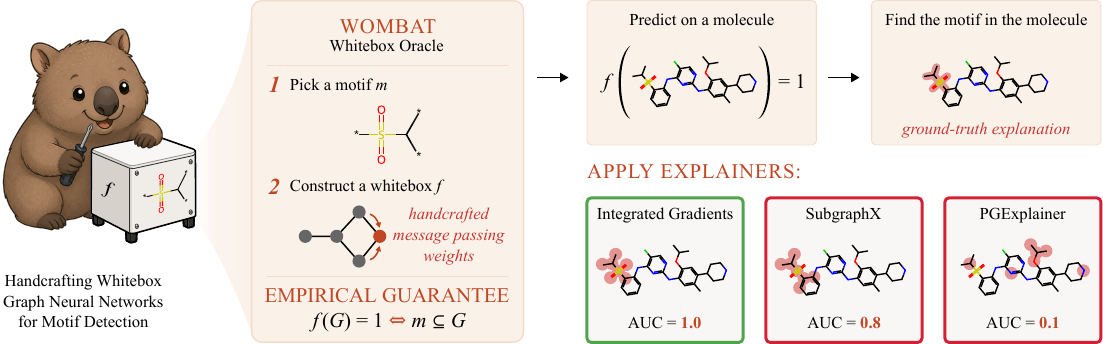}
    \caption{Overview of WOMBAT. For each motif $m$, we handcraft a whitebox GNN $f$ that detects this motif in the input graph $G$ with $f(G)\!=\!1 \iff m\subseteq G$. This yields a ground truth explanation against which post-hoc explainers can be benchmarked.}
    \label{fig:teaser}
\end{figure}

We introduce WOMBAT (Whitebox Oracle for Molecular Benchmarking and Attribution Testing) as a different way of approaching this problem (Figure~\ref{fig:teaser}). WOMBAT consists of 14 manually engineered Whitebox GNNs, each designed to detect a specific, chemically meaningful motif from a subset of the 4860 SMARTS patterns identified in ~\citep{klekota2008chemical}. By manually setting the weights of the message, aggregation, and update functions, we create whitebox models that do not require training yet achieve 100\% accuracy on many validation tasks. While we are not the first to propose whitebox GNNs, we are -- to the best of our knowledge -- the first to combine whitebox GNNs on real molecules grounded in chemistry (Table~\ref{tab:dataset_comparison}). 

\begin{table}[t]
    \centering
    \scriptsize
    \setlength{\tabcolsep}{5.2pt}
    \caption{Comparison of graph datasets and benchmarks. WOMBAT is the only one that combines designed-to-be-correct MPNNs (to eliminate confusion in explanation evaluations) with real molecular structures. 
    Abbreviations of features and benchmarks are explained in the Related Work section.
    }
    \begin{tabular}{lccccccccc}
        \toprule
        \textbf{Feature} & \textbf{MUTAG} & \textbf{Benzene} & \textbf{Flu.-Carb.} & \textbf{Alk.-Carb.} & \textbf{BA-2Mot.} & \textbf{ShapeG} & \textbf{B-XAIC} & \textbf{Tttm}  & \textbf{WOMBAT} \\
        \midrule
        Real molecules & \greencheckmark & \greencheckmark & \greencheckmark & \greencheckmark & \redcross & \redcross & \greencheckmark & \redcross & \greencheckmark \\
        Neg. Expl. eval. & \redcross & \redcross & \redcross & \redcross & \redcross & \redcross & \greencheckmark & \redcross & \greencheckmark \\
        Whitebox Models & \redcross & \redcross & \redcross & \redcross & \redcross & \redcross & \redcross & \greencheckmark & \greencheckmark \\
        No. of Tasks & 1 & 1 & 1 & 1 & 1 & 1 & 7 & 3 & 14 \\
        \bottomrule
    \end{tabular}
    \label{tab:dataset_comparison}
\end{table}

This setup provides us with a ground truth for attribution. Since we define exactly which nodes and edges are used for a prediction, discrepancies in an explainer's output cannot be attributed to shortcomings of the underlying model, such as shortcut learning. In turn, WOMBAT can be used not only to benchmark explainability methods, but -- much more importantly -- as a \textbf{diagnostic tool} for improving GNN explainability methods, since any failure can be carefully studied. By applying standard explainers to these oracles, we perform a qualitative analysis which, among other findings, allows us to precisely characterize a failure mode of Integrated Gradients (IG) in WOMBAT and to construct a network that maintains its predictive reliability while causing IG's performance to drastically decline.

Because WOMBAT relies on a deliberately constructed MPNN architecture, we do not claim that the observed failure modes must occur with the same prevalence or must occur at all in other GNN architectures. Rather, WOMBAT provides controlled examples in which such behaviours can be isolated, reproduced, and studied.

In summary, our contributions are:
\begin{itemize}
    \item We develop an MPNN architecture that supports manual weight-setting, allowing us to create 14 whitebox GNNs designed for the perfectly accurate detection of complex chemical motifs. These MPNNs are meant to be used as diagnostic tools for current and future XAI methods for GNNs.
    \item We provide a dataset of millions of molecules, sampled from PubChem with a focus on hard negatives using Tversky distance to challenge both models and explainers.
    \item We perform an evaluation of GNN explainers, identifying their specific failure modes and demonstrating sensitivity of Integrated Gradients and PGExplainer to specific weight configurations.
\end{itemize}

\section{Related Work}

\paragraph{GNN Explainability Methods.}
Post-hoc explainers attempt to identify important nodes, edges, or subgraphs responsible for a model's prediction. Optimization-based methods like GNNExplainer \citep{ying2019gnnexplainer} and its parameterized successor PGExplainer \citep{luo2020parameterized} seek to find a compact subgraph that maximizes mutual information. Axiomatic approaches such as Integrated Gradients (IG) \citep{sundararajan2017axiomatic} use path integrals. Inherently interpretable models, e.g. ProtGNN~\citep{zhang2022protgnn}, ProGReST~\citep{rymarczyk2023progrest}, and PGIB~\citep{seo2023interpretable}, identify subgraphs in the form of prototypes to make an interpretable decision.  

\paragraph{Evolution of Graph XAI Benchmarks.}

Historical benchmarks relied on small, domain-specific datasets with static ground truths. MUTAG~\citep{debnath1991structure} and Benzene~\citep{mccloskey2019using} require models to detect simple functional groups (e.g., \(\mathrm{NO}_2\) or aromatic rings). These were followed by Fluoride-Carbonyl (Flu.-Carb.) and Alkane-Carbonyl (Alk.-Carb.) \citep{sanchez2020evaluating}, which introduced multi-pattern tasks. However, these datasets are limited by their scale and the ``shortcut problem'': since the models are trained on statistical correlations, there is no guarantee they utilize the intended motifs for prediction~\citep{agarwal2023evaluating}. To address these limitations, synthetic generators such as BA-2Motifs (BA-2Mot.) and ShapeG  \citep{agarwal2023evaluating} were introduced, providing graph generation with implanted motifs. Unfortunately, they lack the chemical fidelity of real molecules. The B-XAIC benchmark~\citep{proszewska2025b} improves this by providing 50,000 molecules across seven tasks with known ground-truth rationales including negative explanations evaluations (Neg. Expl. eval.). Yet, even B-XAIC relies on trained black-box models, which maintain the model-explainer ambiguity, where a failure in attribution could still stem from a model that learned a spurious correlation rather than the explainer's inaccuracy.

\paragraph{Axiomatic Evaluation.}
A paradigm shift occurred with the Tttm (True-to-the-model) benchmark \citep{monti2024true, nerini2025true}, which introduced the axiom that an explainer must be truthful to the model's actual logic rather than the data's patterns. Tttm utilizes white-box classifiers with hard-coded logic to audit explainers. However, Tttm focuses on abstract node classification in synthetic graphs. WOMBAT combines the rigor of Tttm with the chemical realism of B-XAIC. By providing 14 manually engineered, designed to be correct MPNNs that detect SMARTS-defined motifs \citep{klekota2008chemical}, WOMBAT establishes a benchmark dedicated to XAI for GNN applications in chemistry. 

\section{WOMBAT}

\subsection{Data Preparation} 
\label{sec:data_prep}

\textbf{Selection of Motifs.} \benchmark{} evaluates GNN explainers using 14 hand-crafted whitebox GNNs, each designed to detect a specific molecular motif drawn from a comprehensive set of chemically meaningful SMARTS strings~\citep{klekota2008chemical}. We present them in Figure~\ref{fig:patterns}. The motifs were hand-selected to cover different substructure sizes; we analyse properties of the patterns in Appendix \ref{app:pattern_stats}. Moreover, we restrict our selection to acyclic motifs, due to the well-documented inability of standard Message Passing Neural Networks (MPNNs) to detect specific cyclic substructures in general. Note that this excludes aromatic systems and heterocycles from the set of patterns. This restriction is not arbitrary, however: cycle-based evaluation would likely require either a cycle-aware architecture, graph features that explicitly encode cycle information, or a separate construction. Cycle-related weaknesses of MPNNs have been discussed in the literature before~\citep{chen2020can, huang2022boosting}. This restriction applies only to the \textit{motif pool} and not to our dataset (our whiteboxes work properly even if cycles \textit{are} present in the input molecules).

\textbf{Motif adaptation.} WOMBAT consists of substructures with wildcards. We adapt the original SMARTS definitions to better align with MPNN mechanics by relaxing the strict uniqueness requirement for wildcard neighbours (Section~\ref{par:smarts_comp}). This relaxation ensures the detection task remains tractable, as strictly enforcing neighbour uniqueness is, in general, incompatible with standard message-passing operations. The disagreements between SMARTS and relaxed SMARTS are rare for the majority of patterns, with Patterns 2, 5 and 6 being the exceptions (Table~\ref{tab:app_disagreements}). Finally, for Patterns 3 and 5, we relax \texttt{[NH]} atoms to \texttt{[N;H1,H2]}, allowing one or two attached hydrogens instead of strictly one; this changes neither class labels nor ground truth atom sets across their entire validation sets.

\textbf{Ground Truths.} We label nodes in the union of all motif occurrences under the matching rules above as $1$, and all other nodes as $0$. This includes neighbours required by wildcard terms, even when these requirements are expressed as degree constraints in the relaxed SMARTS (see Section~\ref{par:smarts_comp}). Such neighbours contribute through their presence and therefore belong to the ground truth. Molecules without the motif have all-zero labels. Since our whiteboxes are constructed to implement the motif rules and empirically validated against them, our ground truths follow the \textit{true-to-the-model} view of attribution~\citep{chen2020true}.

\input{figures/patterns}

\textbf{Whitebox Validation Dataset Creation.} \label{par:validation_dataset} Before we test the explainers, we first empirically test the correctness of each of the whiteboxes. To perform this, we extracted all SMILES strings from the PubChem~\citep{kim2025pubchem} database. Then, we tested if they are parsable by RDKit~\citep{rdkit}, and removed those that failed parsing. To ensure chemical validity, we filtered out molecules violating maximum valence constraints (carbon \(>4\), nitrogen \(>4\), oxygen \(>2\), and sulphur \(>6\)). For each target motif, we then constructed a dataset by sampling up to 5 million molecules\footnote{or the maximum available, if fewer} per category:

\begin{itemize}[leftmargin=*]
    \item \textbf{Positive samples:} Randomly selected molecules containing the given motif.
    \item \textbf{Random negatives:} Randomly selected molecules lacking the motif.
    \item \textbf{Hard negatives (Tversky):} Molecules lacking the motif but sharing at least half of the pattern's active ECFP4 fingerprint bits (Tversky index~\citep{tversky1977features} $\geq 0.5$ with $\alpha=1$, $\beta=0$).\footnote{Pattern fingerprints were computed by replacing wildcards with carbon atoms and treating the resulting SMARTS as molecules.} This allows us to focus on molecules that are capable of ``mimicking'' parts of the motif's local structure.

\end{itemize}

Finally, we removed duplicated molecules from the resulting dataset. If a SMILES string was sampled as both a random and a hard (Tversky) negative, we prioritized the Tversky negative label. The distributions of distances and similarities across different patterns are detailed in Figure~\ref{fig:tversky_big}.

\input{figures/better_fig_4}

\subsection{Whitebox MPNNs Design}
\label{sec:design}

To evaluate XAI methods against a verifiable ground truth, we design a series of whitebox MPNNs. The core intuition (see Figure~\ref{fig:wombat}) behind this is to transform the standard message-passing network into a logical circuit. \textbf{We specify the weight matrices and biases manually}. First, we define every non-hydrogen atom as a graph \textit{node}. Its input representation is a one-hot encoded vector of length \(100\) representing the atom's atomic number. Then, we append a one-hot encoded vector of length \(10\), representing the number of hydrogens attached to the atom. We define every bond in the molecule as a graph \textit{edge}. Its input representation is a one-hot encoded vector of length \(4\), representing the bond's type (single, double, triple, aromatic). The message $m_{ij}$ passed from node $v_i$ to node $v_j$ using MPNN layers with weight matrices $W_{\mathrm{from}}$, $W_{\mathrm{to}}$, and $W_{e}$ is defined as:

\begin{equation}
\label{eq:message}
m_{ij} = \mathrm{ReLU}(x_i^T W_{\mathrm{from}}) \odot \mathrm{ReLU}(x_j^T W_{\mathrm{to}}) \odot \mathrm{ReLU}(e_{ij}^T W_{e})
\end{equation}

where $x_i$, $x_j$, and $e_{ij}$ represent the features of nodes $v_i$, $v_j$, and their connecting edge, respectively. Note that $\odot$ denotes element-wise multiplication. To update the node representations, we aggregate the incoming messages via summation, add a bias vector $b$, and apply a ReLU non-linearity. The updated representation $z_i$ for node $v_i$ in its neighbourhood $\mathcal{N}(v_i)$ is thus computed as:

\begin{equation}
\label{eq:z}
z_{i} = \mathrm{ReLU}\left(b + \sum_{v_j \in \mathcal{N}(v_i)} m_{ji}\right).
\end{equation}

\textbf{MPNN Readout.} Let $V$ be the set of nodes and let $x_i \in \mathbb{R}^d$ denote the activation vector of node $v_i$ after the final MPNN layer, where $d$ is the output dimensionality. We design the graph-level readout function, $r$, to act as a structural indicator that yields a positive value if and only if there exists at least one node where all $d$ features are simultaneously ``activated'' (i.e., strictly positive). We propose two alternative readout formulations to achieve this:

\begin{equation}
\label{eqn:readout_1}
r_1 = \max_{v_i \in V} \left( \prod_{j=1}^{d} \mathrm{ReLU} (x_{i,j}) \right)
\end{equation}
\begin{equation}
\label{eqn:readout_2}
r_2 = \mathrm{ReLU}\left(1 - \prod_{v_i \in V} \mathrm{ReLU}\left(1 - \prod_{j=1}^{d} \mathrm{ReLU}(x_{i,j})\right)\right)
\end{equation}

where $x_{i,j}$ represents the $j$-th feature of node $v_i$. The underlying intuition for both equations is identical: the inner product acts as a logical AND over the $d$ node features, while the outer operation acts as a logical OR across all nodes in the graph. While the first equation is simpler, the $\max$ operator negatively impacts gradient routing (negatively influencing gradient-based explainers). Furthermore, it can unfairly confuse explainers, as its output values are not bounded by 1 and can be arbitrarily large. The second equation addresses this by utilizing De Morgan's laws applied to continuous values, simulating the global OR without relying on $\max$ pooling. %A step-by-step illustration of these readout functions is provided in Figure \ref{fig:readouts}.

\textbf{Experimental Validation.} We evaluate the accuracy of our constructed whitebox MPNNs with readouts \(r_1\) and \(r_2\) using the validation dataset (detailed in Section \ref{par:validation_dataset}). Across millions of tested molecules, the models achieved perfect classification performance\footnote{, Zero false positives and zero false negatives at a prediction threshold of \(0.5\). For every motif evaluated, the models were tested against a set of strictly more than 5 million negative samples (see Table \ref{tab:ds_sizes}).} 

\textbf{Masking semantics.} In WOMBAT, nodes can be effectively masked by setting their features to 0. That is because (as long as all biases are not positive\footnote{None of our main whiteboxes have positive biases. We demonstrate one whitebox with this property in Appendix \ref{sec:appendix_impls}, but it's not a part of our main testing suite.}) a node with its features set to zero can, effectively, neither receive nor send any messages due to the message definition -- meaning the calculation is equivalent to the calculation as if it was not present altogether. This property is useful for ensuring fair evaluation of attribution methods that rely on a predefined baseline, such as Integrated Gradients~\citep{sundararajan2017axiomatic}, as the trivial zero-baseline relying on setting all features to vectors of zeroes is guaranteed to result in the model predicting 0.

\subsection{Evaluation of Attribution Methods}

\textbf{Dataset \& metrics.} We construct the XAI evaluation dataset by randomly sampling up to 10,000 positive and 10,000 negative molecules from the whitebox validation set (Section \ref{par:validation_dataset}), using all available molecules if a pool size is below this threshold. The negative subset is stratified: we first sample up to 5,000 hard negatives from the Tversky pool, and then draw from the random negative pool until the 10,000-molecule target is reached. Because our whitebox models are manually constructed and detect specific motifs, the exact ground-truth subgraphs can be directly extracted using their corresponding SMARTS strings. The faithfulness of node-level attributions is assessed using metrics dependent on the presence of the target motif:
\begin{itemize}[leftmargin=*]
    \item \textbf{Positive Samples:} We measure the explainer's ability to prioritize motif nodes over non-motif nodes per molecule using the AUROC and Average Precision (AP). We report the macro-average of these metrics across all positive graphs.\footnote{Molecules composed entirely of the target motif are excluded, as AUROC is undefined in the absence of negative (non-motif) nodes.}  We discuss properties of AUROC in relation to Fidelity in Appendix \ref{sec:fidelity_obs}. 
    \item \textbf{Negative Samples:} Following~\citet{proszewska2025b}, we evaluate attribution uniformity on graphs lacking the motif. We compute the interquartile range ($\mathrm{IQR} = Q_3 - Q_1$) of the node attribution scores. An explainer successfully processes a negative graph if no node receives an outlier attribution score—defined as falling outside the interval $[Q_1 - 1.5 \cdot \mathrm{IQR}, Q_3 + 1.5 \cdot \mathrm{IQR}]$. We report the aggregate success rate across all negative graphs.
\end{itemize}

\textbf{Edge-to-Node Attribution Conversion.} \label{par:edge-to-node} While our benchmark evaluates node-level attributions, certain methods (e.g., PGExplainer~\citep{luo2020parameterized}) exclusively generate edge-level importance scores. Because any edge belonging to a motif strictly implies that its incident nodes also belong to that motif, we project edge scores onto the nodes. Specifically, if native node-level explanations are unavailable, we compute a node's attribution by summing the attribution scores of all its incident edges.

\section{Experiments and Results}
\input{figures/example_attributions}

\input{tables/perfection_table_pos_combined}
\subsection{Quantitative results}
Because the WOMBAT benchmark guarantees that any deviation from the ground truth comes from the explainer's attribution error, we focus on whether an explainer achieved a perfect AUROC score for a given whitebox in Table~\ref{tab:model_eval_perfection_combined}. We also summarize the attribution performance of each explainer across all whitebox models by reporting the mean AUROC and mean Average Precision (AP) in Tables~\ref{tab:model_eval_auroc_continuous_custom_palette} and~\ref{tab:model_eval_ap_continuous_custom_palette}, respectively. We share details of the experimental setup in Appendix~\ref{sec:experimental_setup}.

We observe that gradient-based methods (such as IG, Input $\times$ Gradient, and Saliency) and SHAP Sampling significantly outperform the perturbation-based GNNExplainer and PGExplainer, which largely fail to exceed random baseline performance across most motifs. Second, despite the high average metrics of top-performing explainers, Table~\ref{tab:model_eval_perfection_combined} highlights that achieving strictly perfect explanations remains challenging, with only IG under \(r_2\) demonstrating semi-consistent perfection. Finally, all evaluated explainability methods, with the exception of PGExplainer, exhibited ideal performance on negative samples (i.e., molecules lacking the target motif). PGExplainer failed this evaluation in exactly one instance out of the 140,000 examples tested, a behavior observed in the case of both readout functions. This isolated failure occurred for Pattern 12, on which PGExplainer does not collapse to \(\mathrm{AUROC} = 0.5\). The failure is caused by PGExplainer learning a heuristic to detect Pattern 12; this heuristic fails for a very specific class of edge cases, causing misattributions. We illustrate these edge cases in Figure~\ref{fig:failures_of_pge_on_pattern_12} and discuss their causes in Appendix~\ref{sec:appendix_pge_fail_p12}. 

\subsection{Failure Modes of Explainers}

In this part of the work we present failure modes for Gradient-based methods and PGExplainer, while the analysis for GNNExplainer and Shapley Value Sampling is provided in the Appendix Section \ref{sec:app_pge_gnn}. Representative attributions across explainers are shown in Figure~\ref{fig:attrs}.

\paragraph{Gradient-based methods.}

\input{figures/pattern_5_misattrib}
\input{figures/pattern_5_ig}

Saliency and Input $\times$ Gradient perform well on the max-pooling readout ($r_1$). However, $r_1$ restricts gradient routing: when a molecule contains multiple motif instances, gradients are directed to only one substructure, leaving others without attribution. Under the $r_2$ readout, both methods collapse, a direct consequence of violating the sensitivity axiom~\citep{sundararajan2017axiomatic}. As a result, IG demonstrates an improvement under $r_2$.

\paragraph{Integrated Gradients.} Under the \(r_2\) readout, there are two molecules for which IG is unfaithful to the model recognising Pattern~5. We demonstrate one of these cases in the Fig.~\ref{fig:pattern_5_fail}. Here, a non-motif carbon receives a higher attribution than some atoms belonging to the pattern. However, it has no effect on the final prediction under \(r_2\). Interpolating its features from the zero baseline leaves the \(r_2\) prediction unchanged throughout the entire integration path (Fig. \ref{subfig:pattern5-ig-main-READOUTS}). Nevertheless, the atom increases the network's internal activations, as evidenced by its effect on the intermediate \(r_1\) readout. Consequently, there exists a region along the integration path, near the model's decision threshold, where further scaling this atom's representation increases the network's activation and yields positive local gradients (Fig. \ref{subfig:pattern5-ig-main-PATH}, Fig. \ref{subfig:pattern5-ig-main-IMPACT}). Since Integrated Gradients accumulates gradients along the integration path, these transient positive gradients result in a positive attribution for the non-motif atom. Similar issues happen for Pattern 4, where IG yields imperfect scores frequently (failure rate of approximately \(15\%\)). We present our full analysis in Appendix \ref{sec:app_ig_p_4_5}.

\input{figures/leaky_ig}

\paragraph{PGExplainer.} PGExplainer's failure can be partially attributed to the fact that it bases edge attributions based on final activations of the edge's terminals. In WOMBAT, many nodes (including pattern nodes) may not be activated in the final layer, significantly capping potential PGExplainer's performance in case of the larger substructures (such as Pattern 3 or Pattern 4) and causing PGExplainer to learn a heuristic in case of the Pattern 12. We provide more details in Appendix~\ref{sec:app_pge_gnn}.

\section{Discussion}

WOMBAT isolates explainer failures on models whose decision rules are known. It establishes that a failure \textit{can} occur and helps explain \textit{why}; however, its \textit{frequency} in trained GNNs remains an open question. Changing the architecture (or even weights) can impact the quality of explanations. We demonstrate this in Appendix~\ref{sec:appendix_gineconv}, where we create a more ``idiomatic'' MPNN (using GINEConv architecture~\citep{hu2019strategies}) and in Appendix~\ref{sec:appendix_impls}, where we show that different implementations of whiteboxes for Pattern 12 can yield wildly different explanations, including significantly worsening IG's performance by ``blurring'' its attribution. WOMBAT therefore complements benchmarks that use trained models and focus on quantitative evaluation, such as~\citet{proszewska2025b, sanchez2020evaluating}. 

The collapse of Saliency and Input × Gradient under $r_2$ illustrates the known limitation of relying on local gradients, which can vanish despite the prediction depending on the input \citep{sundararajan2017axiomatic}. Integrated Gradients (IG) generally performs well on WOMBAT, but still can attribute non-motif atoms as well; this is caused by transient gradients along its interpolation path. While the existence of path artifacts is known~\citep{miglani2020investigating, sanyal2021discretized, enguehard2023sequential, kapishnikov2021guided}, WOMBAT lets us trace their causes through interpretable weights. We hypothesise that they can be stronger in MPNNs whose message aggregation admits more non-motif contributions than WOMBAT's gated messages. We create an adversarial whitebox (Appendix \ref{sec:appendix_impls}) which shows that additional propagation of redundant messages can amplify them (see Figure \ref{fig:leaky_ig}), but does not establish their prevalence in trained models.

PGExplainer attributes edges based on their endpoints' final representations. In WOMBAT, many motif and non-motif nodes have final representations equal to zero, making some relevant and irrelevant edges indistinguishable to the PGExplainer. We hypothesise that oversmoothing~\citep{rusch2023survey} can worsen this problem when it removes distinctions needed for explanation. We furthermore demonstrate changes of weights for Pattern 12 that ``help'' PGExplainer learn to explain networks properly and changes that impede it from doing so in Appendix~\ref{sec:appendix_impls}.

\FloatBarrier

\section{Conclusions}
We introduced WOMBAT, a benchmark of 14 hand-crafted whitebox GNNs for acyclic molecular motif detection, validated on millions of PubChem molecules. Their known decision rules allow attribution errors to be isolated and studied without uncertainty about the model’s behaviour. Evaluating seven XAI methods, we traced failures to specific computations and used those findings to construct whiteboxes that amplify them. We release the data, models and code to support further diagnostic studies and the development of better graph explainers. WOMBAT makes even rare failures useful: understanding \textit{why} an explanation fails provides a concrete target for improving it.

\subsection*{Acknowledgements}

This work was funded by ``Interpretable and Interactive Multimodal Retrieval in Drug Discovery'' project. The ``Interpretable and Interactive Multimodal Retrieval in Drug Discovery'' project (FENG.02.02-IP.05-0040/23) is carried out within the First Team programme of the Foundation for Polish Science co-financed by the European Union under the European Funds for Smart Economy 2021-2027 (FENG). The work of T.D. was carried out within the ``AI-Based Virtual Screening Assistant'' project (LIDER15/0033/2024) funded by the National Center for Research and Development (Poland) under the LIDER XV program. We gratefully acknowledge Polish high-performance computing infrastructure PLGrid (HPC Center: ACK Cyfronet AGH) for providing computer facilities and support within computational grant no. PLG/2026/019504.

\subsection*{AI use statement}

In this work, we used generative AI tools to help rephrase or draft parts of the manuscript and to provide feedback on unfinished versions, including identifying potential errors in mathematical arguments and proofs. We also used generative AI tools to provide feedback on experimental methodology and the interpretation and presentation of results. Additionally, generative AI tools were used to assist with writing and editing source code, including TikZ code for illustrations. Furthermore, AI systems were used to assist with information and literature searches. We have not used generative AI tools for unsupervised or agentic code generation, and all whiteboxes in the main testing suite had all their weights set by a human\footnote{The only whitebox with LLM generated weights was created as a test to check if LLMs are capable of creating WOMBAT whiteboxes. It was not used to test XAI methods, but is available and clearly marked in the source code.}. We have not used LLMs for hypothesis generation. We have reviewed AI-assisted work: autocomplete suggestions were immediately verified by the author responsible for source code, while LLM-generated alerts about potential manuscript errors were reviewed and fixed by humans, if appropriate. We take responsibility for the final content of this work, including text, claims or artifacts produced with the aid of generative AI.

\subsection*{Reproducibility Statement.}
We provide open-source code (\url{https://github.com/gmum/wombat/}) for the entire evaluation pipeline, including all whitebox model weights, under the MIT license. The full WOMBAT validation dataset is also publicly available (\url{https://huggingface.co/datasets/dmtsh/wombat-smiles}) under the CC0 1.0 license.

\bibliography{iclr2027_conference}
\bibliographystyle{iclr2027_conference}

\appendix

%%%%%%%%%%%%%%%%%%%%%%%%%%%%%%%%%%%%%%%%%%%%%%%%%%%%%%%%%%%%
\newpage
\appendix

\setcounter{figure}{0}
\renewcommand{\thefigure}{S\arabic{figure}}
\setcounter{table}{0}
\renewcommand{\thetable}{S\arabic{table}}
\onecolumn

\section*{Appendix Table of Contents}
\startcontents[appendix]
\printcontents[appendix]{l}{1}{\setcounter{tocdepth}{2}}
\vspace{10pt}
\hrule
\vspace{20pt}

\section{Limitations}

While our whitebox models have been validated empirically across millions of molecules without a single false positive or false negative, we have not formally verified them using mathematical proof assistants (e.g., Lean). However, we consider our large-scale empirical validation sufficient for the purposes of this benchmark. Furthermore, we \textbf{strongly caution against using WOMBAT primarily as a leaderboard} or to declare that ``Explainer X is better than Explainer Y''. Because WOMBAT guarantees that any attribution discrepancies cannot be explained by unknown model behavior such as shortcut learning, \textbf{the benchmark is best utilized for case-by-case diagnostic studies}. For instance, even if a method such as Integrated Gradients achieves near-perfect attributions for 99.99\% of compounds under a specific readout, WOMBAT isolates the remaining 0.01\% failure cases, allowing researchers to study the exact structural or mathematical conditions that cause the explainer to break down. Furthermore, WOMBAT does not and cannot make any claim about \textit{prevalence} of any identified failure modes in more idiomatic GNNs and rather shows that a weakness is present. That weakness then needs to be judged on case-by-case basis as to its probable severity in more idiomatic GNNs. Finally, our evaluation has focused on whether XAI methods properly rank nodes: while this is critical for an XAI method to work well, other considerations (such as ease of setting appropriate threshold) exist in downstream settings; we discuss these limitations further in Appendix~\ref{sec:limit-auroc}.

\section{Experimental setup}
\label{sec:experimental_setup}

\textbf{PGExplainer and GNNExplainer.} We use PGExplainer's and GNNExplainer's implementations from PyTorch Geometric library~\citep{Fey/etal/2025}, adopting default hyperparameters for qualitative analysis. PGExplainer is trained on all molecules, randomly shuffled each epoch, using the model's predictions as targets. Because PGExplainer yields only edge-level attributions, we map these to node-level importance following the procedure in Section~\ref{par:edge-to-node}. For both explainers, we set \texttt{mode="regression"} and \texttt{return\_type="raw"} to accommodate the unbounded outputs of the $r_1$ readout.

\textbf{Gradient-based explainers \& Shapley Value Sampling.} Using Captum~\citep{kokhlikyan2020captum}, we compute node and edge attributions via Integrated Gradients, Input $\times$ Gradient, Saliency and Shapley Value Sampling. We evaluate only node-level results, aggregated by summing the absolute values of each node's 110-dimensional feature vector. We use default settings, with the exception of SVS (raising $n_{\mathrm{samples}}$ to 100). Due to high computational complexity, evaluation is restricted to a maximum of 100 molecules per class (positive and negative).

\textbf{SubgraphX.} Using the DiG library~\citep{JMLR:v22:21-0343}, we apply SubgraphX with default settings. Since SubgraphX identifies coalitions rather than individual node scores, we assign binary attributions: 1 if a node is in the coalition and 0 otherwise. We omit results for negative explanations, as SubgraphX provides coalitions by design. Due to high computational complexity, evaluation is restricted to a maximum of 100 molecules.

\section{AUROC-Fidelity relationship}
\label{sec:fidelity_obs}

\subsection{Preliminaries}

\paragraph{Assumptions.} In this section, we will focus on readout \(r_2\), which has better theoretical properties and does not unfairly confuse explainers like the \(r_1\) readout. Unless otherwise noted, in this section we assume all models are using \(r_2\) readout. Furthermore, it can be observed that \(r_2\) readout effectively binarizes network's output. This was also checked during the empirical validation, where there was no molecule for which a model using \(r_2\) readout would return logit not being equal to \(1\) or \(0\)\footnote{In practice, such situation happens only if the input features are scaled (i.e. the input features are not one-hot encoded, but are scaled one-hot encoding). In this section, however, we are working with Fidelity metric, which does not \textit{scale} input features, but completely masks them.}. 

\paragraph{Whiteboxes.} Whitebox models in WOMBAT were designed to be correct, which was later confirmed by extensive empirical evaluation on known molecules. In other words, we have empirically validated that for a graph \(G\), satisfying our predefined valence rules, and motif \(m\), whitebox model \(f_m\) (under the readout \(r_2\)) satisfies the following property:

\begin{equation}
    \label{eqn:reliability}
    f_m(G) = 1 \iff m \subseteq G
\end{equation}

\paragraph{Ground truth.} As stated in Section~\ref{sec:data_prep}, WOMBAT's ground truth is the union of the node sets of all motif occurrences. This includes nodes matched by wildcard terms. Where wildcard terms are replaced by degree constraints in the relaxed SMARTS, we include all neighbours of the matched atoms subject to those constraints, since their presence is required to satisfy them.

\paragraph{Fidelity.} Let \(G_s \subseteq G\) be a subgraph of \(G\) that constitutes an explanation of \(f_m\). Let \(G \setminus G_s\) be complement of \(G_s\). Then, we define \textit{fidelity} of \(G_s\) as follows~\citep{yuan2022explainability}:

\begin{equation}
    \label{eqn:fid_plus}
\mathrm{Fidelity}^{+\mathrm{prob}}(G, G_s) = f_m(G) - f_m(G\setminus G_s)
\end{equation}

which intuitively measures whether removal of \(G_s\) from the input changes model's prediction; the greater the \(\mathrm{Fidelity}^+\), the better \(G_s\) serves as model's explanation. Alternatively, fidelity can be defined as follows: 

\begin{equation}
    \label{eqn:fid_minus}
\mathrm{Fidelity}^{-\mathrm{prob}}(G, G_s) = f_m(G) - f_m(G_s)
\end{equation}

which intuitively measures whether removal of all nodes \textit{not} belonging to \(G_s\) changes model's prediction. In that case, the lower the \(\mathrm{Fidelity}^-\), the better \(G_s\) serves as model's explanation.

\paragraph{Sparsity.} Let \(G_s \subseteq G\) be a subgraph of \(G\) that constitutes an explanation of \(f_m\). Let \(G \setminus G_s\) be complement of \(G_s\). Then, we define \textit{sparsity} of \(G_s\) as follows~\citep{yuan2022explainability}:

\begin{equation}
    \label{eqn:sparsity}
\mathrm{Sparsity}(G, G_s) = 1 - \frac{|V_{G_s}|}{|V_{G}|}
\end{equation}

\subsection{AUROC-Fidelity relationship}

\begin{theorem}
    \label{th:gt}
    Let \(f_m\) be a WOMBAT whitebox model with readout \(r_2\). Let \(G=(V,E)\) be a graph, such that \(m \subsetneq G\). Let \(G_\mathrm{GT} \subsetneq G\) be ground truth explanation of \(f_m\) on \(G\). Then, \(G_{\mathrm{GT}}\) maximises \(\mathrm{Fidelity}^{\mathrm{+prob}}\) and minimises \(\mathrm{Fidelity}^{\mathrm{-prob}}\).
\end{theorem}
\begin{proof}
     Recall that: 

    \begin{equation}
            f_m(G) = 1 \iff m \subseteq G
    \end{equation}
    
    Since \(G_{\mathrm{GT}}\) is union of all occurrences of motif \(m\) in \(G\): 
    
    \begin{equation}
        \mathrm{Fidelity}^{+\mathrm{prob}}(G, G_\mathrm{GT}) = f_m(G) - f_m(G \setminus G_{\mathrm{GT}}) = 1 - 0 = 1
    \end{equation}

    and 

    \begin{equation}
        \mathrm{Fidelity}^{-\mathrm{prob}}(G, G_\mathrm{GT}) = f_m(G) - f_m(G_{\mathrm{GT}}) = 1 - 1 = 0
    \end{equation}

    Furthermore, note that since we are using readout \(r_2\), it follows that for any \(G\), \(f_m(G) \in \{0, 1\}\) (as long as the input features are not scaled, which we assume to be the case). This trivially implies that the maximum of \(\mathrm{Fidelity}^{+\mathrm{prob}}\) is 1 and minimum \(\mathrm{Fidelity}^{-\mathrm{prob}}\) is 0. In other words, ground truth explanation graph maximises/minimises relevant fidelities. 
\end{proof}

\begin{theorem}
    \label{th:auroc}
    Let \(f_m\) be a WOMBAT whitebox model with readout \(r_2\). Let \(G=(V,E)\) be a graph, such that \(m \subsetneq G\). Let \(G_\mathrm{GT} \subsetneq G\) be ground truth explanation of \(f_m\) on \(G\) and let \(a_1, a_2, \dots, a_n\) be ground truth attributions of nodes \(v_1, v_2, \dots, v_n\), respectively, where \(a_i = 1 \iff v_i \in G_{\mathrm{GT}}\) and \(a_i = 0 \iff v_i \not\in G_{\mathrm{GT}}\). Then let \(\hat{a_1}, \hat{a_2}, \dots, \hat{a_n}\) be attributions generated by some attribution method. Then, if:
    \[
        \mathrm{AUROC}([\hat{a_1}, \dots, \hat{a_n}] \ ; \ [a_1, \dots, a_n]) = 1
    \]
    there exists \(t \in \real \) such that there exists \(V_s = \{v_i: \hat{a_i} \geq t\} \) for which exists \(G_s = G[V_s]\), such that \(G_s\) maximises \(\mathrm{Fid}^{+\mathrm{prob}}\) and minimises \(\mathrm{Fid}^{-\mathrm{prob}}\).
\end{theorem}

\begin{proof}
    By definition of AUROC, if \(\mathrm{AUROC}([\hat{a_1}, \dots, \hat{a_n}] \ ; \ [a_1, \dots, a_n]) = 1\), then:

    \[
    \max_{i: \ a_i=0} \hat{a_i} < \min_{i: \ a_i=1} \hat{a_i}
    \]

    Hence, there exists \(t \in \real\) such that:

    \[t \in \left(\max_{i: \ a_i=0} \hat{a_i} \ ;  \ \min_{i: \ a_i=1} \hat{a_i}\right]  \]

    Then we can define \(V_s\):

    \[
    V_s = \{v_i : \hat{a_i} \geq t \}
    \]

    and observe that \(v_i \in V_s \iff \hat{a_i} \geq t \iff a_i = 1 \iff v_i \in G_{\mathrm{GT}}\). Hence, \(G[V_s] = G_{\mathrm{GT}}\), which maximises \(\mathrm{Fid}^{+\mathrm{prob}}\) and minimises \(\mathrm{Fid}^{-\mathrm{prob}}\) (Theorem \ref{th:gt}).
\end{proof}

\begin{theorem}
    \label{th:sparsity}
    Let \(f_m\) be a WOMBAT whitebox model with readout \(r_2\).  Let \(G_\mathrm{GT} \subsetneq G\) be ground truth explanation of \(f_m\) on \(G\). Let there only be one occurrence of motif \(m\) in \(G\). Then, \(G_{GT}\) maximises sparsity among all possible explanations that minimise \(\mathrm{Fidelity}^{-\mathrm{prob}}\). 
\end{theorem}
\begin{proof}
    Note that under assumption that there is only one motif \(m\) present in \(G\), it follows that \(G_{\mathrm{GT}}= m\). Assume that there exists some \(G'\) such that: 
    \[
            \mathrm{Fid}^{-\mathrm{prob}}(G, G') = 0
    \]
    and
    \[
        \mathrm{Sparsity}(G, G') > \mathrm{Sparsity}(G, G_{\mathrm{GT}})\]

    Since \(\mathrm{Sparsity}(G, G') > \mathrm{Sparsity}(G, G_{\mathrm{GT}})\), it follows that \(|V_{G'}| < |V_{G_{\mathrm{GT}}}|\). Hence, there exists a node \(v \in G\) such that \(v\in G_{\mathrm{GT} } = m\) but \(v \not\in G'\). However:

    \[
    \exists_{v \in m} \; v \not\in G' \implies m \not\subseteq G' \implies f_m(G') = 0
    \]
    
   This implies that \(\mathrm{Fid}^{-\mathrm{prob}}(G, G') = f_m(G) - f_m(G') = 1 - 0 = 1\), leading to a contradiction.
\end{proof}

Note that Theorem \ref{th:sparsity} does \textit{not} necessarily hold in the multi-motif case, as attributing only one of the motifs could be more sparse than attribution created by the ground truth. Attributing only one motif, would, however, not maximise \(\mathrm{Fid}^{+\mathrm{prob}}\): optimal attribution (sparsity-wise) that would maximise/minimise relevant fidelities would likely need to contain one ``full'' motif and attribute singular nodes belonging to other occurrences of \(m\) in the graph. We believe, however, such attribution to be unintuitive.

\subsection{Limitations of AUROC-based evaluation}
\label{sec:limit-auroc}

Our evaluation has focused on whether the explainer ranks nodes correctly, as we believe incorrect ranking to be of fundamental significance. That said, achieving \(\mathrm{AUROC}=1.0\) by itself is not enough for an explainer to be regarded as \textit{perfect}. While \(\mathrm{AUROC}=1\) guarantees that there exists a threshold that perfectly separates ground-truth from non-ground-truth nodes, it does not quantify calibration of attribution scores or margins between them. In downstream applications the ground truth is unknown, making it harder to identify an ``importance'' threshold. Therefore, the results in this section should be interpreted as statements about the existence of fidelity-optimal thresholded masks, rather than about the quality of attribution magnitudes.

% \newpage
\input{pattern_4}

\FloatBarrier
\input{appendix_pge_gnn}
\FloatBarrier
\input{appendix_architectures}
\FloatBarrier
\input{appendix_additional_results}
\FloatBarrier
\input{appendix_additional_info}

\FloatBarrier

\end{document}

%% file: math_commands.tex
\usepackage{amsmath,amsfonts,bm}

\def\eqref#1{equation~\ref{#1}}
\def\1{\bm{1}}

\DeclareMathAlphabet{\mathsfit}{\encodingdefault}{\sfdefault}{m}{sl}
\SetMathAlphabet{\mathsfit}{bold}{\encodingdefault}{\sfdefault}{bx}{n}

%% file: figures/patterns.tex
\begin{figure}
    \centering
    \includegraphics[width=0.99\linewidth]{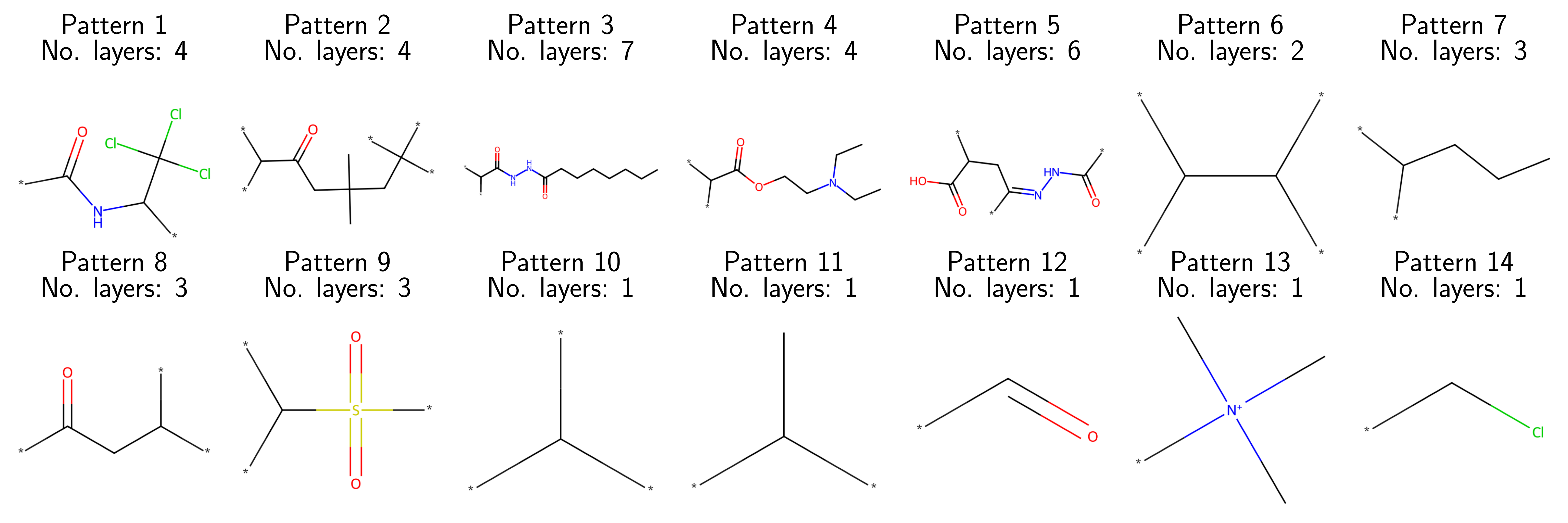}
    \caption{WOMBAT introduces 14 whitebox GNN models for detection of 14 distinct, chemically meaningful motifs. The biggest model, detecting Pattern 3, uses 7 MPNN layers before readout.}
    \label{fig:patterns}
\end{figure}

%% file: figures/better_fig_4.tex
\begin{figure}[t]
    \centering
    \includegraphics[width=\linewidth]{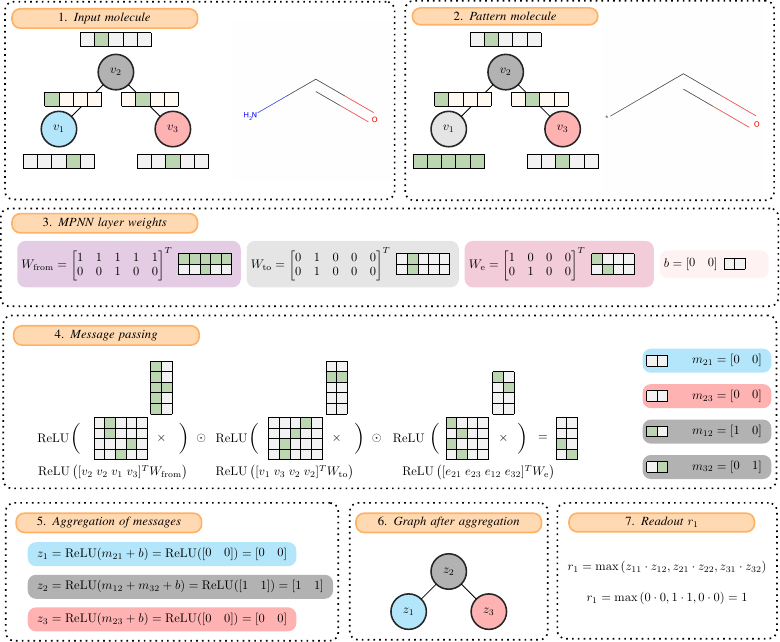}

    \caption{Schematic of the whitebox MPNN architecture and molecular computation flow. The process begins by defining the message-passing weights $W_{\text{from}}$, $W_{\text{to}}$, and $W_{e}$; notably, weights corresponding to wildcard tokens in the target motif are initialized to 1. Input graphs are represented via one-hot encoding. Message propagation from node $v_i$ to $v_j$ is computed according to Eq.~\ref{eq:message}, followed by the update of node representations $z_i$ via Eq.~\ref{eq:z} (here, the bias term $b=[0,0]$). The final prediction, in this case, a value of 1, is obtained through the readout operation defined in Eq.~\ref{eqn:readout_1}. Note that, we only present one type of readout, and the colorcoding denotes the destination of incoming messages: cyan, gray, and red correspond to nodes $v_1$, $v_2$, and $v_3$, respectively.}
    \label{fig:wombat}
    \vspace{-1em}
\end{figure}

%% file: figures/example_attributions.tex
\begin{figure}[t]
    \begin{subfigure}[t]{0.5\linewidth}
        \centering
        \includegraphics[width=0.7\linewidth]{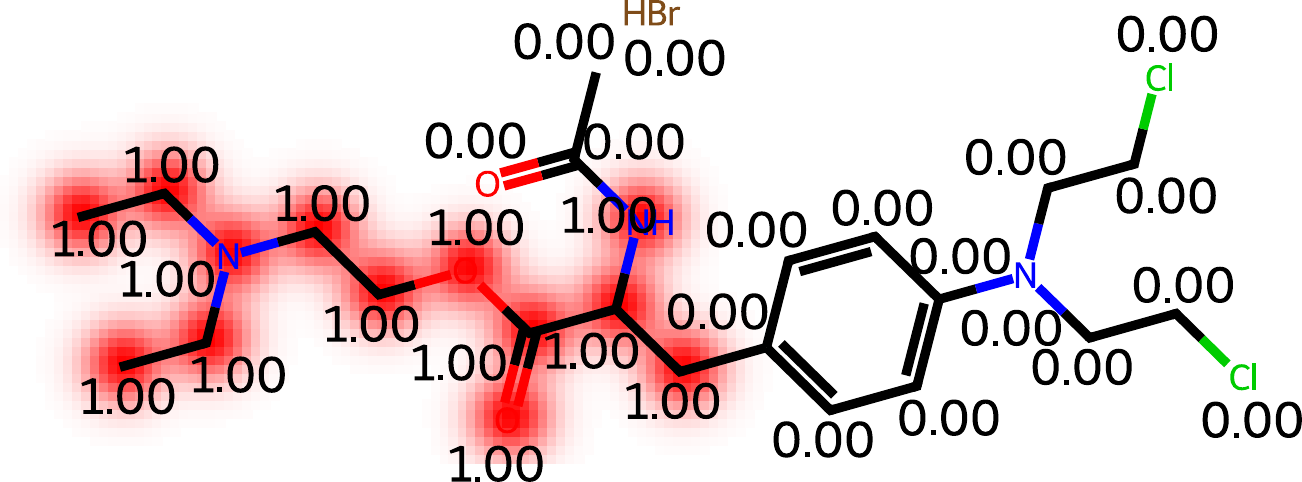}
        \caption{Ground truth}
        \label{fig:attrs_gt}
    \end{subfigure}%
    \hfill 
    \begin{subfigure}[t]{0.5\linewidth}
        \centering
        \includegraphics[width=0.7\linewidth]{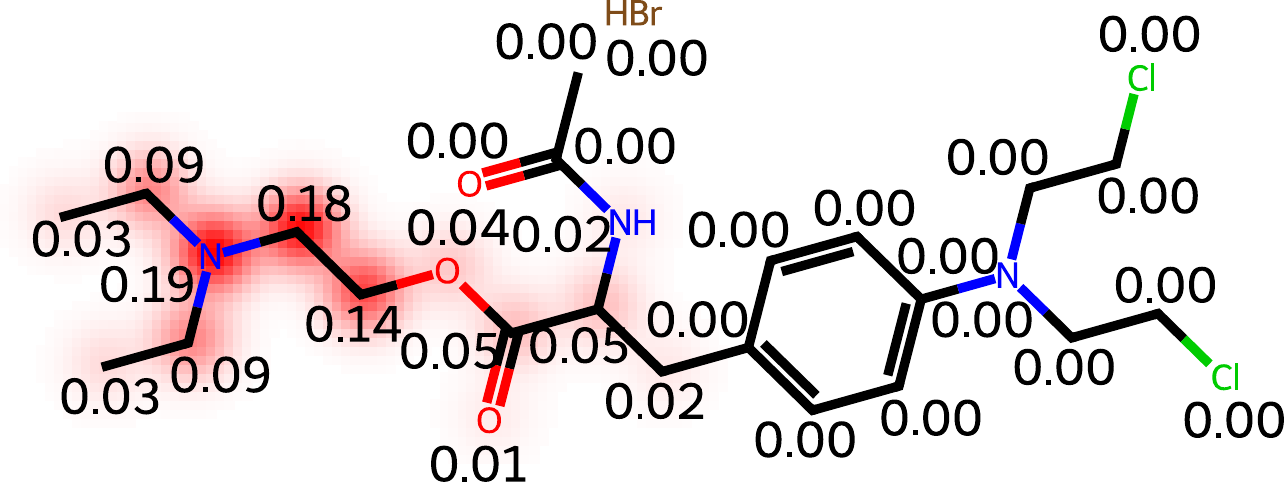}
        \caption{Integrated Gradients}
        \label{fig:attrs_ig}
    \end{subfigure}%
    \\
    \begin{subfigure}[t]{0.5\linewidth}
        \centering
        \includegraphics[width=0.7\linewidth]{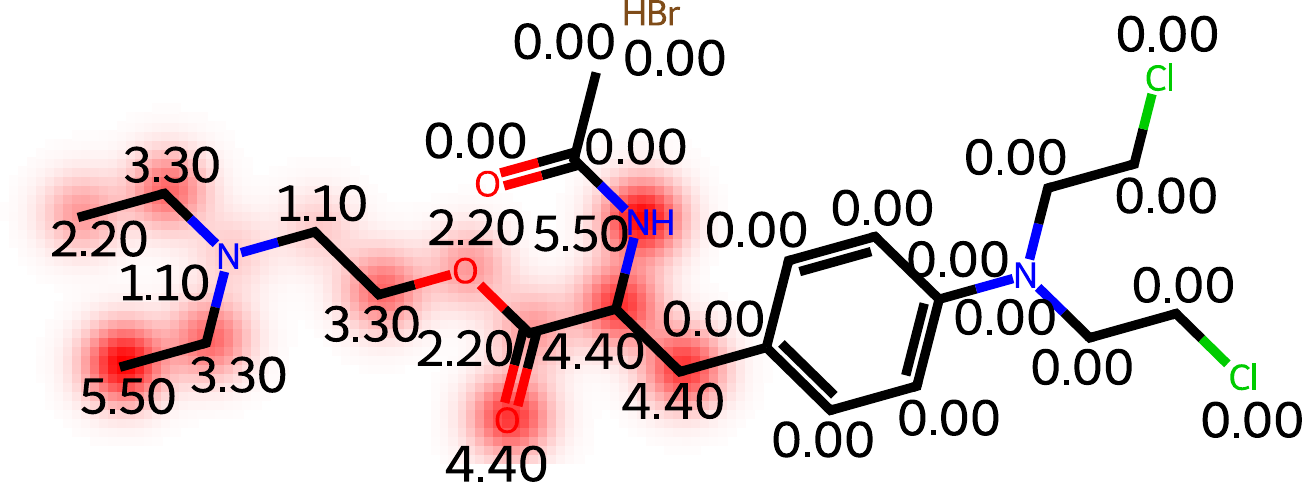}
        \caption{Shapley Value Sampling}
        \label{fig:attrs_shap}
    \end{subfigure}%
    \hfill 
        \begin{subfigure}[t]{0.5\linewidth}
        \centering
        \includegraphics[width=0.7\linewidth]{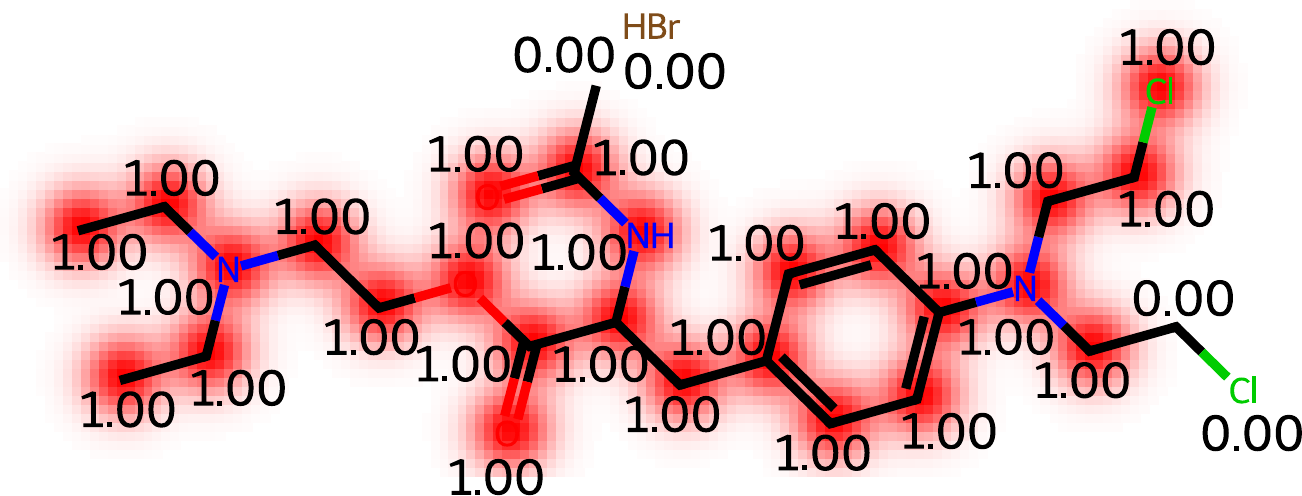}
        \caption{SubgraphX}
        \label{fig:attrs_sgx}
    \end{subfigure}%
    \caption{Attributions generated by different explainers when evaluated on the whitebox task of recognizing Pattern 4. Other explainers (i.e., PGExplainer, GNNExplainer, Saliency, and Input \(\times\) Gradient) produced null explanations. Integrated Gradients and Shapley Value Sampling attribute all motif nodes higher than non-motif nodes, whereas SubgraphX attributes importance to multiple non-motif nodes.}
    \label{fig:attrs}
    \vspace{-1em}
\end{figure}

%% file: tables/perfection_table_pos_combined.tex
\begin{table}[t]
\centering
\caption{Explainer performance across different motifs. We observe that IG is the most successful in this scenario. Moreover, we can see that a change of a readout function directly impacts the performance of IG. A checkmark indicates the explainer achieved a perfect score (\(\mathrm{AUROC}=1\)) for all positive instances of the given pattern. The sub-columns \(r_1\) and \(r_2\) denote the readout functions utilized by the whitebox models, corresponding to Eq. \ref{eqn:readout_1} and Eq. \ref{eqn:readout_2}, respectively.}
\label{tab:model_eval_perfection_combined}
\begin{scriptsize}
\begin{tabular}{c cccccccccccccc}
\toprule
\multirow{2}{*}{\textbf{Pattern}} & \multicolumn{2}{c}{\textbf{GNN Expl.}} & \multicolumn{2}{c}{\textbf{Input $\times$ Grad.}} & \multicolumn{2}{c}{\textbf{IG}} & \multicolumn{2}{c}{\textbf{PG Expl.}} & \multicolumn{2}{c}{\textbf{SHAP Sampl.}} & \multicolumn{2}{c}{\textbf{Saliency}} & \multicolumn{2}{c}{\textbf{SubgraphX}} \\
\cmidrule(lr){2-3} \cmidrule(lr){4-5} \cmidrule(lr){6-7} \cmidrule(lr){8-9} \cmidrule(lr){10-11} \cmidrule(lr){12-13} \cmidrule(lr){14-15}
& \(r_1\) & \(r_2\) & \(r_1\) & \(r_2\) & \(r_1\) & \(r_2\) & \(r_1\) & \(r_2\) & \(r_1\) & \(r_2\) & \(r_1\) & \(r_2\) & \(r_1\) & \(r_2\) \\
\midrule
1  & \redcross & \redcross & \redcross & \redcross & \redcross & \redcross & \redcross & \redcross & \redcross & \redcross & \redcross & \redcross & \redcross & \redcross \\
2  & \redcross & \redcross & \redcross & \redcross & \redcross & \greencheckmark & \redcross & \redcross & \redcross & \redcross & \redcross & \redcross & \redcross & \redcross \\
3  & \redcross & \redcross & \greencheckmark & \redcross & \greencheckmark & \greencheckmark & \redcross & \redcross & \redcross & \redcross & \greencheckmark & \redcross & \redcross & \redcross \\
4  & \redcross & \redcross & \redcross & \redcross & \redcross & \redcross & \redcross & \redcross & \redcross & \redcross & \redcross & \redcross & \redcross & \redcross \\
5  & \redcross & \redcross & \redcross & \redcross & \redcross & \redcross & \redcross & \redcross & \redcross & \redcross & \redcross & \redcross & \redcross & \redcross \\
6  & \redcross & \redcross & \redcross & \redcross & \redcross & \greencheckmark & \redcross & \redcross & \redcross & \redcross & \redcross & \redcross & \redcross & \redcross \\
7  & \redcross & \redcross & \redcross & \redcross & \redcross & \greencheckmark & \redcross & \redcross & \redcross & \redcross & \redcross & \redcross & \redcross & \redcross \\
8  & \redcross & \redcross & \redcross & \redcross & \redcross & \greencheckmark & \redcross & \redcross & \redcross & \redcross & \redcross & \redcross & \redcross & \redcross \\
9  & \redcross & \redcross & \redcross & \redcross & \redcross & \greencheckmark & \redcross & \redcross & \redcross & \redcross & \redcross & \redcross & \redcross & \redcross \\
10 & \redcross & \redcross & \redcross & \redcross & \redcross & \greencheckmark & \redcross & \redcross & \redcross & \redcross & \redcross & \redcross & \redcross & \redcross \\
11 & \redcross & \redcross & \redcross & \redcross & \redcross & \greencheckmark & \redcross & \redcross & \redcross & \redcross & \redcross & \redcross & \redcross & \redcross \\
12 & \redcross & \greencheckmark & \redcross & \redcross & \redcross & \greencheckmark & \redcross & \redcross & \greencheckmark & \greencheckmark & \redcross & \redcross & \redcross & \redcross \\
13 & \redcross & \redcross & \redcross & \redcross & \redcross & \greencheckmark & \redcross & \redcross & \redcross & \redcross & \redcross & \redcross & \redcross & \redcross \\
14 & \redcross & \redcross & \redcross & \redcross & \redcross & \greencheckmark & \redcross & \redcross & \greencheckmark & \greencheckmark & \redcross & \redcross & \redcross & \redcross \\
\bottomrule
\end{tabular}
\end{scriptsize}
\end{table}

%% file: figures/pattern_5_misattrib.tex
\begin{figure}[t]
    \begin{subfigure}[t]{0.45\linewidth}
        \centering
        \includegraphics[trim={0 41cm 0 41cm},clip, width=\linewidth]{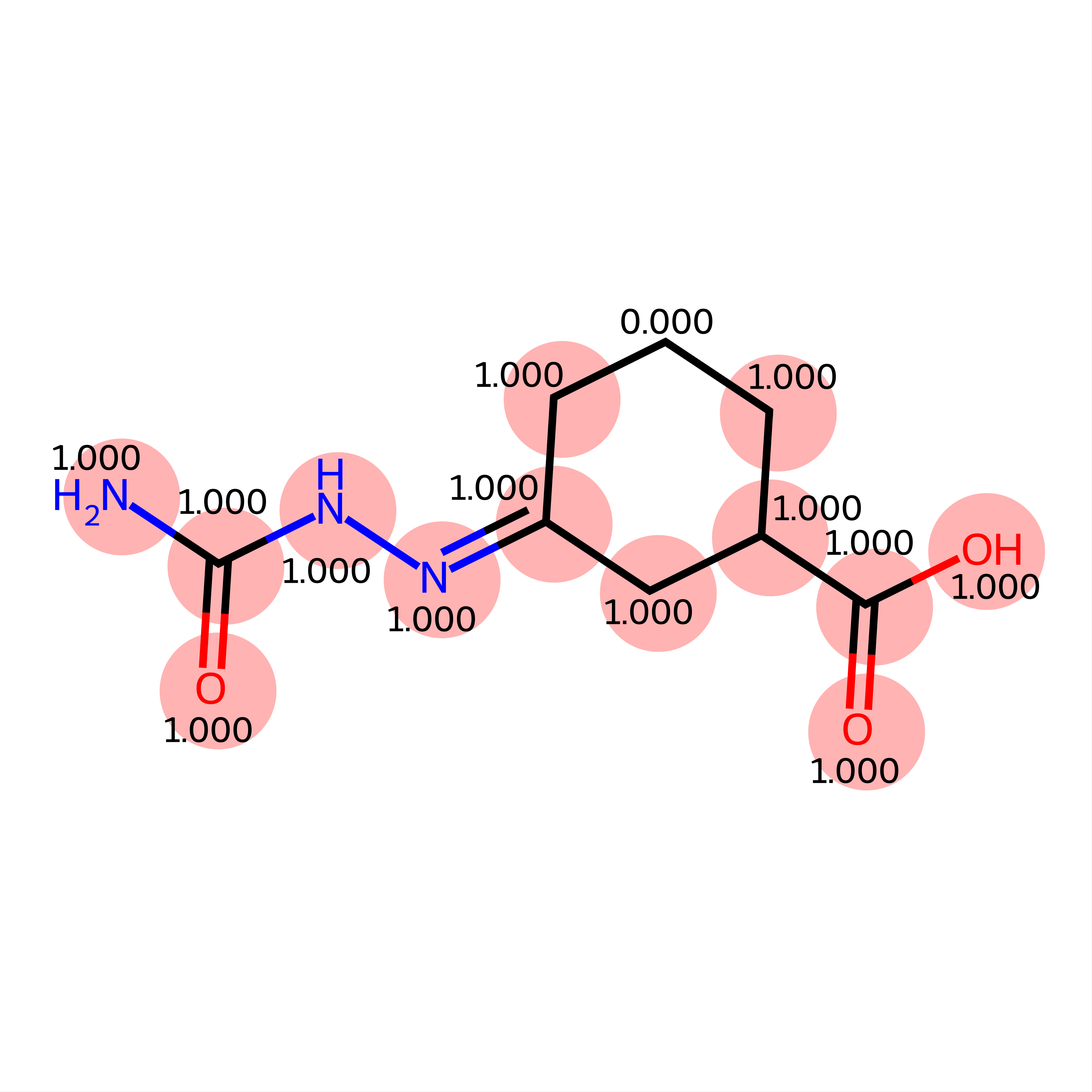}
        \caption{Ground truth attribution}
        \label{fig:pattern_5_fail_GT}
    \end{subfigure}%
    \hfill
    \begin{subfigure}[t]{0.45\linewidth}
        \centering
        \includegraphics[trim={0 41cm 0 41cm},clip, width=\linewidth]{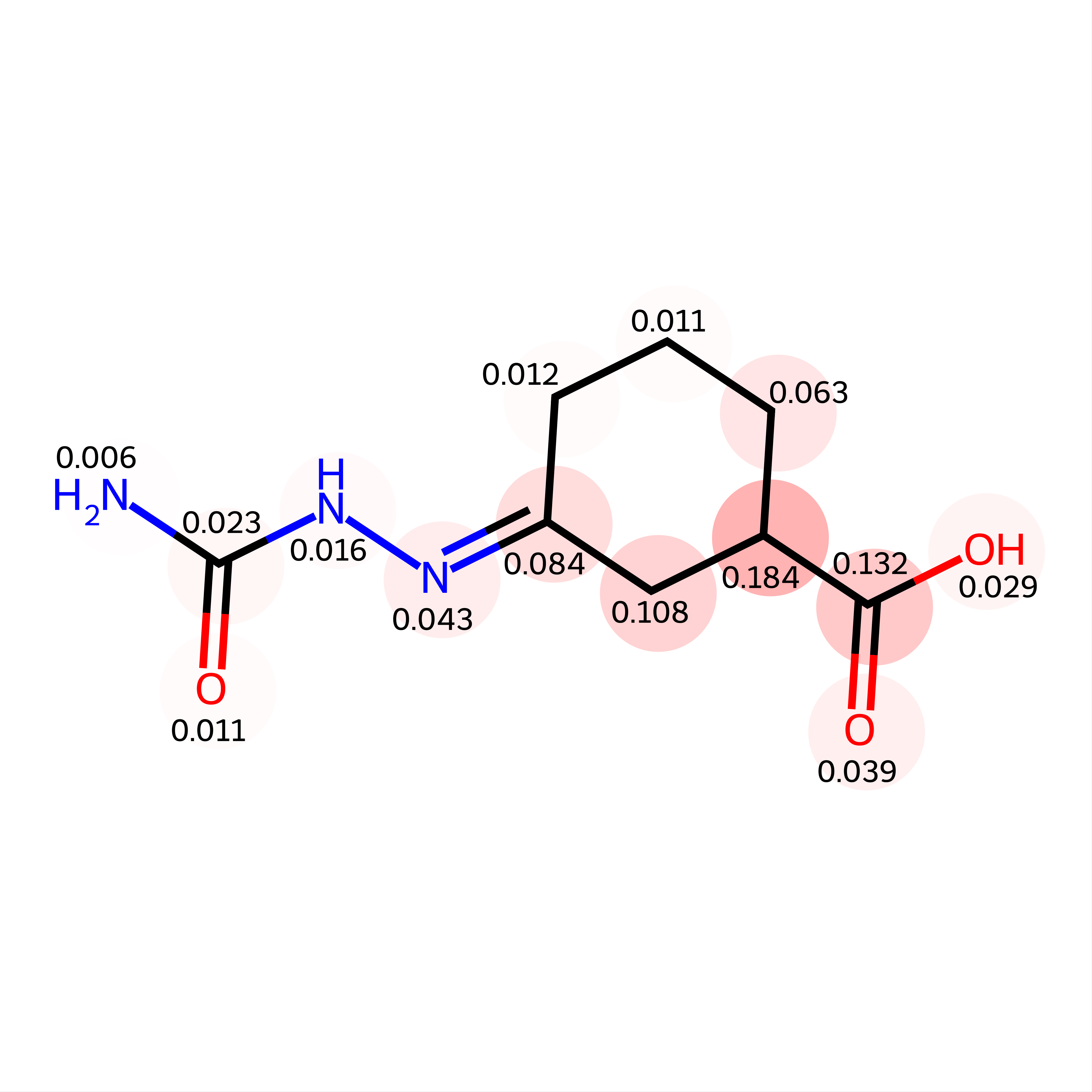}
        \caption{Integrated Gradients attribution}
        \label{fig:pattern_5_misattrib}
    \end{subfigure}%
    \hfill 
    \caption{Comparison of attribution generated by Integrated Gradients and the ground truth. Integrated Gradients is unfaithful to the model by assigning higher attribution to the non-motif carbon in the cyclohexane ring than to the \(\mathrm{NH}_2\) functional group.}
    \label{fig:pattern_5_fail}
    \vspace{-1em}
\end{figure}

%% file: figures/pattern_5_ig.tex
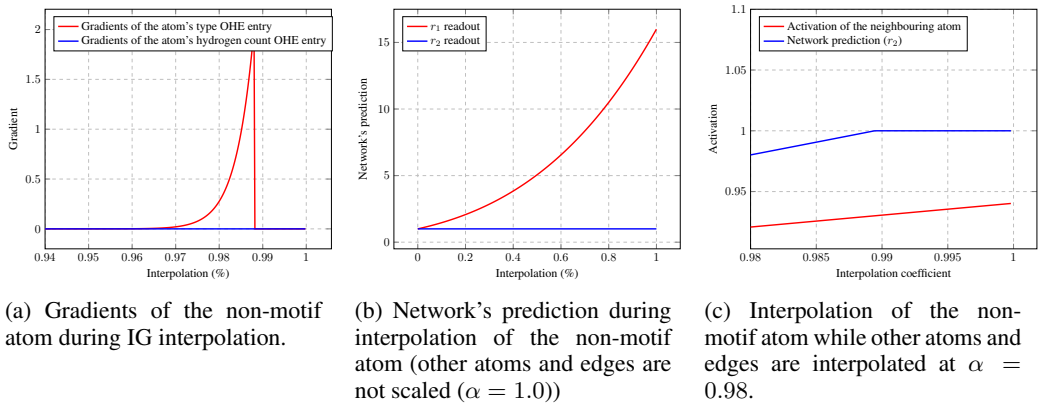
\begin{figure}[t]
\begin{subfigure}[t]{0.3\linewidth}
    \begin{tikzpicture}[scale=0.45]
    \begin{axis}[
        title={},
        xlabel={Interpolation (\%)},
        ylabel={Gradient},
        legend pos=north west,
        ymajorgrids=true,
        xmajorgrids=true,
        grid style=dashed,
        scaled y ticks=false,
        xmin=0.94,
        legend cell align=left,
        /pgf/number format/fixed,
        /pgf/number format/precision=3,
    ]

    \addplot[
        color=red,
        mark=none,
        very thick
        ] table [col sep=comma]{figures/data/atom_7_feats_gradients.csv};
    \addlegendentry{Gradients of the atom's type OHE entry}

    \addplot[
            color=blue,
            mark=none,
            very thick
            ] table [col sep=comma]{figures/data/atom_7_hydrogen_feats_gradients.csv};
        \addlegendentry{Gradients of the atom's hydrogen count OHE entry}
        
    \end{axis}
    \end{tikzpicture}
    \caption{Gradients of the non-motif atom during IG interpolation.}
    \label{subfig:pattern5-ig-main-PATH}
    \end{subfigure}%
    \hfill 
    \begin{subfigure}[t]{0.3\linewidth}
        \begin{tikzpicture}[scale=0.45]
    \begin{axis}[
        title={},
        xlabel={Interpolation (\%)},
        ylabel={Network's prediction},
        legend pos=north west,
        ymajorgrids=true,
        xmajorgrids=true,
        grid style=dashed,
        scaled y ticks=false,
        legend cell align=left,
        /pgf/number format/fixed,
        /pgf/number format/precision=3,
    ]

    \addplot[
        color=red,
        mark=none,
        very thick
        ] table [col sep=comma] {figures/data/predictions_vs_scaling_atom_7_r1.csv};
        
    \addlegendentry{\(r_1\) readout}
    
      \addplot[
        color=blue,
        mark=none,
        very thick
        ] table [col sep=comma]{figures/data/predictions_vs_scaling_atom_7_r2.csv};
        \addlegendentry{\(r_2\) readout}
    \end{axis}
    \end{tikzpicture}
    \caption{Network's prediction during interpolation of the non-motif atom (other atoms and edges are not scaled (\(\alpha=1.0\)))}
    \label{subfig:pattern5-ig-main-READOUTS}
    \end{subfigure}%
    \hfill 
        \begin{subfigure}[t]{0.3\linewidth}
    \begin{tikzpicture}[scale=0.45]
    \begin{axis}[
        title={},
        xlabel={Interpolation coefficient},
        ylabel={Activation},
        legend pos=north west,
        ymajorgrids=true,
        xmajorgrids=true,
        grid style=dashed,
        scaled y ticks=false,
        xmin=0.98,
        ymax=1.1,
        legend cell align=left,
        /pgf/number format/fixed,
        /pgf/number format/precision=3,
    ]

    \addplot[
        color=red,
        mark=none,
        very thick
        ] table [col sep=comma]{figures/data/atom_8_activations_vs_scaling_atom_7.csv};
    \addlegendentry{Activation of the neighbouring atom}

\addplot[
        color=blue,
        mark=none,
        very thick
        ] table [col sep=comma]{figures/data/full_predictions_vs_scaling_atom_7.csv};
    \addlegendentry{Network prediction (\(r_2\))}
    
    \end{axis}
    \end{tikzpicture}
    \caption{Interpolation of the non-motif atom while other atoms and edges are interpolated at \(\alpha=0.98\).}
    \label{subfig:pattern5-ig-main-IMPACT}
    \end{subfigure}
    \hfill 
    \label{fig:pattern5-ig-main}
    \caption{Analysis of causes of spurious attribution of the non-motif atom in Pattern 5. Figure \ref{subfig:pattern5-ig-main-PATH} shows gradients that are accumulated by this atom near the decision boundary. Figure \ref{subfig:pattern5-ig-main-READOUTS} demonstrates that under readout \(r_2\) the atom does not impact the network's final prediction, regardless of its interpolation state. However, it increased the network's output under readout \(r_1\), as it increases internal activations. Figure \ref{subfig:pattern5-ig-main-IMPACT} shows the effect of interpolating the non-motif atom on activation of a neighbouring node and final network's prediction under readout \(r_2\), providing a justification for curve observed in the Figure \ref{subfig:pattern5-ig-main-PATH}.}
\end{figure}

%% file: figures/leaky_ig.tex
\begin{figure}
    \begin{subfigure}[t]{0.45\linewidth}
        \includegraphics[trim={0 32cm 0 33cm},clip, width=\linewidth]{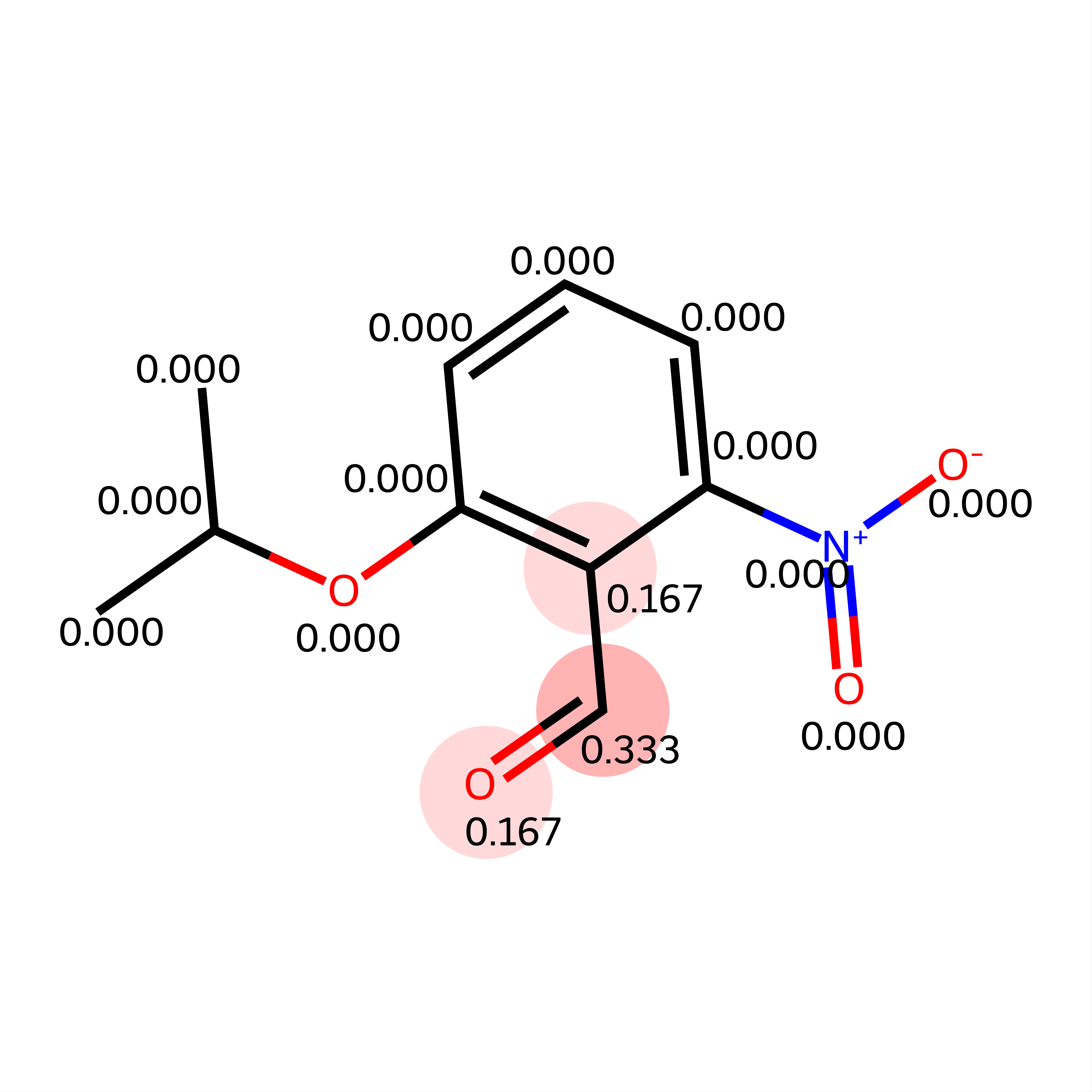}
        \caption{IG attribution for the Pattern 12 whitebox (regular version).}
    \end{subfigure}
    \hfill 
    \begin{subfigure}[t]{0.45\linewidth}
        \includegraphics[trim={0 32cm 0 33cm},clip, width=\linewidth]{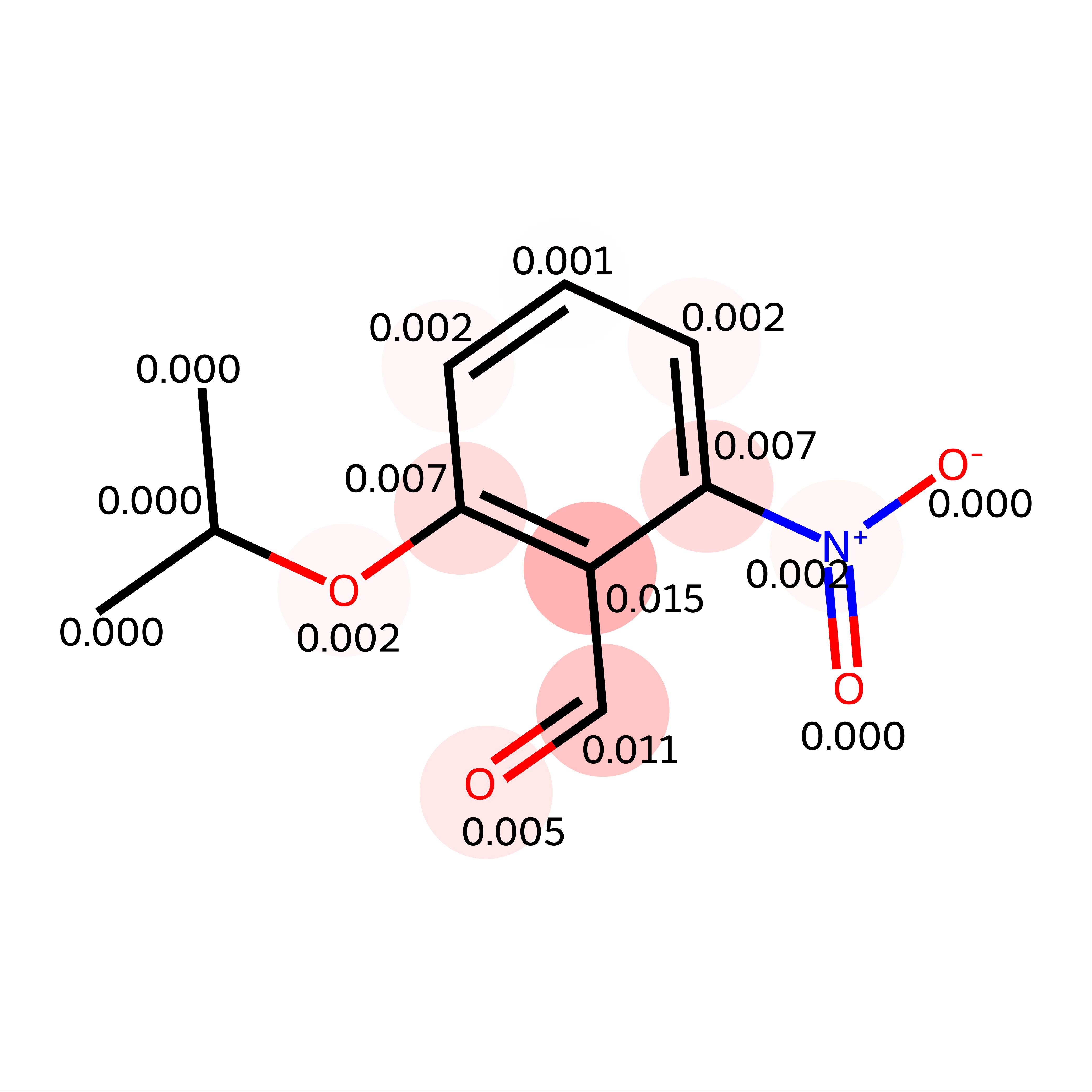}
        \caption{IG attribution for the anti-IG Pattern 12 whitebox. Note the ``leaky'' attribution.}
    \end{subfigure}
    \caption{IG attributions for two Pattern 12 implementations with identical \(r_2\) predictions on one-hot molecular inputs. Here, the regular implementation yields perfect AUROC, while redundant computations involving neighbours cause attribution spillover.}
    \label{fig:leaky_ig}
\end{figure}

%% file: pattern_4.tex
\section{Causes of misattribution in Integrated Gradients}

\label{sec:app_ig_p_4_5}

In this section, we provide implementation details of whiteboxes recognizing Pattern 4 and Pattern 5. Then, we study exact failure modes of Integrated Gradients in these networks.

\subsection{Notation}
\label{sec:notation}

In our architecture, each node is represented by a latent vector $z_i^{(\ell)} \in \mathbb{R}^{d_i}$ after passing through the $\ell$-th layer. We adopt a notation where each dimension is represented by a letter (A through G). A node is considered \textbf{active} for a given feature at layer $\ell$ if its activation at the corresponding vector index exceeds a threshold of $0.5$. For example, if there is a node that after 3rd MPNN layer has 5th element of its latent vector activated, we will say that this node is E3-activated. We adopt this notation, as presenting whitebox weights without any intuition would be highly impractical.

For brevity, we define the relational operator $X \sim Y$ to mean: \textit{An $X$-active node connected via a single bond to at least one $Y$-active node.}, as this is a pattern by which majority of our features are defined.

\subsection{Pattern 4}

\subsubsection{Implementation Detail: Neighbour Counting} 
Features are generally defined by a pattern: \textit{"A node with activation $X$ connected via edge type $Z$ to at least $N$ nodes with activation $Y$."} While ensuring the correct node and edge types is straightforward, counting neighbours beyond the first layer requires mathematical care. In our architecture, counting is implemented by applying a negative bias \textit{after} sum aggregation. Because aggregated messages can exceed binary values (e.g., a node could receive a summed message of 3), naive counting without proper thresholding can yield false positives. Pattern 4 requires nitrogen atom to be connected to 2 identical non-trivial substructures ([CH2][CH3] chains). In this case, this requires 2 separate logical circuits.

\subsubsection{Layer definitions}
\label{sec:pattern4-wihtebox-def}

\paragraph{Layer 1:}
The first layer establishes the foundational chemical features:

\begin{minipage}{0.55\linewidth}
\begin{itemize}[leftmargin=7pt]
    \item \textbf{A1:} Carbon with exactly 2 hydrogens, that is connected via single bond to at least one carbon atom with exactly 3 hydrogens.
    \item \textbf{B1:} Nitrogen that is connected via single bond to at least 3 carbon atoms.
    \item \textbf{C1:} Nitrogen that is connected via single bond to at least 4 carbon atoms.
    \item \textbf{D1:} Carbon with exactly 2 hydrogens, that is connected via single bond to at least one carbon atom.
    \item \textbf{E1:} Oxygen with no hydrogens, that is connected via single bond to at least one carbon atom.
    \item \textbf{F1:} Carbon with no hydrogens, that is connected via a double bond to an oxygen atom.
    \item \textbf{G1:} Carbon with exactly 1 hydrogen that is connected via a single bond to at least 3 other nodes.
\end{itemize}
\end{minipage}%
\begin{minipage}{0.1\linewidth}
    \begin{center}
        \rule{1pt}{100pt}
    \end{center}
\end{minipage}%
\begin{minipage}{0.35\linewidth}
    \begin{center}
        \includegraphics[width=\linewidth]{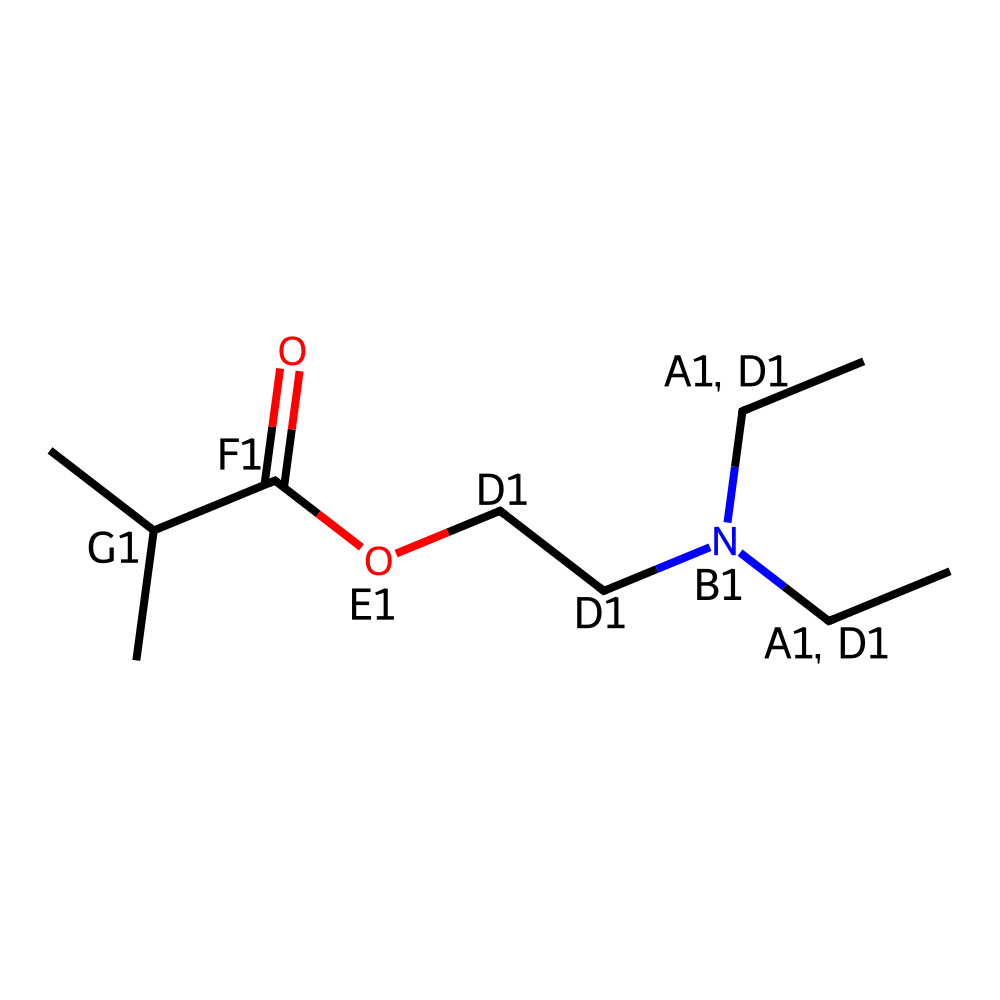}
    \end{center}
\end{minipage}

\paragraph{Layer 2.}
The second layer aggregates low-level chemical features into higher-level ones and splits ''reasoning'' of the model into 2 separate logical circuits (A2, B2).

\begin{minipage}{0.55\linewidth}
\begin{itemize}[leftmargin=7pt]
    \item \textbf{A2:} B1 that is \textbf{not} C1 and is connected via a single bond to at least 2 A1-activated nodes.
    \item \textbf{B2:} C1 that is connected via a single bond to at least 2 A1-activated nodes.
    \item \textbf{C2:} \(\mathrm{D1} \sim \mathrm{D1}\)
    \item \textbf{D2:} \(\mathrm{E1} \sim \mathrm{D1}\)
    \item \textbf{E2:} \(\mathrm{F1} \sim \mathrm{G1}\)
\end{itemize}
\end{minipage}%
\begin{minipage}{0.1\linewidth}
    \begin{center}
          \rule{1pt}{100pt}  
    \end{center}
\end{minipage}%
\begin{minipage}{0.35\linewidth}
    \begin{center}
        \includegraphics[width=\linewidth]{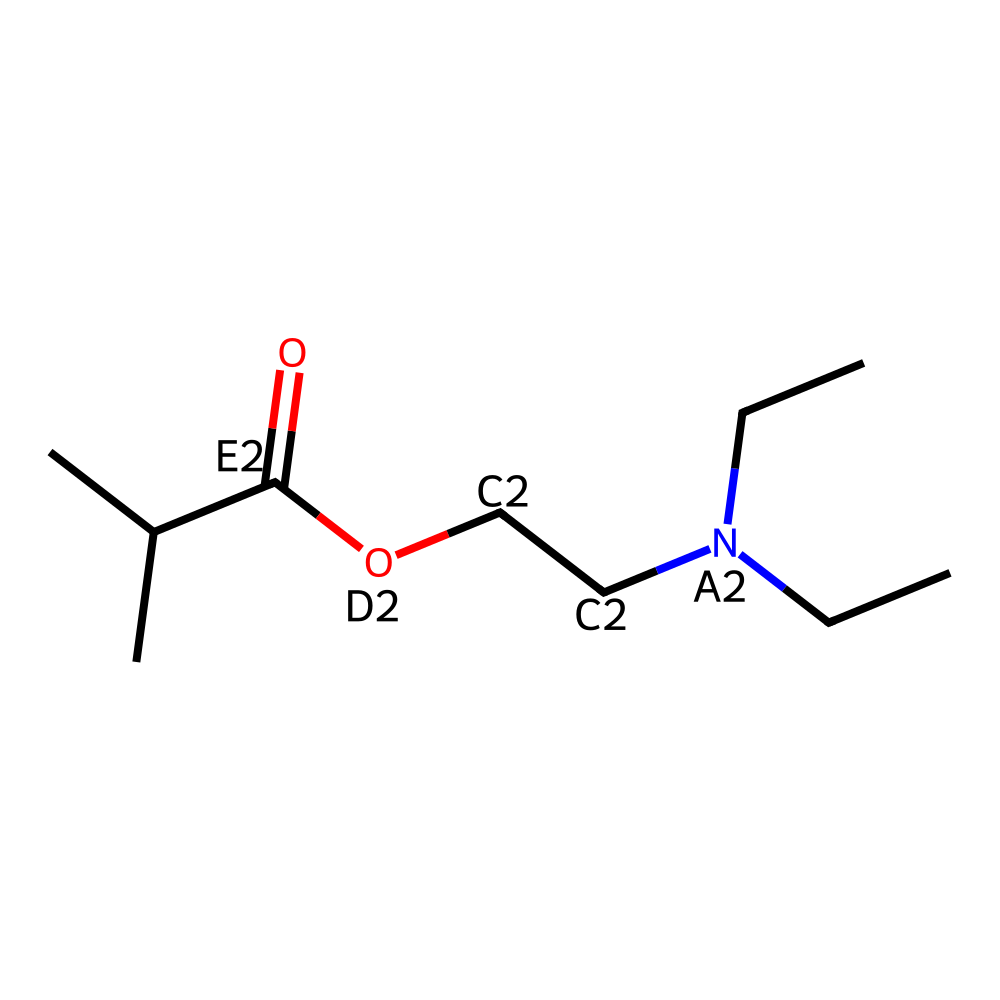}
    \end{center}
\end{minipage}

\textbf{Note that a node cannot be A2 and B2-activated at the same time.} The need to split these computations into 2 separate activations is a direct result of the fact that B1-active nodes can either have \(\mathrm{B1}=1\) (if the nitrogen has 3 neighbouring carbons) or \(\mathrm{B1} = 2\) (if it has 4 neighbouring carbons). Because our message passing relies on multiplication, a naive implementation of A2 activation which does not take that into account would be sensitive to false positives. As such, the split into B1 and C1 activations in the first layer and A2 and B2 in the second guarantees that a correct calculation is performed regardless of nitrogen's number of neighbours. A2 feature (''B1 that is not C1'') is implemented by setting large negative weight for C1 activation, preventing A2-activations on nodes that were C1-activated.  

\paragraph{Layer 3.} This layer continues aggregation of lower-level features into higher-level ones. 
\begin{minipage}{0.55\linewidth}

 \begin{itemize}[leftmargin=7pt]
    \item \textbf{A3:} C2 with A2 or B2.
    \item \textbf{B3:} \(\mathrm{C2} \sim \mathrm{D2}\)
    \item \textbf{C3:} \(\mathrm{D2} \sim \mathrm{E2}\)
\end{itemize}

\end{minipage}%
\begin{minipage}{0.1\linewidth}
    \begin{center}
        \rule{1pt}{100pt}
    \end{center}
\end{minipage}%
\begin{minipage}{0.35\linewidth}
    \begin{center}
        \includegraphics[width=\linewidth]{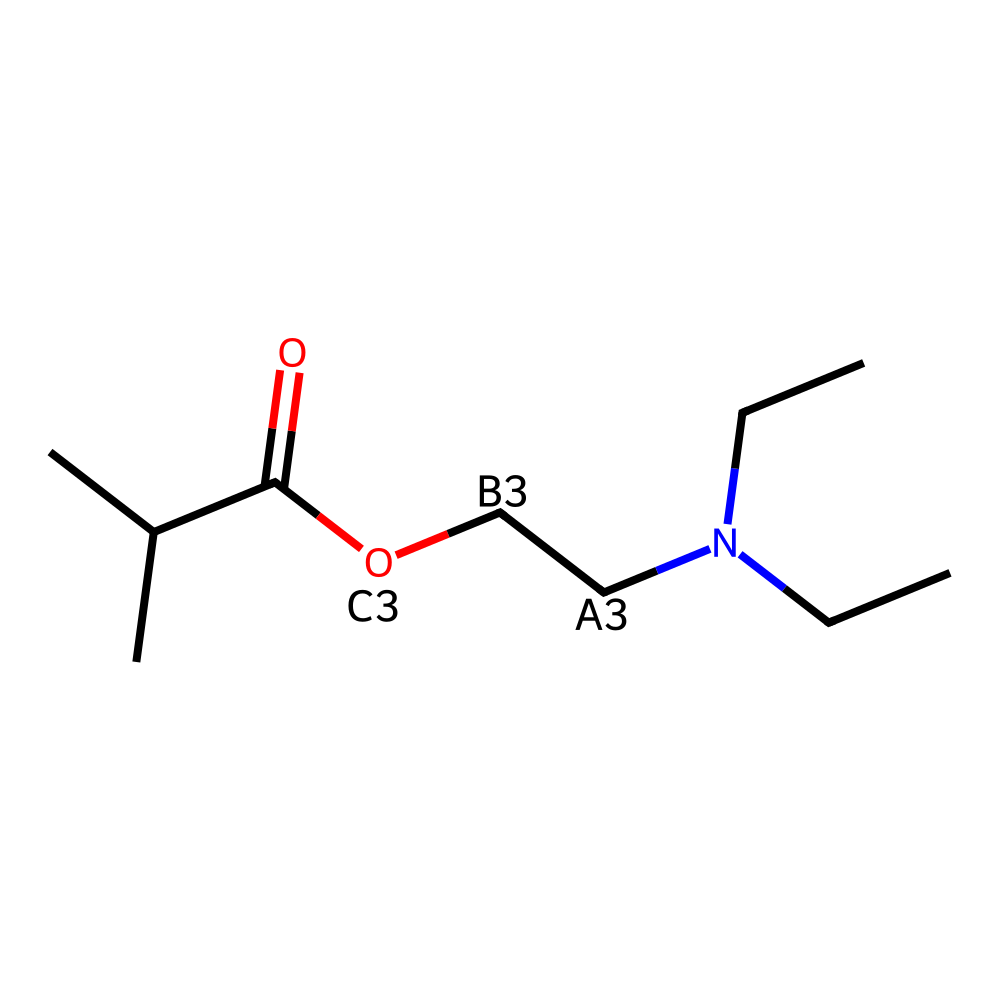}
    \end{center}
\end{minipage}

\paragraph{Layer 4.} For network to classify a molecule as containing Pattern 4, there should exist a node that contains all of the following activations: 

\begin{minipage}{0.55\linewidth}

 \begin{itemize}[leftmargin=7pt]
    \item \textbf{A4:} \(\mathrm{B3} \sim \mathrm{A3}\)
    \item \textbf{B4:} \(\mathrm{B3} \sim \mathrm{C3}\)
\end{itemize}

\end{minipage}%
\begin{minipage}{0.1\linewidth}
    \begin{center}
        \rule{1pt}{100pt}
    \end{center}
\end{minipage}%
\begin{minipage}{0.35\linewidth}
    \begin{center}
        \includegraphics[width=\linewidth]{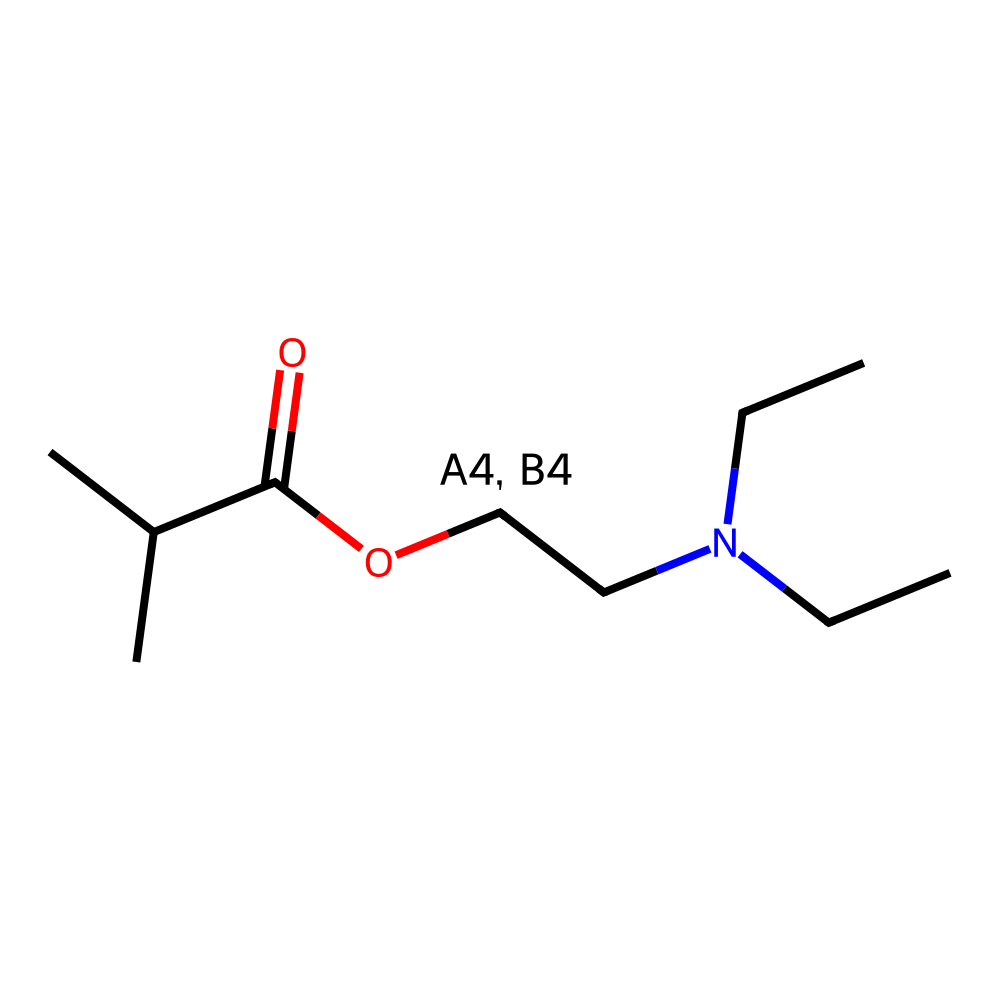}
    \end{center}
\end{minipage}

\paragraph{Readout.}
The final graph-level readout function simply checks for the existence of at least one node that is \(\mathrm{A4}\)-active and \(\mathrm{B4}\)-active.

\subsection{Pattern 5}

In contrast to the GNN for Pattern 4, GNN for Pattern 5 is conceptually simpler, without relying on 2-circuit architecture. 

\subsubsection*{Layer definitions}
\label{sec:pattern5-wihtebox-def}

\paragraph{Layer 1:}
The first layer establishes the foundational chemical features:

\begin{minipage}{0.55\linewidth}
\begin{itemize}[leftmargin=7pt]
    \item \textbf{A1:} Oxygen with exactly 0 hydrogens, that is connected via a double bond to at least one carbon atom.
    \item \textbf{B1:} Oxygen with exactly 1 hydrogen, that is connected via a single bond to at least one carbon atom.
    \item \textbf{C1:} Carbon with exactly 1 hydrogen, that is connected via single bonds to at least 3 other atoms.
    \item \textbf{D1:} Carbon with exactly two hydrogens, that is connected via single bonds to at least 2 other carbon atoms.
    \item \textbf{E1:} Carbon with exactly 0 hydrogens, that is connected to at least 3 other atoms.
    \item \textbf{F1:} Nitrogen with exactly 0 or 1 hydrogens, that is connected via a double bond to at least 1 carbon.
    \item \textbf{G1:} Nitrogen with exactly 1 or 2 hydrogens, that is connected via a single bond to at least 1 other nitrogen. 

\end{itemize}
\end{minipage}%
\begin{minipage}{0.1\linewidth}
    \begin{center}
        \rule{1pt}{100pt}
    \end{center}
\end{minipage}%
\begin{minipage}{0.35\linewidth}
    \begin{center}
        \includegraphics[width=\linewidth]{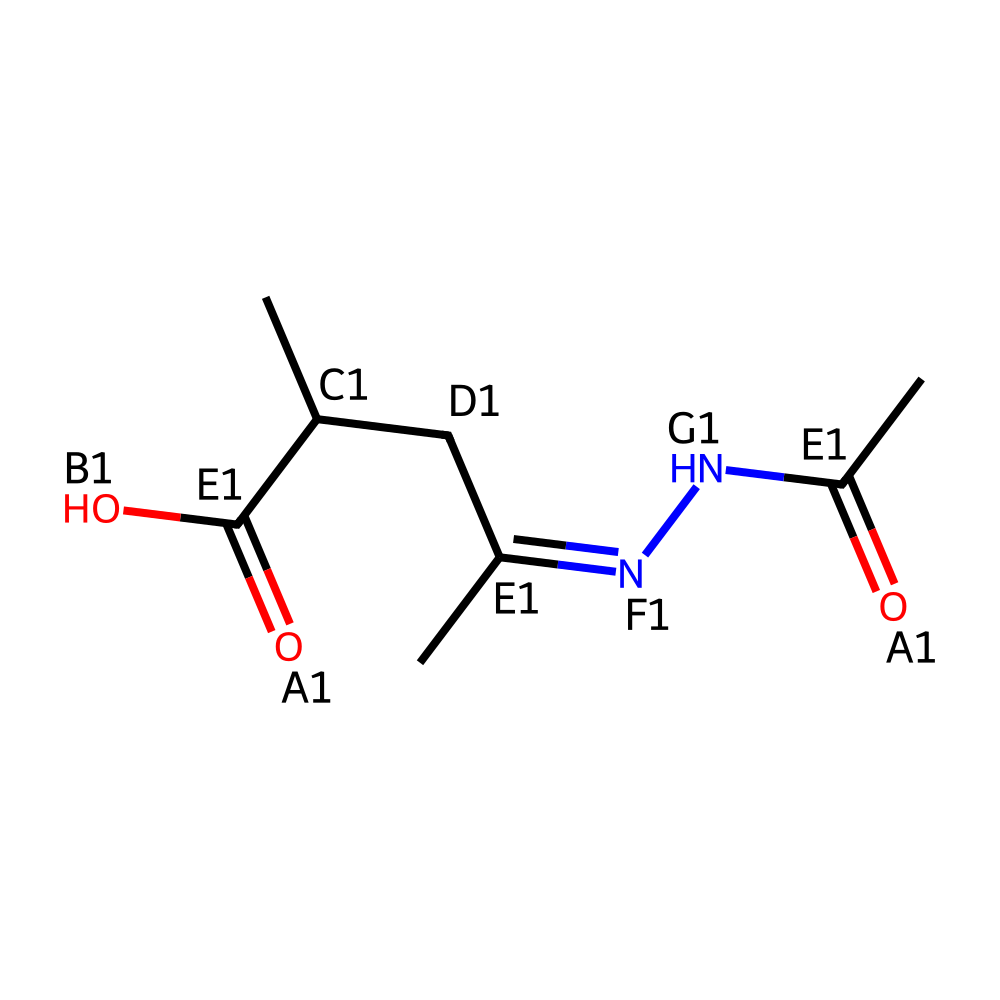}
    \end{center}
\end{minipage}

\paragraph{Layer 2:}
The second layer aggregates low-level chemical features into higher-level ones.

\begin{minipage}{0.55\linewidth}
\begin{itemize}[leftmargin=7pt]
    \item \textbf{A2:} An E1-active node connected via a double bond to at least one A1-active node.
    \item \textbf{B2:} \(\mathrm{B1} \sim \mathrm{E1}\)
    \item \textbf{C2:} \(\mathrm{C1} \sim \mathrm{D1}\)
    \item \textbf{D2:} \(\mathrm{D1} \sim \mathrm{C1}\)
    \item \textbf{E2:} An E1-active node connected via a double bond to at least one F1-active node.
    \item \textbf{F2:} \(\mathrm{F1} \sim \mathrm{G1}\)
    \item \textbf{G2:} \(\mathrm{G1} \sim \mathrm{F1}\)

\end{itemize}
\end{minipage}%
\begin{minipage}{0.1\linewidth}
    \begin{center}
        \rule{1pt}{100pt}
    \end{center}
\end{minipage}%
\begin{minipage}{0.35\linewidth}
    \begin{center}
        \includegraphics[width=\linewidth]{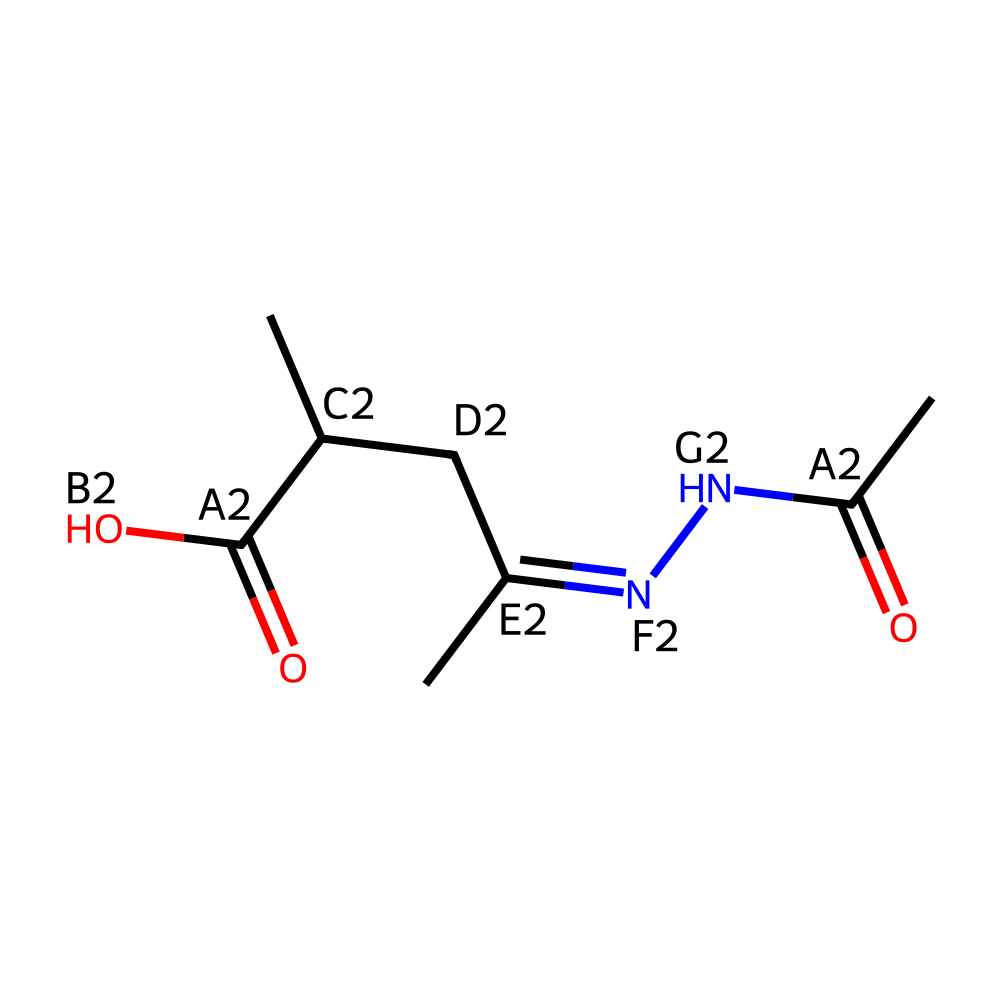}
    \end{center}
\end{minipage}

\paragraph{Layer 3:}
The third layer continues to aggregate lower-level chemical features into higher-level ones.

\begin{minipage}{0.55\linewidth}
\begin{itemize}[leftmargin=7pt]
    \item \textbf{A3:} \(\mathrm{A2} \sim \mathrm{B2}\)
    \item \textbf{B3:} \(\mathrm{C2} \sim \mathrm{A2}\)
    \item \textbf{C3:} \(\mathrm{D2} \sim \mathrm{C2}\)
    \item \textbf{D3:} \(\mathrm{E2} \sim \mathrm{D2}\)
    \item \textbf{E3:}  A F2-active node connected via a double bond to at least one E2-active node.
    \item \textbf{F3:} \(\mathrm{G2} \sim \mathrm{A2}\)
\end{itemize}
\end{minipage}%
\begin{minipage}{0.1\linewidth}
    \begin{center}
        \rule{1pt}{100pt}
    \end{center}
\end{minipage}%
\begin{minipage}{0.35\linewidth}
    \begin{center}
        \includegraphics[width=\linewidth]{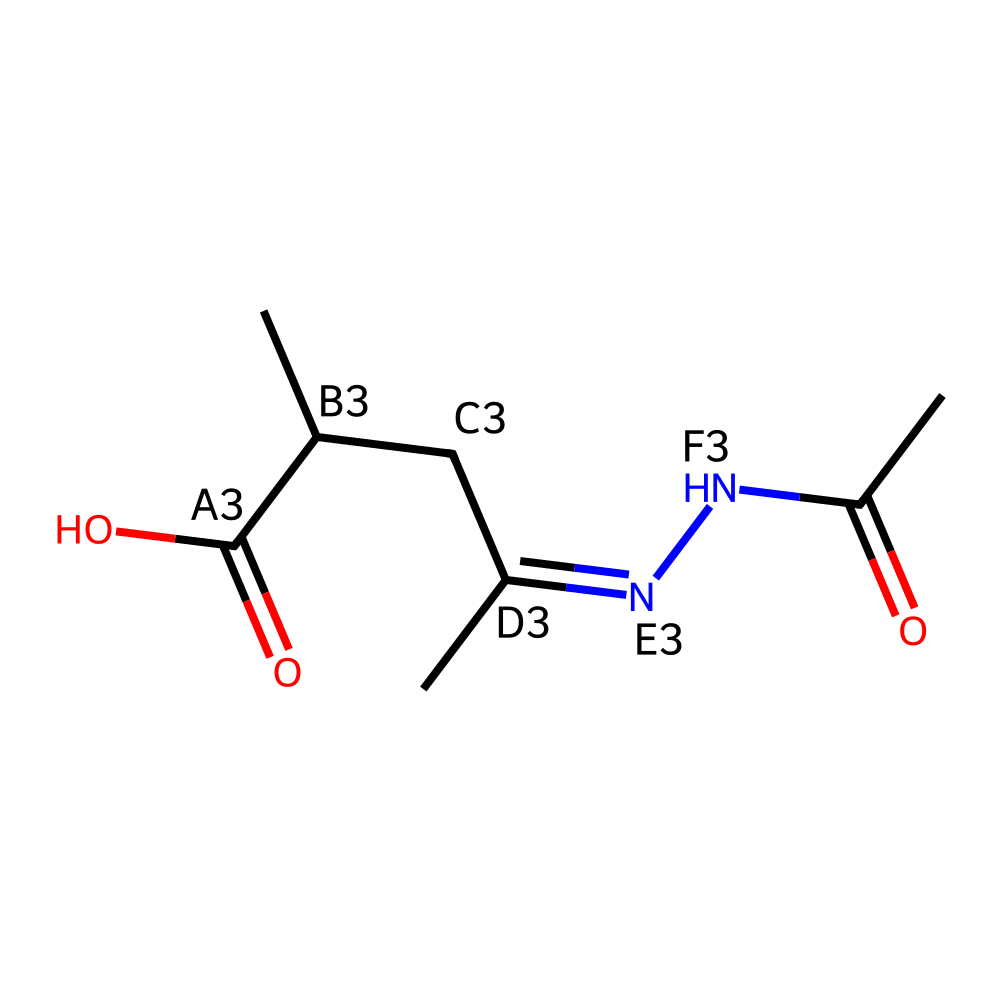}
    \end{center}
\end{minipage}

\paragraph{Layer 4:}
The fourth layer continues to aggregate lower-level chemical features into higher-level ones.

\begin{minipage}{0.55\linewidth}
\begin{itemize}[leftmargin=7pt]
    \item \textbf{A4:} \(\mathrm{B3} \sim \mathrm{A3}\)
    \item \textbf{B4:} \(\mathrm{C3} \sim \mathrm{B3}\)
    \item \textbf{C4:} \(\mathrm{D3} \sim \mathrm{C3}\)
    \item \textbf{D4:} \(\mathrm{E3} \sim \mathrm{F3}\)
\end{itemize}
\end{minipage}%
\begin{minipage}{0.1\linewidth}
    \begin{center}
        \rule{1pt}{100pt}
    \end{center}
\end{minipage}%
\begin{minipage}{0.35\linewidth}
    \begin{center}
        \includegraphics[width=\linewidth]{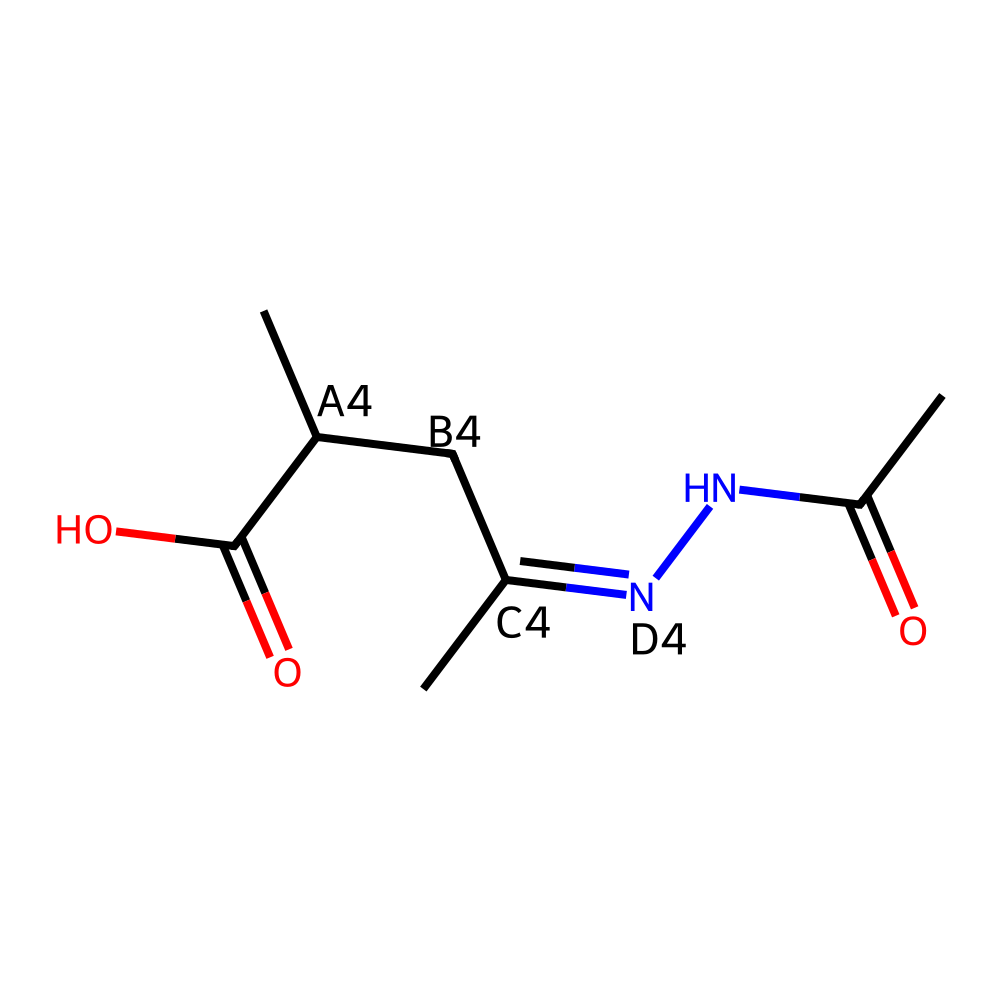}
    \end{center}
\end{minipage}

\paragraph{Layer 5:}
The fifth layer continues to aggregate lower-level chemical features into higher-level ones.

\begin{minipage}{0.55\linewidth}
\begin{itemize}[leftmargin=7pt]
    \item \textbf{A5:} \(\mathrm{B4} \sim \mathrm{A4}\)
    \item \textbf{B5:} A C4-active node connected via a double bond to at least one D4-active node.
\end{itemize}
\end{minipage}%
\begin{minipage}{0.1\linewidth}
    \begin{center}
        \rule{1pt}{100pt}
    \end{center}
\end{minipage}%
\begin{minipage}{0.35\linewidth}
    \begin{center}
        \includegraphics[width=\linewidth]{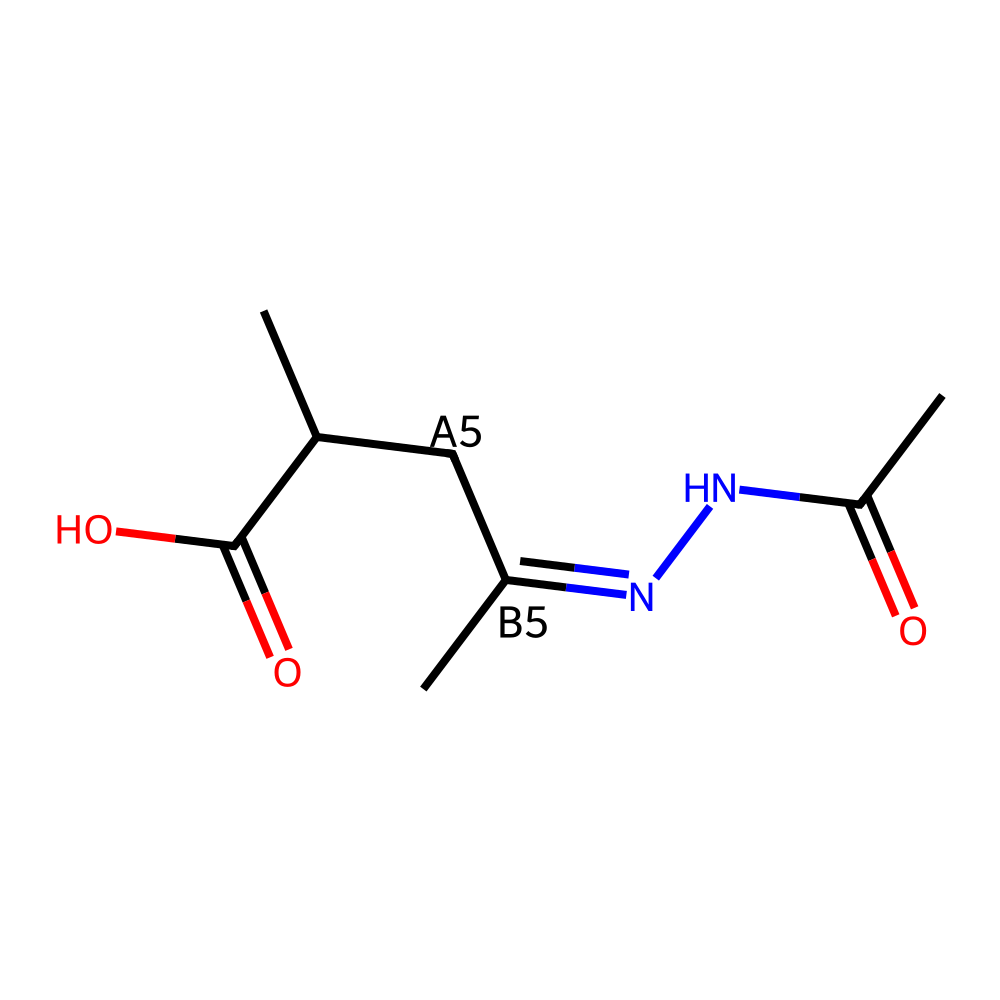}
    \end{center}
\end{minipage}

\paragraph{Layer 6:}
The sixth layer ''looks'' for A5-active node connected via a single bond to B5-active node:

\begin{minipage}{0.55\linewidth}
\begin{itemize}[leftmargin=7pt]
    \item \textbf{A6:} \(\mathrm{A5} \sim \mathrm{B5}\)
\end{itemize}
\end{minipage}%
\begin{minipage}{0.1\linewidth}
    \begin{center}
        \rule{1pt}{100pt}
    \end{center}
\end{minipage}%
\begin{minipage}{0.35\linewidth}
    \begin{center}
        \includegraphics[width=\linewidth]{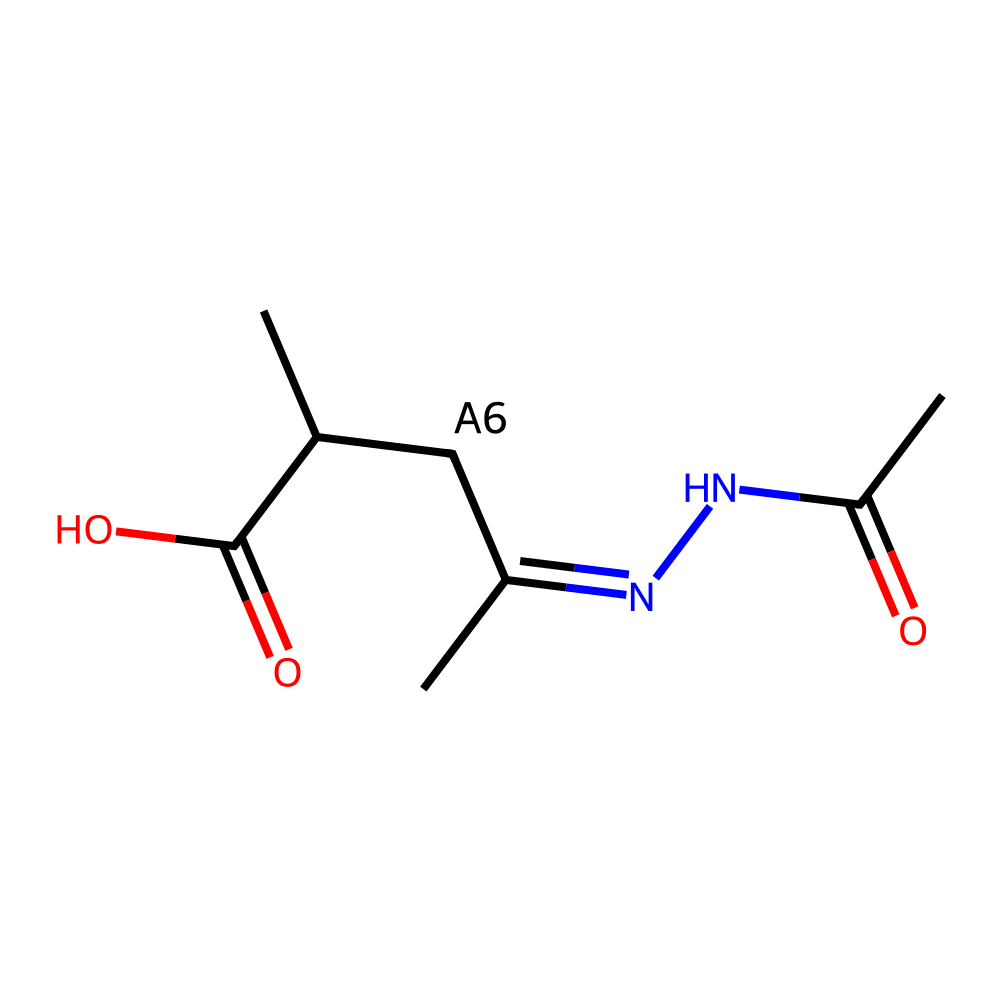}
    \end{center}
\end{minipage}

\paragraph{Readout.}
The final graph-level readout function simply checks for the existence of at least one node that is \(\mathrm{A6}\)-active.

\subsection{Preliminaries}

Introduced by~\cite{sundararajan2017axiomatic}, Integrated Gradients is an attribution method formally defined as: 

\begin{equation}
    \label{eq:ig}
    \mathrm{IG}_i(x) = (x_i - x_i') \int_{\alpha=0}^{1} \frac{\partial f(x' + \alpha(x-x'))}{\partial x_i} \mathrm{d}\alpha
\end{equation}

where: 
\begin{enumerate}
    \item \(f\) is a model being explained. In WOMBAT's case, this means one of our 14 whiteboxes.
    \item \(x\) is an input to the network. In WOMBAT's case, this means tensors encoding features of nodes and edges (which we will denote, respectively, as \(x_V\) and \(x_{E}\)) of a given graph (as specified in Section \ref{sec:design}).
    \item \(\mathrm{IG}_i(x)\) is attribution of a specific element of the input, which can be identified by the number \(i\). 
    \item \(x_i'\) is a \textit{baseline}. This is a parameter of the method; it should represent a ''neutral'' input to the network, i.e. an input that yields a neutral prediction. In WOMBAT's case, this means tensors \(x'_V\) and \(x'_E\) that match shapes of \(x_V\) and \(x_E\), but fully consist of zeroes. All of our networks predict that there is no pattern present when receiving zero-baseline as an input, making the baseline selection justified.
\end{enumerate}

As such, in WOMBAT's case Equation \ref{eq:ig} can be seen as 2 equations:

\begin{align*}
    \mathrm{IG}_{V_i}(x_V, x_E) &= x_{V_i} \int_{\alpha=0}^1 \frac{\partial f(\alpha x_V, \alpha x_E)}{\partial x_{V_i}} \mathrm{d}\alpha\\
    \mathrm{IG}_{E_i}(x_V, x_E) &= x_{E_i} \int_{\alpha=0}^1 \frac{\partial f(\alpha x_V, \alpha x_E)}{\partial x_{E_i}} \mathrm{d}\alpha
\end{align*}

In practice, however, we don't use edge attributions for our evaluations. As such, we can simplify these equations back to 1 equation:

\begin{equation}
    \mathrm{IG}_{i}(x_V, x_E) = x_{V_i} \int_{\alpha=0}^1 \frac{\partial f(\alpha x_V, \alpha x_E)}{\partial x_{V_i}} \mathrm{d}\alpha
\end{equation}

For brevity, throughout the remainder of this paper we use the term gradient at interpolation step \(\alpha\), meaning:
\[
x_{V_i} \frac{\partial f(\alpha x_V, \alpha x_E)}{\partial x_{V_i}}
\]
rather than the raw partial derivative.

\subsection{Integrated Gradients in Pattern 4}
\label{sec:ig_in_pattern4_deep}
\input{figures/ig_misattributions}
\input{figures/pattern_4_idxs_and_interp}
\input{figures/pattern_4_ig}

As shown in Figure \ref{fig:pattern_4_fails}, Integrated Gradients can fail by misattributing non-motif atoms. Figure \ref{subfig:pattern4-ig-main-PATH} demonstrates a sudden spike in gradients accumulated by the non-motif atom during the IG's interpolation between baseline and actual input to the network which causes the misattribution. In this section, we investigate \textit{why} it happens in a greater detail. First, we need to index atoms belonging to the molecule. We provide the exact same molecule from Figure \ref{fig:pattern_4_fails} in Figure \ref{fig:pattern_4_idx}, along with node indices and their respective attributions (denoted as \textsc{[Node idx]: [Attribution]}). In this case, Atom 3 receives bigger attribution than Atom 0 or Atom 5, even though it does not belong to the motif -- while these atoms do. 

As demonstrated in Figure \ref{subfig:pattern4-ig-main-PATH}, the gradient spike for Atom 3 occurs when input features are approximately scaled to \(\alpha=0.96\), peaking around \(\alpha=0.98\). We identify one of the paths leading to this issue and demonstrate it intuitively in Figure \ref{fig:pattern4-more-interpolations}. 

Assume that all features are interpolated by a given \(\alpha\). Consider an atom \(a_i\), with atomic number \(n_i\) and number of adjacent hydrogens \(h_i\). Then it is represented by a vector \(v_i\) such that:

\[
v_i = [ 0 \ 0 \ \dots \ \overbrace{\alpha}^{v[n_i]} \ \dots \ \overbrace{\alpha}^{v[100+h_i]} \ \dots ]
\]

That is because in WOMBAT's case, representation of every atom is a concatenation of one hot encoding of its atomic number and number of its adjacent hydrogens. 

Similarly, all edge representations can be expressed as vector of length 4, where \(\alpha\) is the sole non-zero value; its placement is dependent on bond arity.
 
In case of the atom 3, \(n_3=6\) (because it is a carbon) and \(h_3=3\) (because it is a \(\mathrm{CH}_3\) carbon). We define now \(v_i'\):

\[
v_i' = [ 0 \ 0 \ \dots \ \overbrace{\alpha + \varepsilon}^{v[n_i]} \ \dots \ \overbrace{\alpha}^{v[100+h_i]} \ \dots ]
\]

where \(\varepsilon\) is infinitesimally small. In this section, we occasionally refer to \(v_i'\) as \(\varepsilon\)-perturbed representation of \(a_i\). Note that we will only consider perturbing atomic numbers, as we observe gradient spikes for them. 

We now demonstrate a computation path that causes gradients during interpolation for \(v_3\) to be non-zero. Consider activation C1, which was previously defined (see Section \ref{sec:pattern4-wihtebox-def}) as a \textit{nitrogen that is connected via single bond to at least 4 carbon atoms}. In practice, this is implemented by following weight set\footnote{Note that here, for simplicity, we define a weight sets for C1  \textit{only}, which can be denoted as vectors. In practice, \(W_{\mathrm{from}}\), \(W_{\mathrm{to}}\) and \(W_{e}\) are \textit{matrices} (see Figure \ref{fig:wombat}).}: 

\begin{align*}
    W_{\mathrm{from}} &= [\overbrace{\mathrm{OHE(100, 6)}}^{\text{From carbon}} \ \ \ | \ \ \ \overbrace{0 \ \dots \ 0}^{\text{regardless of its hydrogen count}}] \\  
    W_{\mathrm{to}} &= [\overbrace{\mathrm{OHE(100, 7)}}^{\text{To nitrogen}} \ \ \ | \ \ \ \overbrace{0 \ \dots \ 0}^{\text{regardless of its hydrogen count}}] \\ 
    W_{e} &= [\overbrace{1}^{\text{single bond
    }} \ \ 0 \ \ 0 \ \ 0] \\ 
    b &= -3
\end{align*}

Consider messages received by the atom 2 (nitrogen ion):

\[
m_{i2} = v_i W_{\mathrm{from}} \odot v_2W_{\mathrm{to}} \odot e_{\mathrm{i2}} W_{\mathrm{e}}
\]

Since atom 2 is a nitrogen, it follows that:

\[
v_2W_{\mathrm{to}} = \alpha
\]

Similarly, atom 2 is connected \textit{only} to carbon atoms via single bonds, which trivially yields:

\begin{align*}
    \forall_{i \in \mathcal{N}(v_2)} \ v_iW_{\mathrm{from}} &= \alpha \\
    \forall_{i\in \mathcal{N}(v_2)} \ e_{i2}W_{\mathrm{e}} &= \alpha
\end{align*}

As such, \(\forall_{i\in\mathcal{N}(v_2)} \ m_{i2} = \alpha^3\), and, subsequently, sum \(s\) of messages received by atom 2 is equal to \(s =\sum_{i\in\mathcal{N}(v_2)} m_{i2} = 4 \alpha^3\). This means that C1-activation in atom 2 becomes: 

\[
C_1 = \mathrm{ReLU}(s+b) = \mathrm{ReLU}(4\alpha^3 - 3) 
\] 

Consider now message from an \(\varepsilon\)-perturbed node \(a_i\). Since the edge encoding and recipient node encoding did not change, we only need to consider the ''from'' component of the message multiplication:

\[
v_i'W_{\mathrm{from}} = \alpha + \varepsilon
\]

hence: 

\[
m_{i2}' = \alpha^2(\alpha + \varepsilon) = \alpha^3 + \varepsilon \alpha^2
\]

meaning that sum of messages received by the atom 2 would in turn become \(s' =4\alpha^3 + \varepsilon\alpha^2 = s + \varepsilon\alpha^2\), and:

\[
C_1' = \mathrm{ReLU}(4\alpha^3 + \varepsilon\alpha^2 - 3)
\]

As such, for sufficiently large \(\alpha\), \(C_1' = C_1 + \varepsilon_2\), where \(\varepsilon_2\) is some positive constant. Isomorphic reasoning can be performed to show the following:

\begin{enumerate}
    \item \(\varepsilon_2\)-perturbed C1 activation in atom 2 causes some \(\varepsilon_3\)-perturbation of activation B2 in atom 2 for sufficiently large \(\alpha\).
    \item \(\varepsilon_3\)-perturbed B2 activation in atom 2 causes some \(\varepsilon_4\)-perturbation of activation A3 in atom 6 for sufficiently large \(\alpha\).
    \item \(\varepsilon_4\)-perturbed A3 activation in atom 6 causes \(\varepsilon_5\)-perturbation of activation A4 in atom 7 for sufficiently large \(\alpha\).
    \item \(\varepsilon_5\)-perturbation of activation A4 in atom 7 causes \(\varepsilon_6\)-perturbation of model's output for sufficiently large \(\alpha\).
\end{enumerate}

This provides an explicit explanation for the gradient spike.

\subsection{Integrated Gradients in Pattern 5}

\begin{figure}
    \centering
\begin{subfigure}[t]{0.42\linewidth}
        \centering
        \includegraphics[trim={0 40cm 0 40cm},clip, width=\linewidth]{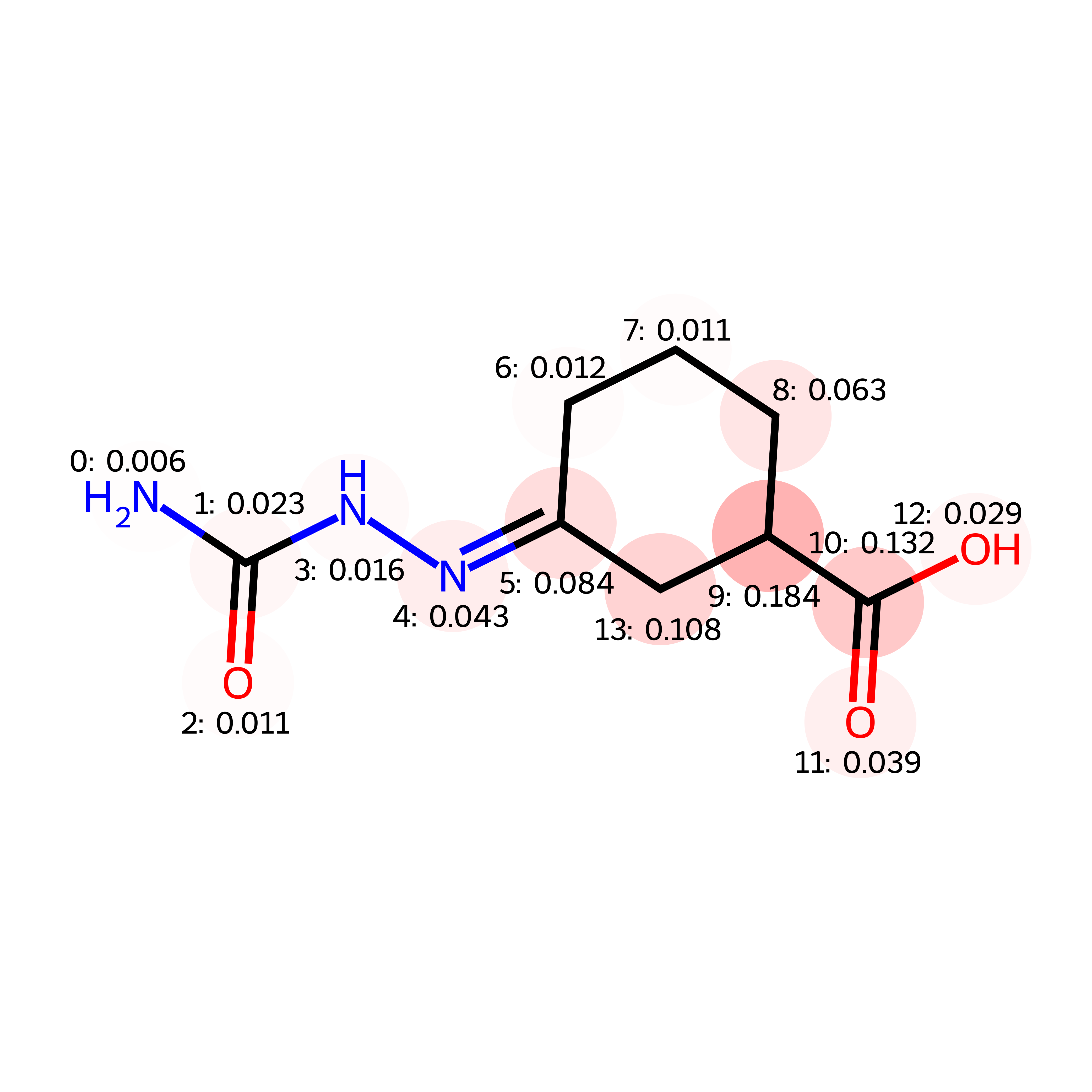}
        \caption{Molecule from Figure \ref{fig:pattern_5_fail}, with atom indices and IG attributions annotated.}
        \label{fig:pattern_5_idx}
\end{subfigure}%
\hfill 
\begin{subfigure}[t]{0.42\linewidth}
    \centering
    \input{figures/pattern_5_interpolation_impact}
    \caption{Impact of interpolation of atom 7 while other atoms and edges are interpolated at \(\alpha = 0.98\).}
    \label{fig:pattern5-more-interpolations}
\end{subfigure}
    \caption{Case study of IG's failure in Pattern 5. Figure \ref{fig:pattern_5_idx} represents node indices and their respective attributions, while Figure \ref{fig:pattern5-more-interpolations} shows some of activations impacted by interpolation of the Atom 7 when all elements are interpolated at \(\alpha=0.98\).}
\end{figure}

Ultimately, the cause of failure of Integrated Gradients in some of Pattern 5's examples is very similar to the one presented for Pattern 4 (Section \ref{sec:ig_in_pattern4_deep}). As such, we will only describe a chain of perturbation propagation contributing to IG's unfaithfulness. We present indices of atoms belonging to the molecule in Figure \ref{fig:pattern_5_idx}. 
\begin{enumerate}
    \item \(\varepsilon_0\)-perturbation of Atom 7 (not belonging to motif) causes \(\varepsilon_1\)-perturbation of D1 activation in the Atom 8 for sufficiently large \(\alpha\).
    \item \(\varepsilon_1\)-perturbation of D1 activation in Atom 8 causes \(\varepsilon_2\)-perturbation of C2 activation in Atom 9 for sufficiently large \(\alpha\).
    \item \(\varepsilon_2\)-perturbation of C2 activation in Atom 9 causes \(\varepsilon_3\)-perturbation of B3 activation in Atom 9 for sufficiently large \(\alpha\).
    \item \(\varepsilon_3\)-perturbation of B3 activation in Atom 9 causes \(\varepsilon_4\)-perturbation of A4 activation in Atom 9 for sufficiently large \(\alpha\).
    \item \(\varepsilon_4\)-perturbation of A4 activation in Atom 9 causes \(\varepsilon_5\)-perturbation of A5 activation in Atom 13 for sufficiently large \(\alpha\).
    \item \(\varepsilon_5\)-perturbation of A5 activation in Atom 13 causes \(\varepsilon_6\)-perturbation of A6 activation in Atom 13 for sufficiently large \(\alpha\).
    \item \(\varepsilon_6\)-perturbation of A6 activation in Atom 13 causes \(\varepsilon_7\)-perturbation of \(r_2\) readout for sufficiently large \(\alpha\), yielding a positive gradient.
\end{enumerate}
We plot values of these activations in Figure \ref{fig:pattern5-more-interpolations}.

%% file: figures/ig_misattributions.tex
\begin{figure}
    \begin{subfigure}{0.45\linewidth}
        \includegraphics[width=\linewidth]{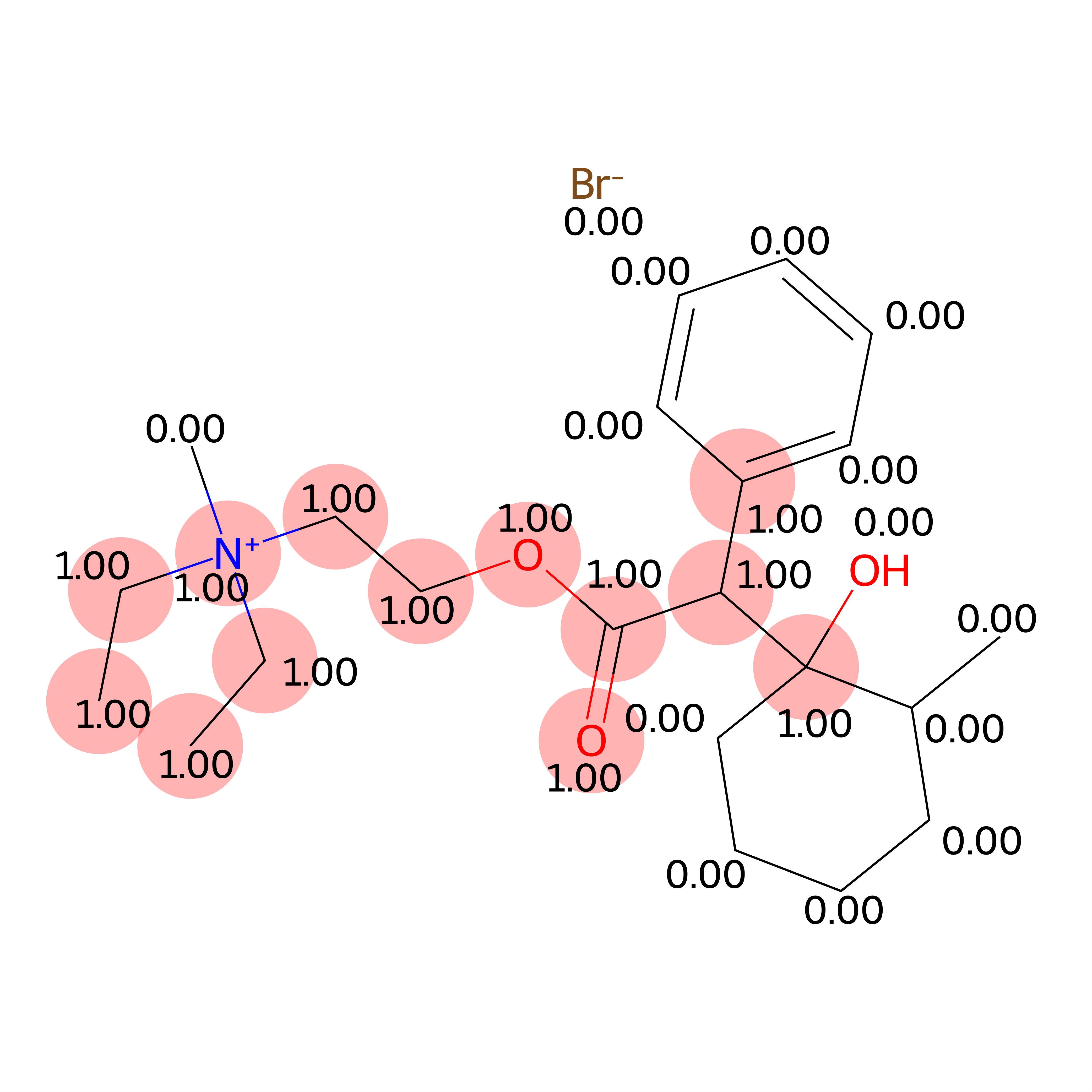}
        \caption{Ground truth}
    \end{subfigure}%
    \hfill 
    \begin{subfigure}{0.45\linewidth}
        \includegraphics[width=\linewidth]{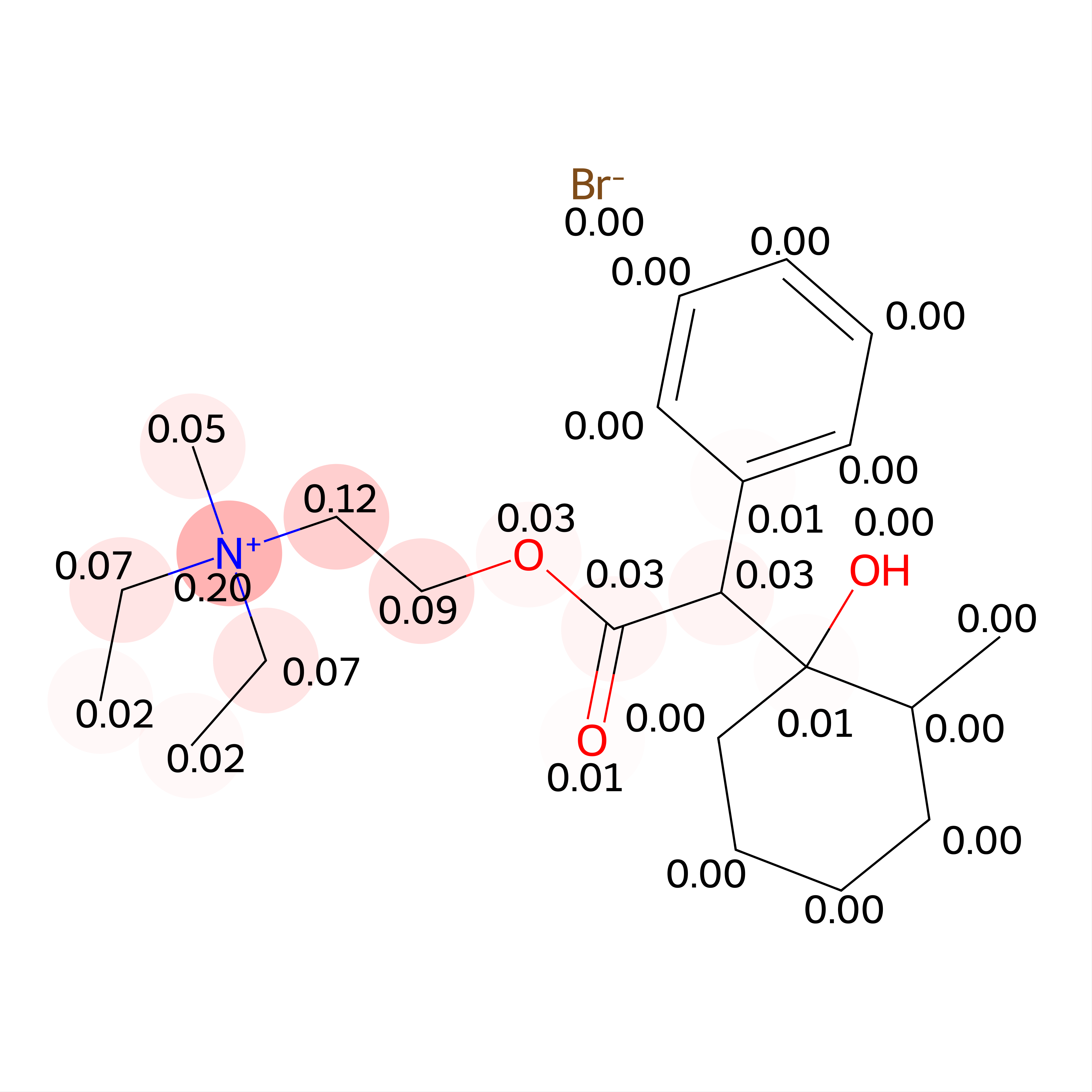}
                \caption{IG's imperfect attribution.}

    \end{subfigure}%
    \caption{Ground truth attribution for one of molecules containing Pattern 4 vs. IG's attributions. IG is unfaithful to the model by attributing the \(\mathrm{CH}_3\) carbon neighbouring the nitrogen ion, even though it does not belong to the ground truth attribution.}
    \label{fig:pattern_4_fails}
\end{figure}

%% file: figures/pattern_4_idxs_and_interp.tex
\begin{figure}
    \centering

    \begin{subfigure}[t]{0.45\linewidth}
        \begin{center}
            \includegraphics[width=0.95\linewidth]{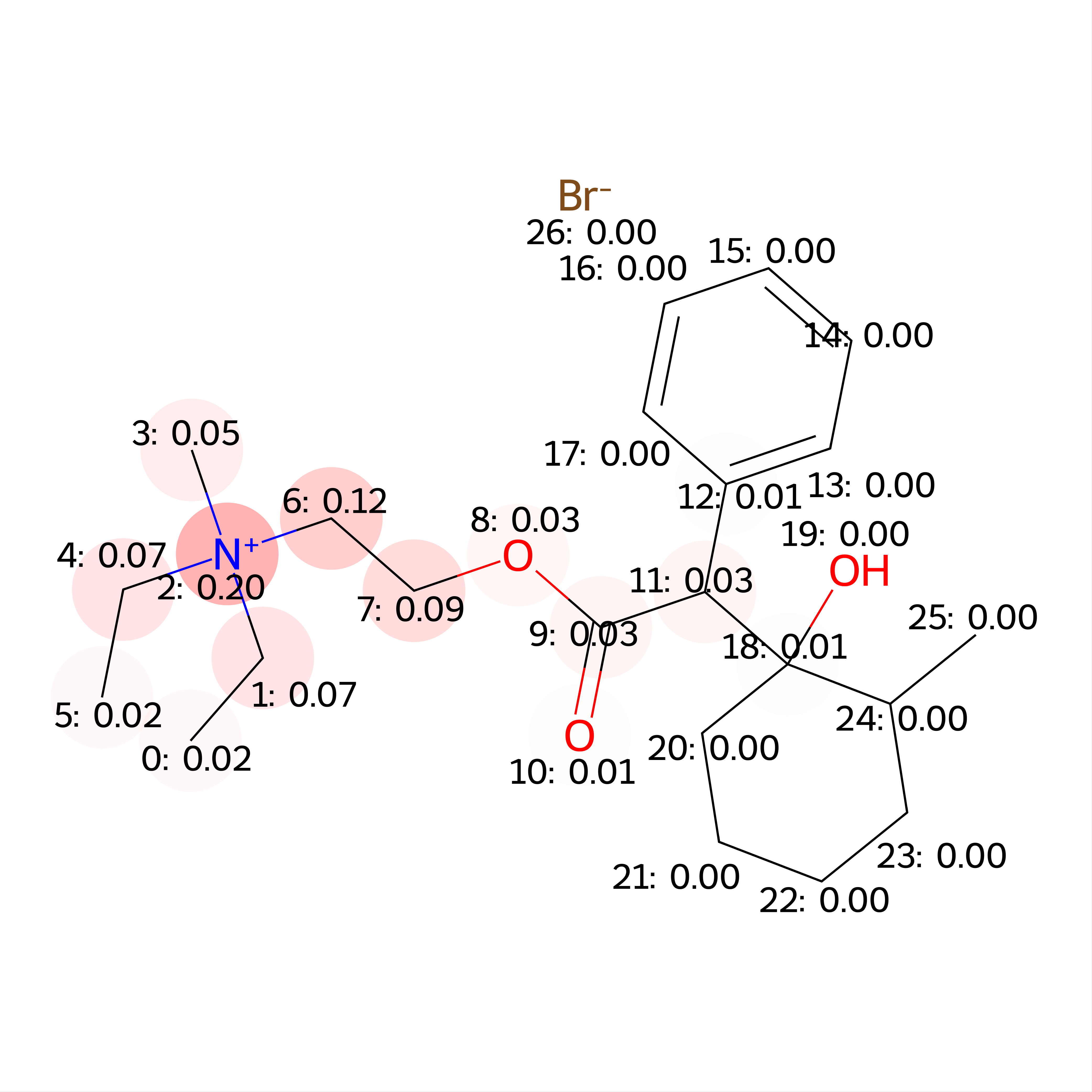}
        \end{center}
        \caption{Molecule from Figure \ref{fig:pattern_4_fails}, with atom indices and IG attributions annotated.}
        \label{fig:pattern_4_idx}
    \end{subfigure}
    \hfill 
    \begin{subfigure}[t]{0.45\linewidth}
        {
\definecolor{blue}{HTML}{0072B2}
\definecolor{skyblue}{HTML}{56B4E9}
\definecolor{green}{HTML}{009E73}
\definecolor{amber}{HTML}{E69F00}
\definecolor{vermillion}{HTML}{D55E00}
\definecolor{purple}{HTML}{8E5AA9}
\definecolor{ink}{HTML}{17212B}
    \begin{center}
        \begin{tikzpicture}[scale=0.6]
    \begin{axis}[
        title={},
        xlabel={Atom 3 Interpolation coefficient (\(\alpha'\)); other atoms at \(\alpha=0.97\)},
        ylabel={Activation},
        legend pos=north west,
        ymajorgrids=true,
        xmajorgrids=true,
        grid style=dashed,
        scaled y ticks=false,
        xmin=0.97,
        ymax=1,
        legend cell align=left,
        /pgf/number format/fixed,
        /pgf/number format/precision=3,
    ]

    \addplot[
            color=blue,
            mark=none,
            very thick,
            ] table [col sep=comma]{figures/data_pattern_4/c12_activations_vs_scaling_atom_3.csv};
        \addlegendentry{C1-activation of atom 2}

    \addplot[
            color=skyblue,
            mark=none,
            very thick,
            ] table [col sep=comma]{figures/data_pattern_4/b22_activations_vs_scaling_atom_3.csv};
        \addlegendentry{B2-activation of atom 2}

    \addplot[
            color=green,
            mark=none,
            very thick,
            ] table [col sep=comma]{figures/data_pattern_4/a36_activations_vs_scaling_atom_3.csv};
        \addlegendentry{A3-activation of atom 6}

    \addplot[
            color=amber,
            mark=none,
            very thick,
            ] table [col sep=comma]{figures/data_pattern_4/a47_activations_vs_scaling_atom_3.csv};
        \addlegendentry{A4-activation of atom 7}

    \addplot[
            color=ink,
            mark=none,
            very thick,
            ] table [col sep=comma]{figures/data_pattern_4/full_predictions_vs_scaling_atom_3.csv};
        \addlegendentry{Network prediction (\(r_2\))}

    \end{axis}
    \end{tikzpicture}
    \end{center}
    \hfill 
    \caption{Impact of interpolation of atom 3 while other atoms and edges are interpolated at \(\alpha = 0.97\).}
    \label{fig:pattern4-more-interpolations}
}
    \end{subfigure}
    
    \caption{Atom indexes (Figure \ref{fig:pattern_4_idx}) and impact of interpolation of atom 3's representation (Figure \ref{fig:pattern4-more-interpolations}).}
\end{figure}

%% file: figures/pattern_4_ig.tex
\begin{figure}[t]
\begin{subfigure}[t]{0.3\linewidth}
    \begin{tikzpicture}[scale=0.45]
    \begin{axis}[
        title={},
        xlabel={Interpolation (\%)},
        ylabel={Gradient},
        legend pos=north west,
        ymajorgrids=true,
        xmajorgrids=true,
        grid style=dashed,
        scaled y ticks=false,
        xmin=0.88,
        legend cell align=left,
        /pgf/number format/fixed,
        /pgf/number format/precision=3,
    ]

    \addplot[
        color=red,
        mark=none,
        very thick
        ] table [col sep=comma]{figures/data_pattern_4/atom_3_feats_gradients.csv};
    \addlegendentry{Gradients of the atom's type OHE entry}

    \addplot[
            color=blue,
            mark=none,
            very thick
            ] table [col sep=comma]{figures/data_pattern_4/atom_3_hydrogen_feats_gradients.csv};
        \addlegendentry{Gradients of the atom's hydrogen count OHE entry}
        
    \end{axis}
    \end{tikzpicture}
    \caption{Gradients of the non-motif atom during IG interpolation.}
    \label{subfig:pattern4-ig-main-PATH}
    \end{subfigure}%
    \hfill 
    \begin{subfigure}[t]{0.3\linewidth}
        \begin{tikzpicture}[scale=0.45]
    \begin{axis}[
        title={},
        xlabel={Interpolation (\%)},
        ylabel={Network's prediction},
        legend pos=north west,
        ymajorgrids=true,
        xmajorgrids=true,
        grid style=dashed,
        scaled y ticks=false,
        legend cell align=left,
        /pgf/number format/fixed,
        /pgf/number format/precision=3,
    ]

    \addplot[
        color=red,
        mark=none,
        very thick
        ] table [col sep=comma] {figures/data_pattern_4/predictions_vs_scaling_atom_3_r1.csv};
        
    \addlegendentry{\(r_1\) readout}
    
      \addplot[
        color=blue,
        mark=none,
        very thick
        ] table [col sep=comma]{figures/data_pattern_4/predictions_vs_scaling_atom_3_r2.csv};
        \addlegendentry{\(r_2\) readout}
    \end{axis}
    \end{tikzpicture}
    \caption{Network's prediction during interpolation of the non-motif atom.}
    \label{subfig:pattern4-ig-main-READOUTS}
    \end{subfigure}%
    \hfill 
        \begin{subfigure}[t]{0.3\linewidth}
    \begin{tikzpicture}[scale=0.45]
    \begin{axis}[
        title={},
        xlabel={Interpolation coefficient},
        ylabel={Activation},
        legend pos=north west,
        ymajorgrids=true,
        xmajorgrids=true,
        grid style=dashed,
        scaled y ticks=false,
        xmin=0.97,
        ymax=1.1,
        legend cell align=left,
        /pgf/number format/fixed,
        /pgf/number format/precision=3,
    ]

    \addplot[
        color=red,
        mark=none,
        very thick
        ] table [col sep=comma]{figures/data_pattern_4/atom_2_activations_vs_scaling_atom_3.csv};
    \addlegendentry{Activation of the neighbouring atom}

\addplot[
        color=blue,
        mark=none,
        very thick
        ] table [col sep=comma]{figures/data_pattern_4/full_predictions_vs_scaling_atom_3.csv};
    \addlegendentry{Network prediction (\(r_2\))}
    
    \end{axis}
    \end{tikzpicture}
    \caption{Interpolation of the non-motif atom while other atoms and edges are interpolated at \(\alpha=0.97\).}
    \label{subfig:pattern4-ig-main-IMPACT}
    \end{subfigure}
    \hfill 
    \label{fig:pattern4-ig-main}
    \caption{Analysis of causes of spurious attribution of the non-motif atom in Pattern 4. Figure \ref{subfig:pattern4-ig-main-PATH} shows gradients that are accumulated by this atom near the decision boundary. Figure \ref{subfig:pattern4-ig-main-READOUTS} demonstrates that, regardless of atom's presence, network makes a correct classification. However, interpolating atom between states causes model's prediction to drop to 0. Figure \ref{subfig:pattern4-ig-main-IMPACT} showcases impact interpolation of the non-motif atom on activation of neighbouring node and final network's prediction under readout \(r_2\), providing a justification for curve observed in the Figure \ref{subfig:pattern4-ig-main-PATH}.}
\end{figure}

%% file: figures/pattern_5_interpolation_impact.tex
{
\definecolor{blue}{HTML}{0072B2}
\definecolor{skyblue}{HTML}{56B4E9}
\definecolor{green}{HTML}{009E73}
\definecolor{amber}{HTML}{E69F00}
\definecolor{vermillion}{HTML}{D55E00}
\definecolor{purple}{HTML}{8E5AA9}
\definecolor{ink}{HTML}{17212B}

\begin{tikzpicture}[scale=0.5]
    \begin{axis}[
        title={},
        xlabel={Atom 7 Interpolation coefficient (\(\alpha'\)); other atoms at \(\alpha=0.98\)},
        ylabel={Activation (log scale)},
        legend pos=north west,
        ymajorgrids=true,
        xmajorgrids=true,
        grid style=dashed,
        scaled y ticks=false,
        xmin=0.98,
        ymax=1.2,
        legend cell align=left,
        /pgf/number format/fixed,
        /pgf/number format/precision=2,
        ymode=log,
        log ticks with fixed point,
    ]

    \addplot[
            color=blue,
            mark=none,
            very thick,
            ] table [col sep=comma]{figures/data/d18_activations_vs_scaling_atom_7.csv};
        \addlegendentry{D1-activation of atom 8}

    \addplot[
            color=skyblue,
            mark=none,
            very thick,
            ] table [col sep=comma]{figures/data/c29_activations_vs_scaling_atom_7.csv};
        \addlegendentry{C2-activation of atom 9}
        
    \addplot[
            color=green,
            mark=none,
            very thick,
            ] table [col sep=comma]{figures/data/b39_activations_vs_scaling_atom_7.csv};
        \addlegendentry{B3-activation of atom 9}   

    \addplot[
            color=amber,
            mark=none,
            very thick,
            ] table [col sep=comma]{figures/data/a49_activations_vs_scaling_atom_7.csv};
        \addlegendentry{A4-activation of atom 9}  

    \addplot[
            color=vermillion,
            mark=none,
            very thick,
            ] table [col sep=comma]{figures/data/a513_activations_vs_scaling_atom_7.csv};
        \addlegendentry{A5-activation of atom 13}

    \addplot[
            color=purple,
            mark=none,
            very thick,
            ] table [col sep=comma]{figures/data/a613_activations_vs_scaling_atom_7.csv};
        \addlegendentry{A6-activation of atom 13}

    \addplot[
            color=ink,
            mark=none,
            very thick,
            ] table [col sep=comma]{figures/data/full_predictions_vs_scaling_atom_7.csv};
        \addlegendentry{Network prediction (\(r_2\))}
    
    \end{axis}
\end{tikzpicture}
}

%% file: appendix_pge_gnn.tex
\section{Additional Results on Failure Modes of Shapley Value Sampling, PGExplainer, and GNNExplainer}

\input{figures/pge_logits_pattern_12}
\input{figures/failures_of_pge_on_pattern_12}

\label{sec:app_pge_gnn} Both PGExplainer and GNNExplainer perform poorly on WOMBAT, even when evaluating simple patterns. At the same time, Shapley Value Sampling performs reasonably, although its performance often yields imperfect attributions. We focus our analysis on models utilizing the readout in Eq.~\ref{eqn:readout_2}, as its bounded output ($[0, 1]$) should theoretically simplify the attribution task. 

\subsection{Shapley Value Sampling}

Although Shapley Value Sampling (SVS) maintains relatively high AUROC across $r_1$ and $r_2$ readouts, it frequently fails to achieve perfect node ranking, a trend that intensifies with motif complexity. This is most visible for Pattern 3 under the $r_2$ readout, where SVS achieves perfect attribution ($\mathrm{AUROC}=1.0$) for only 3 out of 22 positive molecules (Table \ref{tab:shap_perfection}). We empirically verify that  increasing SVS's sampling budget improves its performance, especially across bigger motifs (see Figure \ref{fig:svs_performance}).
\input{figures/shap_improvements}
\input{tables/shap_perfection}

\subsection{PGExplainer}

\input{figures/activated_nodes_end}
\paragraph{Edge classification.} PGExplainer relies on a Multilayer Perceptron (MLP) to predict attributions of edges. Let \(z_1, z_2, z_3, \dots, z_n\) be model's latent representations of nodes \(v_1, v_2, v_3, \dots, v_n\) after the last MPNN layer (before the readout). To predict attribution for edge \(e_{ij}\), it first requires calculation of the edge logit:

\begin{equation}
    \label{eq:mlp_eq}
    \omega_{ij} = \mathrm{MLP}([z_i;z_j])
\end{equation}

for more details, please refer to~\cite{luo2020parameterized}. In WOMBAT, the final-layer representations \(z_i\) are frequently zero vectors, including nodes belonging to the target pattern. On average, in WOMBAT's whiteboxes a vast majority of nodes is zero-activated after the last layer (Figure~\ref{fig:activated_nodes_end:A}). Consequently, many edges receive exactly the same input
\([0;0]\) in Equation~\ref{eq:mlp_eq} and must therefore receive the same score. In particular, if a pattern edge and an irrelevant edge both connect pairs of zero-activation nodes, PGExplainer cannot distinguish between them, irrespective of how successfully its explanation network is optimized. Edges incident to the active node may remain distinguishable from edges whose two endpoints have zero representations. This gives PGExplainer some capacity to recover the portion of a pattern lying in the immediate neighbourhood of that node. Perfect separation of motif and non-motif edges is impossible however whenever the ground-truth pattern contains zero--zero edges while the surrounding graph also contains irrelevant zero--zero edges. This ambiguity is consistent with our empirical results: under the main evaluation protocol, PGExplainer achieved reasonable, although imperfect, performance only on Pattern~12, one of the
smallest evaluated patterns, consisting of three nodes.

We establish an approximate upper bound on PGExplainer's performance in Figure~\ref{fig:activated_nodes_end:B}. We construct it as follows:
\begin{enumerate}
    \item Assign a score of 1 to every ground truth motif node
    whose final-layer representation is non-zero, and to every
    ground-truth motif node adjacent to such a node.
    \item Assign a score of 0 to all remaining nodes.
\end{enumerate}
Note that we assume that PGExplainer assigns zero attribution to an edge when both endpoints have zero final layer representations. This does not necessarily have to be true, but given that a vast majority of nodes tend to be zero-activated (see Figure~\ref{fig:activated_nodes_end:A}), we believe it to be a sensible assumption. As such, we believe that this approximation should be treated as a practical (and loose) upper bound. It can be observed that this bound does not impact smaller Patterns (such as Pattern 12, for which PGExplainer achieves a decent performance), but caps PGExplainer's performance severely in case of Patterns such as Pattern 2 or 3.

\paragraph{PGExplainer in Pattern 12.} \label{sec:appendix_pge_fail_p12} While PGExplainer manages to explain majority of molecules containing Pattern 12, there are a few cases where it fails to score ideal AUROC score. Furthermore, it also fails the IQR test for 1 of the molecules in the negative sample test. All of these failures are caused by the presence of formyl anions in the graphs (see Figure \ref{fig:failures_of_pge_on_pattern_12}). To understand this, one needs to consider weights of the Pattern 12 whitebox. It contains only 1 MPNN layer with the following activation definitions\footnote{For activation notation in WOMBAT, please refer to Section \ref{sec:notation}.}:
\begin{minipage}{0.55\linewidth}
\begin{enumerate}
    \item \textbf{A1:} Carbon with exactly 1 hydrogen, that is connected via double bond to at least one oxygen atom.
    \item \textbf{B1:} Carbon with exactly 1 hydrogen, that is connected via a single bond to at least one other atom.
\end{enumerate}
\end{minipage}%
\begin{minipage}{0.1\linewidth}
    \begin{center}
        \rule{1pt}{100pt}
    \end{center}
\end{minipage}%
\begin{minipage}{0.35\linewidth}
    \begin{center}
    \includegraphics[width=\linewidth]{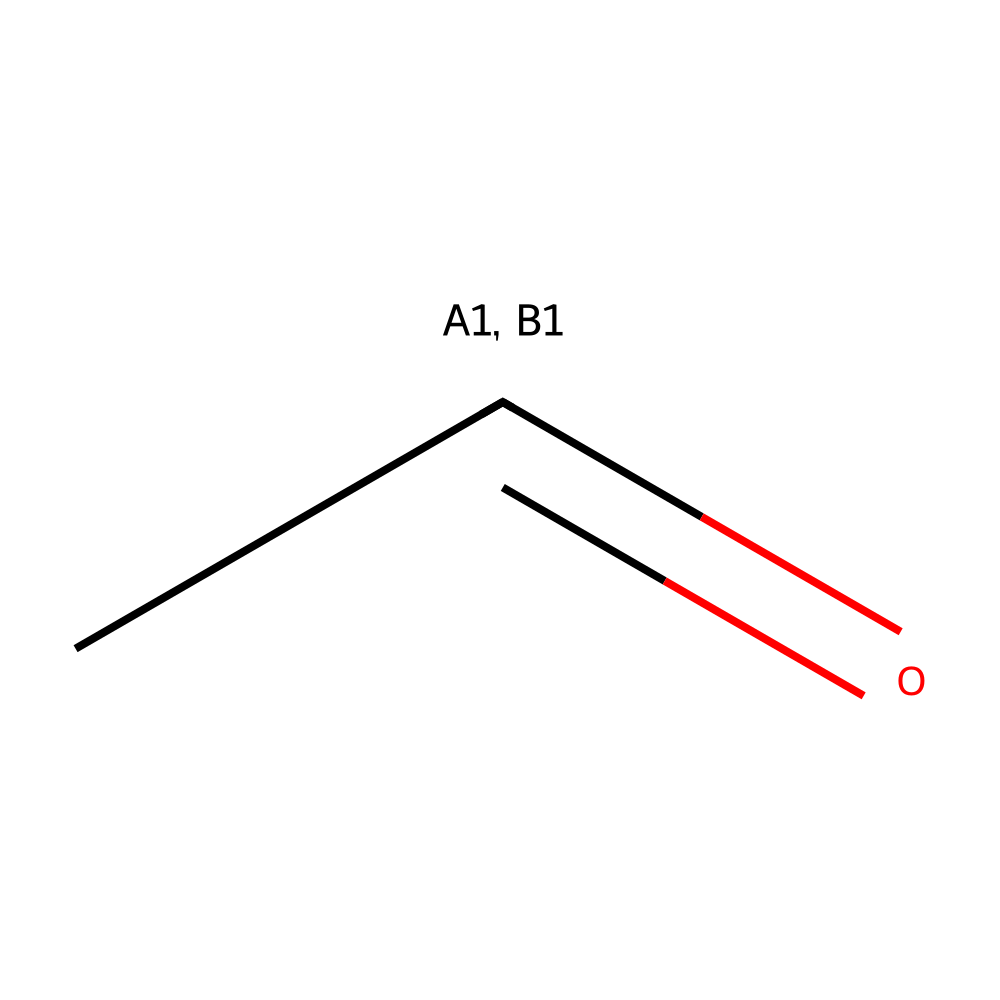}
    \end{center}
\end{minipage}

  The readout effectively computes \(\mathrm{A1}\land\mathrm{B1}\). However, for a carbon atom with conventional valence, A1 generally implies B1. A1 specifies a carbon atom with a double bond to oxygen and a bond to exactly one hydrogen. Such a carbon usually has an additional single bond to another atom and therefore also activates B1. Accordingly, \(\mathrm{A1}\implies\mathrm{B1}\) holds for most examples in the training dataset, although counterexamples are present. In turn, PGExplainer's MLP instead of learning to check for \(\mathrm{A1} \land \mathrm{B1}\) in one of edge's terminals checks for \(\mathrm{A1}\); in other words, presence of \(\mathrm{A1}\)-activation in one of edge's terminals is necessary and sufficient for edge attribution by PGExplainer (see Figure \ref{fig:pge_logits_pattern_12}). This shortcut leads to PGExplainer failing by attributing formyl anions, which do not contain Pattern 12, but they do exhibit \(\mathrm{A1}\)-activation (see Figure \ref{fig:act_formyl_anion_pattern_12}).  

\begin{minipage}{0.45\linewidth}
\paragraph{Optimisation.} As PGExplainer requires a separate training phase, it is highly sensitive to hyperparameters such as learning rate and epoch count. We conducted a grid search on Pattern 14, one of the simplest patterns the explainer failed to identify, to assess this sensitivity (Figure~\ref{fig:pgexplainer_grid}). The results indicate a tendency for the model to fail to converge. To test stability, we executed seven independent runs using the best-performing hyperparameters from the grid search ($\mathrm{lr}=0.1$, $20$ epochs); however, the model successfully converged in only one instance.
\end{minipage}%
\begin{minipage}{0.1\linewidth}
    \hfill 
\end{minipage}%
\begin{minipage}{0.45\linewidth}
    \includegraphics[width=0.9\linewidth]{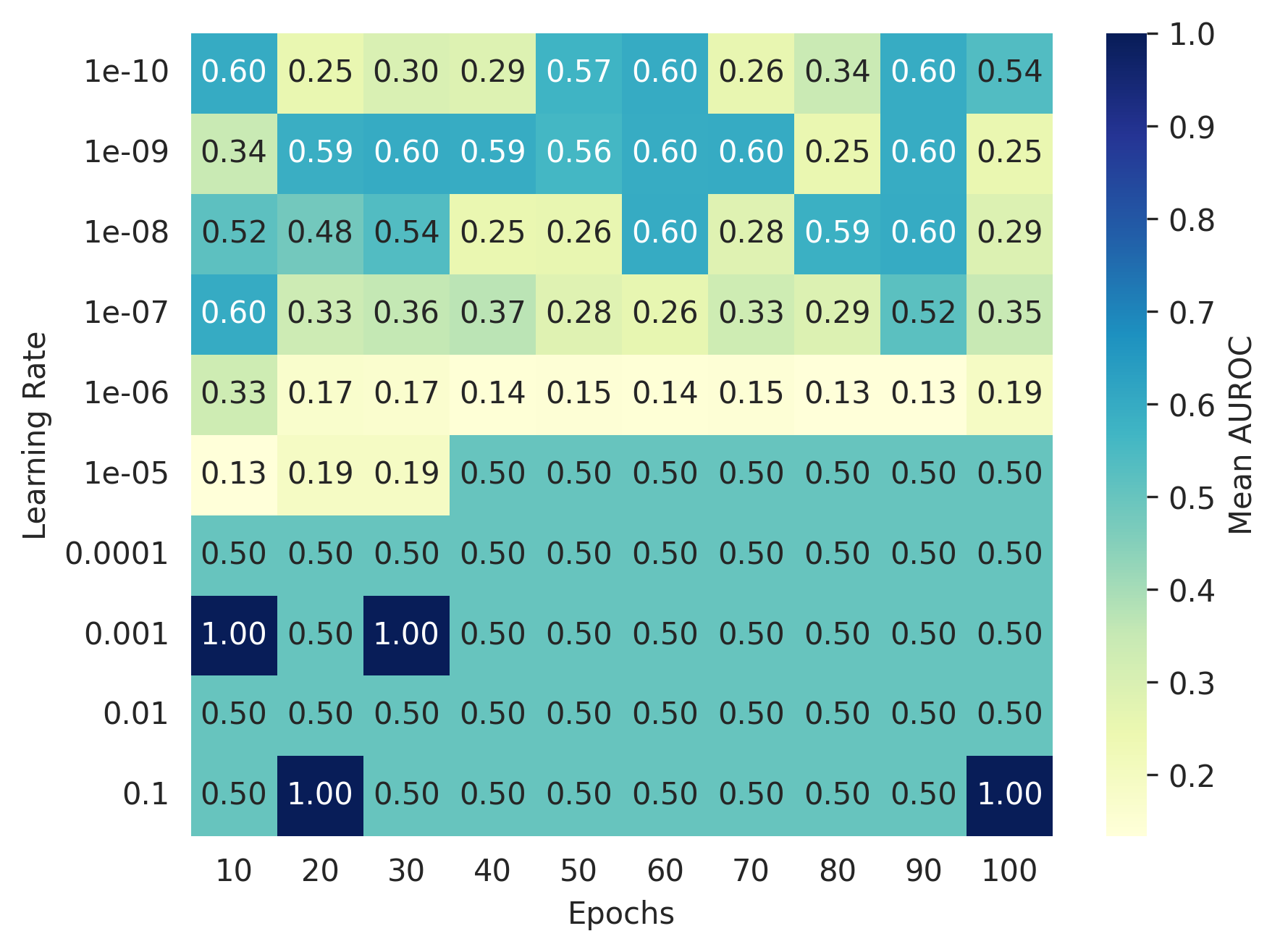}
    \captionof{figure}{Results of performing grid search for training parameters of PGExplainer for Pattern 14 with readout \(r_2\). Each value comes from a fully separate training run.}
    \label{fig:pgexplainer_grid}
\end{minipage}

\subsection{GNNExplainer}

\paragraph{Optimisation.} Similarly to PGExplainer, GNNExplainer is sensitive to hyperparameter selection. We performed an exhaustive random search across 50,000 parameter combinations to determine if it could correctly attribute a single positive molecule for Pattern 14. Surprisingly, not a single combination was a successful attribution.

%% file: figures/pge_logits_pattern_12.tex
\begin{figure}
    \centering
    \begin{subfigure}[t]{0.6\linewidth}
            \includegraphics[width=\linewidth]{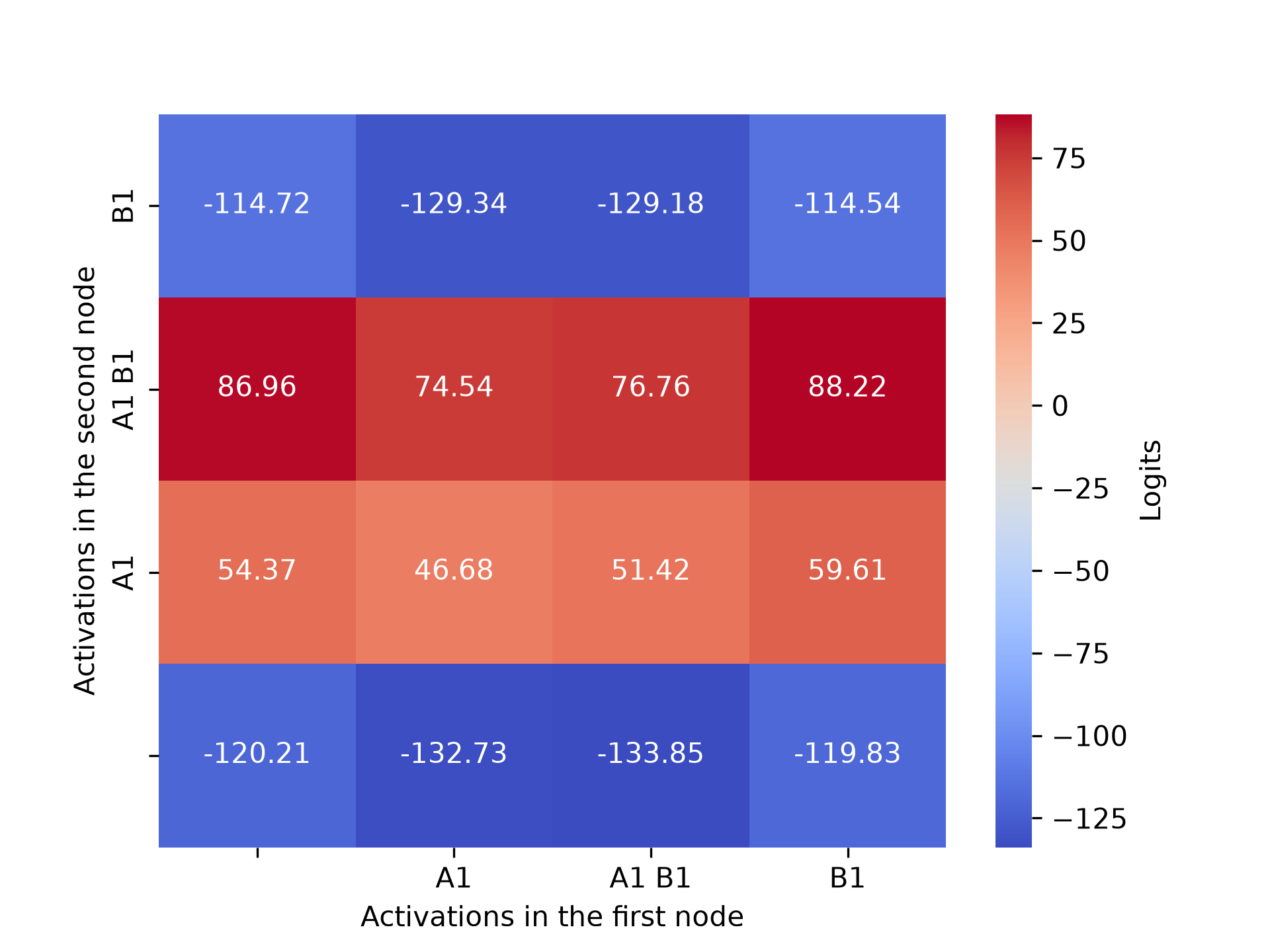}
            \caption{Logits returned by MLP trained by PGExplainer to explain Pattern 12 depending on activations of the last layer in Pattern 12 whitebox.}
    \label{fig:pge_logits_pattern_12}
    \end{subfigure}\hfill 
    \begin{subfigure}[t]{0.35\linewidth}
    \centering 
        \includegraphics[width=\linewidth]{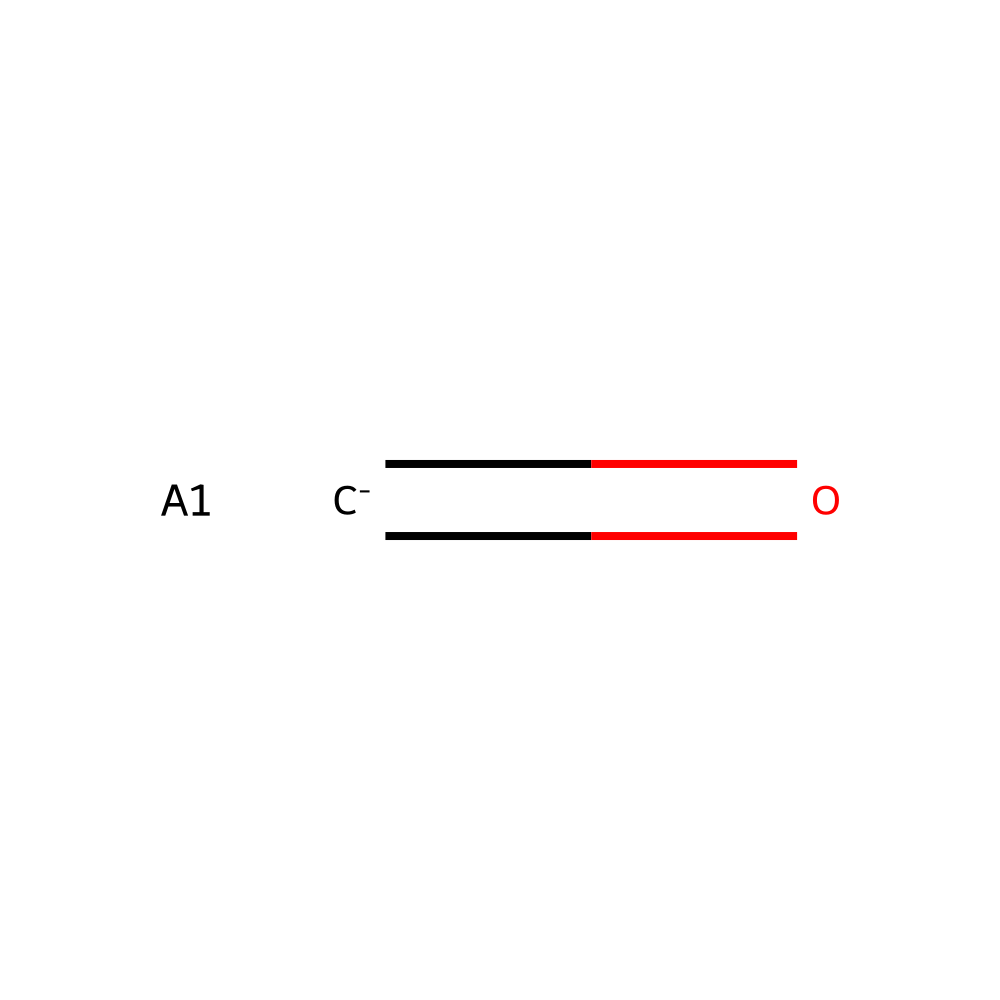}
    \caption{Activations of Pattern 12 whitebox for a formyl anion.}
    \label{fig:act_formyl_anion_pattern_12}
    \end{subfigure}
    \caption{Figure \ref{fig:pge_logits_pattern_12} demonstrates that while presence of B1 activation in one of edges terminals increases logit values, A1 activation is sufficient and necessary for PGExplainer to attribute the edge connecting 2 nodes. As such, carbons belonging to formyl anions, despite only being \(\mathrm{A1}\)-activated (Figure \ref{fig:act_formyl_anion_pattern_12}) still get attributed by PGExplainer.}
\end{figure}

%% file: figures/failures_of_pge_on_pattern_12.tex
\begin{figure}[htbp]
    \centering
    % ROW 1
    \begin{subfigure}[t]{0.45\linewidth}
        \centering
        \includegraphics[width=\linewidth]{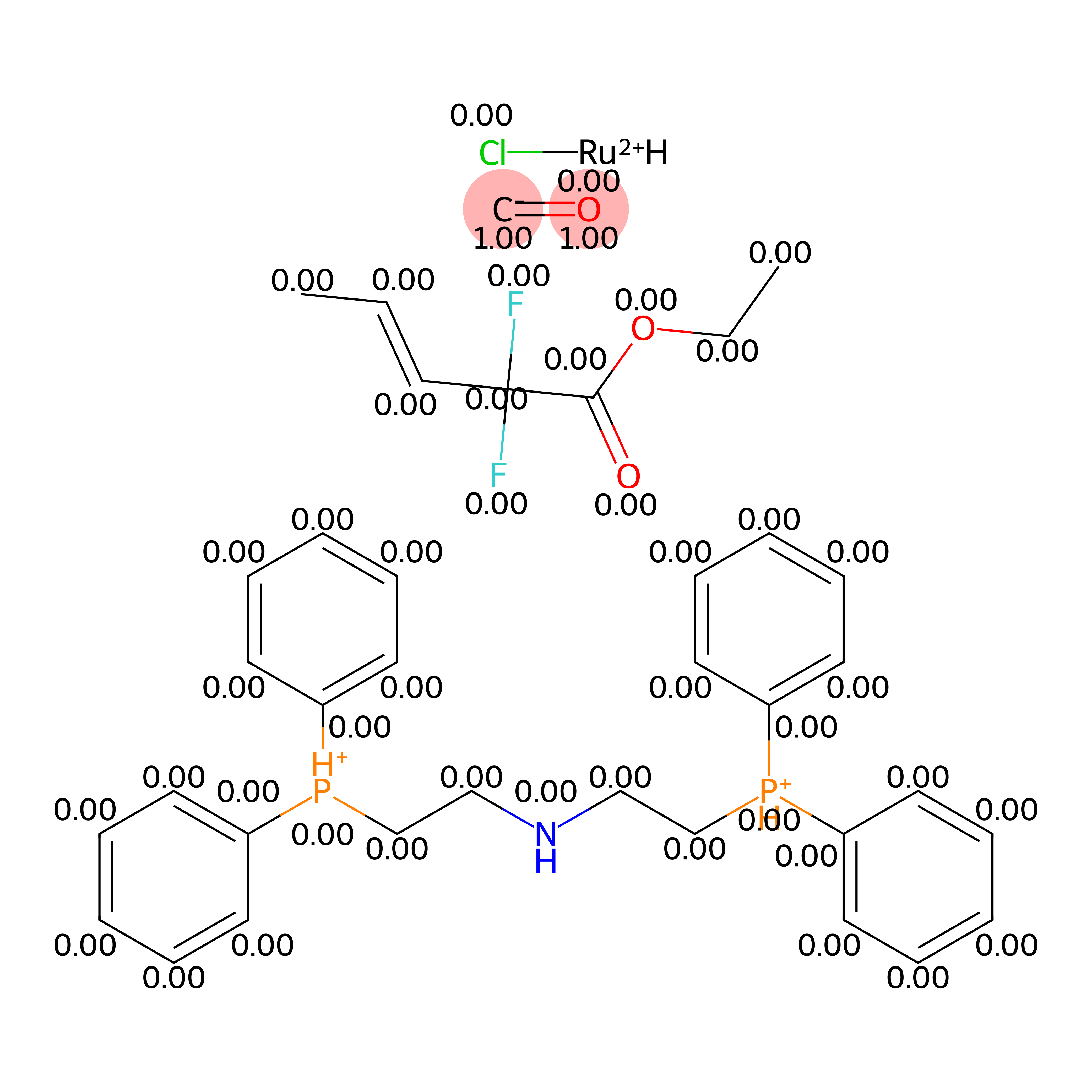}
        \caption{Sole failing negative sample.}
        \label{fig:failures_of_pge_on_pattern_12_pattern_12:NEGATIVE}
    \end{subfigure}\hfill
    \begin{subfigure}[t]{0.45\linewidth}
        \centering
        \includegraphics[width=\linewidth]{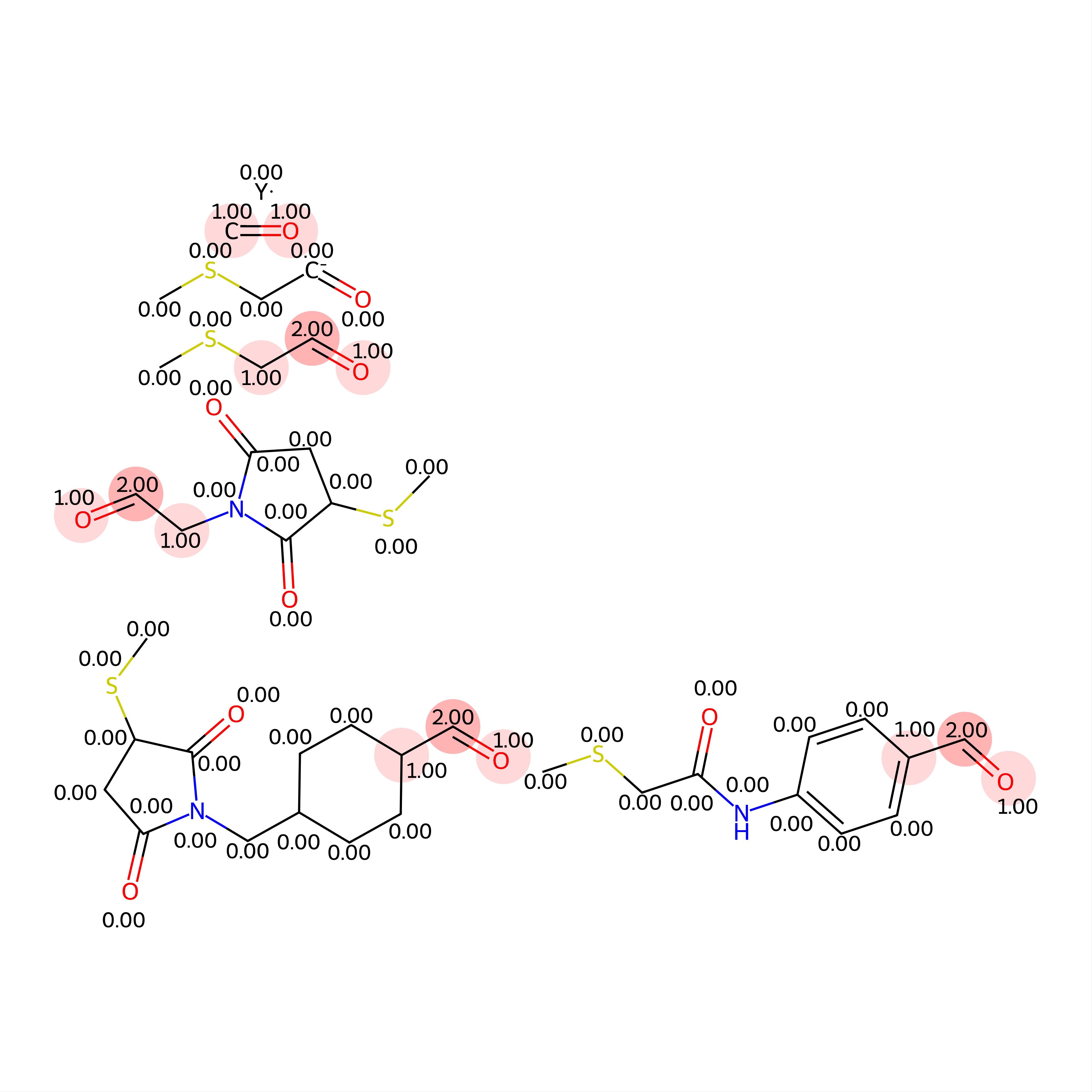}
        \caption{One of 2 failing positive samples.}
        \label{fig:failures_of_pge_on_pattern_12_negative_sample:POSITIVE}
    \end{subfigure}

    \caption{In case of both negative (Figure \ref{fig:failures_of_pge_on_pattern_12_pattern_12:NEGATIVE}) and positive samples (Figure \ref{fig:failures_of_pge_on_pattern_12_negative_sample:POSITIVE}), PGExplainer fails by misattributing formyl anions. In the former, this causes IQR criterion failure (as null explanation was expected); in the latter, formyl anions are attributed alongside legitimate occurrences of Pattern 12 in molecule, causing AUROC to drop.}
    \label{fig:failures_of_pge_on_pattern_12}
\end{figure}

%% file: figures/shap_improvements.tex
\begin{figure}
\begin{subfigure}[t]{0.45\linewidth}
    \begin{tikzpicture}[scale=0.6]
    \begin{axis}[
        ymajorgrids=true,
        xmajorgrids=false,
        grid style={dashed},
    	ylabel={SHAP AUROC Perfection rate},
        xlabel={Pattern Size},
    	enlarge x limits=0.05,
        legend pos=south east,
    ]
    \addplot[
    only marks,
    mark=*,
    mark options={
        draw=blue!60,
        fill=blue!60
    }
    ] table[
        x=Pattern Size,
        y=SHAP 100,
        col sep=comma
    ] {figures/data_other/shap_improvement.csv};
    \end{axis}
    \end{tikzpicture}
\caption{Performance of Shapley Value Sampling degrades relative to the pattern size.}
\label{fig:svs_performance:A}
\end{subfigure}
\hfill 
\begin{subfigure}[t]{0.45\linewidth}
    \begin{tikzpicture}[scale=0.6]
    \begin{axis}[
        ybar,
        bar width=5pt,
        ymajorgrids=true,
        xmajorgrids=false,
        grid style={dashed},
        xtick={0,1,...,13},
        xticklabels={1,2,...,14},
    	ylabel={Ideal AUROC rate},
        xlabel={Pattern Number},
    	enlarge x limits=0.05,
        legend pos=south east,
    ]
    \addplot table[
        x=Index,
        y=SHAP 100,
        col sep=comma
    ] {figures/data_other/shap_improvement.csv};
    \addplot table[
        x=Index,
        y=SHAP 500,
        col sep=comma
    ] {figures/data_other/shap_improvement.csv};

    \legend{100 samples, 500 samples}
    \end{axis}
    \end{tikzpicture}
\caption{Increasing number of samples helps to recover perfectly majority of motifs.}
\label{fig:svs_performance:B}
\end{subfigure}
\caption{Shapley Value Sampling's ideal attribution rate is dependent on the size of the motif (Figure \ref{fig:svs_performance:A}). Increasing of the sample size helps it achieve better performance (Figure \ref{fig:svs_performance:B}). Here, we define the ``Pattern Size'' as number of heavy atoms in a molecule after substituting all wildcards as carbons.}
\label{fig:svs_performance}
\end{figure}

%% file: tables/shap_perfection.tex
\begin{table}
\scriptsize
\centering
\caption{Comparison of amount of perfect (i.e. equal to 1.0) AUROCs for Shapley Value Sampling and Integrated Gradients.}
\label{tab:shap_perfection}
\begin{tabular}{ccccccc}
\toprule
\multirow{2}{*}{\textbf{Pattern No.}} & \multicolumn{3}{c}{\textbf{Shapley Value Sampling}} & \multicolumn{3}{c}{\textbf{Integrated Gradients}}  \\
\cmidrule(lr){2-4} \cmidrule(lr){5-7}
& \textbf{Perfect} & \textbf{Total} & \textbf{Perfection Rate} & \textbf{Perfect} & \textbf{Total} & \textbf{Perfection Rate} \\
\toprule 
1 & 85 & 100 & 0.85 & 7621 & 7660 & 0.99 \\ 
2 & 57 & 94 & 0.61 & 236 & 236 & 1.00 \\
3 & 3 & 22 & 0.14 & 22 & 22 & 1.00 \\ 
4 & 60 & 100 & 0.60 & 964 & 1137 & 0.85 \\ 
5 & 20 & 31 & 0.65 & 29 & 31 & 0.94 \\ 
6 & 78 & 100 & 0.78 & 9984 & 9984 & 1.00 \\ 
7 & 97 & 100 & 0.97 & 9990 & 9990 & 1.00 \\
8 & 88 & 100 & 0.88 & 9996 & 9996 & 1.00 \\
9 & 97 & 100 & 0.97 & 9989 & 9989 & 1.00 \\ 
10 & 95 & 100 & 0.95 & 9989 & 9989 & 1.00 \\ 
11 & 98 & 100 & 0.98 & 9999 & 9999 & 1.00 \\ 
12 & 100 & 100 & 1.00 & 9999 & 9999 & 1.00 \\ 
13 & 99 & 100 & 0.99 & 9997 & 9997 & 1.00 \\ 
14 & 100 & 100 & 1.00 & 10000 & 10000 & 1.00 \\
 \bottomrule   
\end{tabular}    

\end{table}

%MoleculeCDetector (SHAP): 85/100 (85.00%)
%MoleculeCDetector (IG): 7621/7660 (99.49%)
%MoleculeEDetector (SHAP): 57/94 (60.64%)
%MoleculeEDetector (IG): 236/236 (100.00%)
%MoleculeFDetector (SHAP): 3/22 (13.64%)
%MoleculeFDetector (IG): 22/22 (100.00%)
%MoleculeGDetector (SHAP): 60/100 (60.00%)
%MoleculeGDetector (IG): 964/1137 (84.78%)
%MoleculeHDetector (SHAP): 20/31 (64.52%)
%MoleculeHDetector (IG): 29/31 (93.55%)
%MoleculeIDetector (SHAP): 78/100 (78.00%)
%MoleculeIDetector (IG): 9984/9984 (100.00%)
%MoleculeJDetector (SHAP): 97/100 (97.00%)
%MoleculeJDetector (IG): 9990/9990 (100.00%)
%MoleculeKDetector (SHAP): 88/100 (88.00%)
%MoleculeKDetector (IG): 9996/9996 (100.00%)
%MoleculeLDetector (SHAP): 97/100 (97.00%)
%MoleculeLDetector (IG): 9989/9989 (100.00%)
%MoleculeMDetector (SHAP): 95/100 (95.00%)
%MoleculeMDetector (IG): 9989/9989 (100.00%)
%MoleculeNDetector (SHAP): 98/100 (98.00%)
%MoleculeNDetector (IG): 9999/9999 (100.00%)
%MoleculeODetector (SHAP): 100/100 (100.00%)
%MoleculeODetector (IG): 9999/9999 (100.00%)
%MoleculePDetector (SHAP): 99/100 (99.00%)
%MoleculePDetector (IG): 9997/9997 (100.00%)
%MoleculeQDetector (SHAP): 100/100 (100.00%)
%MoleculeQDetector (IG): 10000/10000 (100.00%)

%% file: figures/activated_nodes_end.tex
\begin{figure}
    \centering
    \begin{subfigure}[t]{0.45\linewidth}
        \begin{tikzpicture}[scale=0.6]
\begin{axis}[
    ybar,
    bar width=5pt,
    ymajorgrids=true,
    xmajorgrids=false,
    grid style={dashed},
    xtick={0,1,...,13},
    xticklabels={1,2,...,14},
	ylabel={Activated Nodes},
    xlabel={Pattern Number},
	enlarge x limits=0.05,
    legend pos=north west,
]
\addplot [
    draw=blue,
    pattern=north east lines,
    bar shift=-4pt,
    pattern color=blue!30,
] table[
    x=index,
    y=Positive,
    col sep=comma,
] {figures/data_other/activation_total_means.csv};
\addplot [
    bar shift=-4pt,
    draw=blue,
    fill=blue!60,
] table[
    x=index,
    y=Positive,
    col sep=comma,
] {figures/data_other/activation_means.csv};
\addplot [
    draw=red,
    pattern=north east lines,
    bar shift=4pt,
    pattern color=red!60,
]  table[
    x=index,
    y=Negative,
    col sep=comma,
] {figures/data_other/activation_total_means.csv};
\addplot [
    bar shift=4pt,
    draw=red,
    fill=red!30
] table[
    x=index,
    y=Negative,
    col sep=comma,
] {figures/data_other/activation_means.csv};
\legend{Total (positive), Activated (positive), Total (negative), Activated (negative)}
\end{axis}
\end{tikzpicture}
\caption{Mean numbers of activated nodes after passing through the last MPNN layer for each pattern in negative and positive samples compared to mean numbers of \textit{all} nodes in negative and positive samples.}
\label{fig:activated_nodes_end:A}
\end{subfigure}\hfill %
\begin{subfigure}[t]{0.45\linewidth}
        \begin{tikzpicture}[scale=0.6]
\begin{axis}[
    ybar,
    bar width=5pt,
    ymajorgrids=true,
    xmajorgrids=false,
    grid style={dashed},
    xtick={0,1,...,13},
    xticklabels={1,2,...,14},
	ylabel={AUROC upper bound},
    xlabel={Pattern Number},
	enlarge x limits=0.05,
    legend pos=north west,
]
\addplot table[
    x=index,
    y=mean_auroc,
    col sep=comma
] {figures/data_other/mean_aurocs.csv};
\legend{Positive,Negative}
\end{axis}
\end{tikzpicture}
    \caption{Loose upper bound on AUROC of any explainer that classifies edges based on last-layer activations of edge terminal nodes (assuming oracle-like performance if any of the edge's terminals is non-zero, and predicting that an edge is unimportant otherwise).}
    \label{fig:activated_nodes_end:B}
    \end{subfigure}

    \caption{PGExplainer's design choice of attributing edges based on their terminal nodes' activations greatly limits its performance in WOMBAT. Figure \ref{fig:activated_nodes_end:A} shows that, on average, total number of nodes with non-zero activations is very limited, limiting PGExplainer's ability to discriminate between important and unimportant edges. To formalise this intuition, we present a loose upper bound to AUROC it could achieve in WOMBAT for every Pattern (Figure \ref{fig:activated_nodes_end:B}).}
    \label{fig:activated_nodes_end}
\end{figure}

%% file: appendix_architectures.tex
\section{Impact of Whitebox Architecture}
\label{sec:appendix_arch_impact}

\subsection{Decision boundary sharpness}
\label{sec:appendix_arch_sharpness}

We further investigated whether the ``sharp'' decision boundaries of our whitebox architecture impact optimization of PGExplainer and GNNExplainer. To address this, we implemented a modified MPNN formulation that replaces product operators with a minimum function: 

\begin{equation}
m'_{ij} = \min\left(\text{ReLU}(x_i^T W_{\text{from}}), \; \text{ReLU}(x_j^T W_{\text{to}}), \; \text{ReLU}(e_{ij}^T W_{e})\right)
\end{equation}
\begin{equation}
r_2' = \text{ReLU}\left(1 - \prod_{v_i \in V} \text{ReLU}\left(1 - \min_{j} ( \text{ReLU}(x_{i,j}))\right)\right)
\end{equation}

While this architecture maintains perfect accuracy on Patterns 1 and 2, it did not improve explainability. Ten independent PGExplainer runs (varying epoch counts) and a random search of over 11,000 GNNExplainer configurations all failed to explain Pattern 14.

\subsection{GINEConv whitebox experiment}
\label{sec:appendix_gineconv}

WOMBAT's custom architecture facilitates whitebox construction, but may raise concerns about whether our findings transfer to more conventional MPNN architectures. To explore this question, we construct an additional whitebox designed to detect the simultaneous presence of fluorine and a carbonyl group, a task inspired by~\cite{sanchez2020evaluating}. The model uses the GINEConv layer~\cite{hu2019strategies} implemented in PyTorch Geometric~\cite{Fey/etal/2025}, and is released alongside WOMBAT to facilitate further research. Note that we do not claim to have created a whitebox network for~\cite{sanchez2020evaluating}, as we have not performed a validation on the dataset described there and our task definitions and valence assumptions are likely to differ.

\paragraph{Task definition.} We check for simultaneous existence of two SMARTS patterns in a structure:

\begin{enumerate}
    \item \texttt{[!\#1]F} -- Fluorine atom connected to some other nonhydrogen atom.
    \item \texttt{[\#6]=O} -- Oxygen atom connected via a double bond to a carbon atom. 
\end{enumerate}

\paragraph{Input encoding.} Input molecules are encoded as all other molecules in WOMBAT, albeit our network discards the hydrogen OHE as it is not needed to make a prediction. Furthermore, we assume that all molecules satisfy WOMBAT's constraints on valence.

\paragraph{Architecture overview.} Our network consists of the following:

\begin{enumerate}
    \item An embedding layer;
    \item a GINEConv layer;
    \item an MLP-based readout layer.
\end{enumerate}

where a GINEConv's layer is defined as follows:

\begin{equation}
    \label{eq:gineconv}
    x_i' = \mathrm{MLP}\left( (1+\varepsilon) x_i + \sum_{j\in\mathcal{N}(i)} \mathrm{ReLU}(x_j + e_{ji}) \right)
\end{equation}

where \(x_i\) means a representation of the \(i\)-th node before the layer and \(x_i'\) representation of the \(i\)-th node after the layer. 

\paragraph{Hyperparameters and network calibration.} In the Equation \ref{eq:gineconv}, we set \(\varepsilon = 0\). Furthermore, we treat prediction above the threshold of \(0.5\) as positive, and below that threshold as negative.

\paragraph{Notation.} For brevity, we define a vector \(a_i = (1+\varepsilon) x_i + \sum_{j\in\mathcal{N}(i)} \mathrm{ReLU}(x_j + e_{ji})\), meaning aggregated messages received by node \(v_i\) before MLP.

\paragraph{Embeddings.} We define the embeddings as follows:

\begin{itemize}
    \item Carbon atoms get embedded as \(\mathrm{embedding}_\mathrm{C} = \begin{bmatrix} 0 & 0 & -10 & 1 & 0 \end{bmatrix}\)
    \item Oxygen atoms get embedded as \(\mathrm{embedding}_\mathrm{O} = \begin{bmatrix} 0 & 0 & 0 & 0 & 0 \end{bmatrix}\)
    \item Fluorine atoms get embedded as \(\mathrm{embedding}_\mathrm{F} = \begin{bmatrix} 1 & -1 & -10 & 0 & 0 \end{bmatrix}\)
    \item Double bonds get embedded as \(\mathrm{embedding}_\mathrm{double} = \begin{bmatrix} 0 & 0 & 0 & 0 & 1 \end{bmatrix}\)
    \item All other atoms and bonds get embedded as \(\mathrm{embedding}_\mathrm{other} = \begin{bmatrix} 0 & 0 & 0 & -10 & 0 \end{bmatrix}\)
\end{itemize}

\paragraph{MLP layer in GINEConv.} The MLP performs a linear projection from \(\mathbb{R}^5\) to \(\mathbb{R}^2\) and then applies \(\mathrm{tanh}\) activation function. 

The weights for the perceptrons are as follows:

\begin{enumerate}
    \item \(\begin{bmatrix}
        0 & 0 & 1 & 1 & 1
    \end{bmatrix}\) and bias of \(-1\).
    \item \(\begin{bmatrix}
        1 & 1 & 0 & 0 & 0
    \end{bmatrix}\) and bias of \(0\).
\end{enumerate}

The first perceptron detects \(\mathrm{C=O}\) pattern, and the second one detects \(\mathrm{F-*}\) pattern. We will now provide the rationale for weights of the first perceptron.

Note that in the case of the first perceptron, it won't activate in carbon atoms as \(\mathrm{embedding}_{\mathrm{C}}[2] = -10\)\footnote{We assume 0-based indexing.}. It can however be activated in oxygen atoms (\(\mathrm{embedding}_{\mathrm{O}}[2] = 0\)). 

Observe that:
\[\mathrm{ReLU}(\mathrm{embedding}_{\mathrm{C}} + \mathrm{embedding}_{\mathrm{double}}) = \begin{bmatrix}
        0 & 0 & 0 & 1 & 1
\end{bmatrix}\]

but:

\[\mathrm{ReLU}(\mathrm{embedding}_{\mathrm{C}} + \mathrm{embedding}_{\mathrm{single}}) = \begin{bmatrix}
        0 & 0 & 0 & 0 & 0
\end{bmatrix}\]

as \(\mathrm{embedding}_{\mathrm{single}} = \mathrm{embedding}_{\mathrm{other}}\), which effectively ``vetoes'' the carbon signal. As such, for an oxygen atom to achieve the following activation:

\[ 
a_i = \begin{bmatrix}
    0 & 0 & 0 & 1 & 1
\end{bmatrix}
\]

it needs to be an oxygen atom, and it needs to be connected to a carbon atom with a double bond. Note that, if the oxygen atom is connected with a double bond to the carbon atom, \(\mathrm{MLP}(a_{\mathrm{oxygen}})[0] = \mathrm{tanh}(2 - 1) = \mathrm{tanh}(1)\) and otherwise \(\mathrm{MLP}(a_{\mathrm{oxygen}})[0] \leq \mathrm{tanh}(0)\).

Now, we will discuss the second perceptron. Note that:

\[
\mathrm{ReLU}(\begin{bmatrix}
    1 & -1 & -10 & 0 & 0 
\end{bmatrix}) = \begin{bmatrix}
    1 & 0 & 0 & 0 & 0
\end{bmatrix}
\]

Observe that all atoms neighbouring with a fluorine atom will receive a non-zero message from it and subsequently activate the fluorine perceptron.

Note that technically this whitebox is not perfect: in some cases the veto signal could technically be ``overwhelmed'' by many carbon atoms forming double bonds with non-oxygen atoms. While this could be partially rectified by setting the veto weight to a smaller value (e.g. \(-100\)), such cases are unlikely to happen and -- furthermore -- setting this weight to \(-10\) is enough to pass our simplified validation.

At the same time, note that fluorine atom won't activate this perceptron on its own, as \(1 \cdot 1 + 1 \cdot (-1) + 0 = 0\). This prevents fluoride ions from activating this perceptron when it fails our predefined SMARTS (i.e. the fluorine atom must be connected to something else). 

\paragraph{Readout.} Max readout is performed to obtain the \(x \in \mathbb{R}^2\)
 vector representing the graph. We then make a final prediction by using one linear MLP layer (with one perceptron) with the following weights:

 \[
 \begin{bmatrix}
     1 & 1
 \end{bmatrix}
 \]

 and a bias term of \(-1.51\), meant to calibrate the network. Sigmoid activation is used to obtain the final prediction.

 \paragraph{Validation.} We experimentally validated the network on 153173 positives and 846826 negatives randomly sampled from the Pattern 12 validation set (SMILES-wise; labels were generated to match the pattern from this task) and found no misclassifications. Note that this is a ``simplified'' validation pipeline, as the ``regular'' one would require filtering entire PubChem and finding Tversky negatives. We omitted it due to time constraints.

\paragraph{Quantitative results.} None of the explainers produce a perfect explanation, with the exception of Saliency (3 out of 9,999 cases). Their mean AUROC is also substantially lower than in case of WOMBAT's other whiteboxes. We present the full results in Table \ref{fig:app_gineconv_quantitative}.

\begin{table}
\centering
\caption{Explainer performance on the GINEConv whitebox detecting the simultaneous presence of fluorine and a carbonyl group. Mean AUROC and perfect explanations (AUROC = 1) are evaluated on positive molecules; IQR success rates are evaluated on negative molecules. Only Saliency produces any perfect explanations, in 3 of 9,999 cases. PGExplainer's perfect IQR success rate reflects collapsed, uniformly zero attributions. Conversely, non-uniform attributions on negative molecules may faithfully reflect the contribution of one of the two required substructures.
} \label{fig:app_gineconv_quantitative}
\begin{scriptsize}
    
\begin{tabular}{ccccc}
\toprule
Explainer & Mean AUROC\(\uparrow\) & Perfect explanations & Perfect rate & Mean Success Rate (IQR criterion) \\
\toprule

PGExplainer &	0.5000 &	0 / 9,999	& 0.00\% & 1.0000 \\
Integrated Gradients &	0.7077 &	0 / 9,999 &	0.00\% & 0.1478 \\ 
Saliency &  0.7884 & 3 / 9,999 & 0.03\% & 0.0590 \\ 
SHAP Sampling & 0.6494 & 0 / 100 & 0.00\% & 0.3600 \\ 
Input \(\times\) Gradient &  0.7077 & 0 / 9999 & 0.00\% & 0.2274 \\ 
GNN Explainer & 0.7075 & 0 / 9999 & 0.00\% & 0.1516 \\ 
SubgraphX & 0.6560 &  0 / 100 & 0.00\% & N/A \\

 \bottomrule
\end{tabular}
\end{scriptsize}

\end{table}

\paragraph{Qualitative results.} For the evaluation, we’ve used the same pipeline as for WOMBAT’s benchmarks. For qualitative analysis, we randomly sampled 20 SMILES from positive and 20 SMILES from negative set and observed the following:

\begin{itemize}
    \item IG, SVS, Saliency, Input x Gradient consistently fail to properly attribute the neighbouring atom in F-* substructure. This atom is being relied on topologically (as its existence is crucial for activation), but other than participating in message passing it does not ``add'' anything to the final score itself. This may be partially attributed to the fact that we explain only node/edge features. Formulation based on masking graph connectivity itself may help with this issue, as masking the edge would lower the network's prediction. This was not a problem in other whiteboxes, as masking an edge feature vector automatically scaled down all messages sent through that edge, due to our MPNN formulation in Equation \ref{eq:message}.
    \item Similarly, IG, SVS and Input \(\times\) Gradient fail to attribute the oxygen atom in C=O substructure. Again, this atom is being relied on topologically (as its existence is crucial for activation), but other than participating in message passing it does not “add” anything to the final score itself (note that embedding of oxygen atom is a vector of zeroes). In case of IG and SVS, this may also be partially caused due to the baseline being a zero-vector, as in case of this network there is no difference between oxygen’s embedding when it’s masked and when it’s not. This could potentially be addressed by performing IG directly on an atom's embedding space, not on the space of inputs to the GNN, and then using a different baseline for the oxygen atom. This interpolation would likely not fix everything, as failure to attribute neighbouring atom of fluorine does not share this failure mode (carbon atoms, which often neighbour fluorine atoms, do not have zero-embeddings, yet are not properly credited).
    \item Saliency is the only explainer to attribute the oxygen atom in the C=O substructure, as other entries in one hot encoding (that are zeros) accumulate gradients. The oxygen entry itself accumulates zero gradients, as oxygen embedding is a vector of zeroes (this also shows why Input \(\times\) Gradient does not attribute that oxygen). As such, this attribution seems to be accidental, especially given the fact that saliency fails to attribute neighbours of fluorine atoms.
    \item PGExplainer collapsed during training and produces null attributions. However, during our test run on a much smaller number of molecules to explain (and in its training dataset) it managed to converge. Qualitative analysis of that checkpoint shows that it properly attributes C=O motif, while not attributing F-* motif. That said, it sometimes also attributes motifs that are not C=O, like S=O or CON.
    \item GNNExplainer, similarly to Input x Gradient and other explainers, only attributes fluorine atoms belonging to F-* substructures and carbon atoms belonging to the C=O substructures.
    \item SubgraphX is severely disadvantaged, as the explanation is not a connected subgraph (C=O and F-* motif can be on different sides of the molecule). As such, it sometimes produces giant attributions that span through a majority of the molecule to connect the two motifs, while sometimes attributing only one of them. It also fails to attribute oxygen atoms in C=O substructures.
    \item On negative samples, substantially more explainers fail the IQR criterion. PGExplainer is the exception because its null attributions are necessarily uniform and therefore always satisfy the criterion. We attribute the lower success rates to negative molecules containing one of the two required substructures. Although such molecules remain below the classification threshold, the present substructure genuinely increases the network’s output, and faithful explainers may assign it non-uniform attribution. Consequently, failure of the IQR criterion does not necessarily indicate either shortcut learning or explainer failure.
\end{itemize}

\subsection{Differing implementations}
\label{sec:appendix_impls}

WOMBAT whiteboxes for given patterns are \textit{not unique}. In other words, there may exist multiple (or -- in fact -- infinitely many) whiteboxes that are able to detect a given pattern. Appending layers that propagate signal without affecting the readout leaves the \(r_2\) decision function identical (within space of one-hot-encoded inputs), and there are many non-trivial reformulations besides. Note that implementation changes can also change a model's logits (if using the \(r_1\) readout, which is not bounded) without the model losing its correctness (i.e. making correct predictions for a predetermined threshold of 0.5).

To illustrate this point, we will present 4 different whiteboxes for Pattern 12 detection. The first of them is the original whitebox for Pattern 12, while the remaining three are new. For the weight notation, please refer to Section \ref{sec:notation}. Unless stated otherwise, all results will be stated for the readout function \(r_2\). All new whiteboxes were validated on a subset of randomly sampled 1 million molecules from Pattern 12’s validation set.

\subsubsection{The original implementation}

\paragraph{Layer 1.} The first layer aggregates basic chemical features. 

\begin{minipage}{0.55\linewidth}
\begin{enumerate}
    \item \textbf{A1:} Carbon with exactly 1 hydrogen, that is connected via double bond to at least one oxygen atom.
    \item \textbf{B1:} Carbon with exactly 1 hydrogen, that is connected via a single bond to at least one other atom.
\end{enumerate}
\end{minipage}%
\begin{minipage}{0.1\linewidth}
    \begin{center}
        \rule{1pt}{100pt}
    \end{center}
\end{minipage}%
\begin{minipage}{0.35\linewidth}
    \begin{center}
    \includegraphics[width=\linewidth]{files/activation_visualisation/pattern_12/layer_1.png}
    \end{center}
\end{minipage}

\paragraph{Readout.} The final graph-level readout function simply checks for the existence of at least one node that is A1-active and B1-active at the same time.

\paragraph{Observations.} This whitebox exposes PGExplainer’s weakness, i.e. being capable of performing shortcut learning to make explanations that are not always faithful. This is discussed in more detail in the Section \ref{sec:appendix_pge_fail_p12}.

\subsubsection{The second whitebox (pro-PGExplainer)}

This whitebox consists of 2 layers.

\paragraph{Layer 1.} The first layer aggregates basic chemical features. 

\begin{minipage}{0.55\linewidth}
\begin{enumerate}
    \item \textbf{A1:} Oxygen that is connected via a double bond to at least one carbon.
    \item \textbf{B1:} Carbon with exactly 1 hydrogen, that is connected via a single bond to at least one other atom.
\end{enumerate}
\end{minipage}%
\begin{minipage}{0.1\linewidth}
    \begin{center}
        \rule{1pt}{100pt}
    \end{center}
\end{minipage}%
\begin{minipage}{0.35\linewidth}
    \begin{center}
    \includegraphics[width=\linewidth]{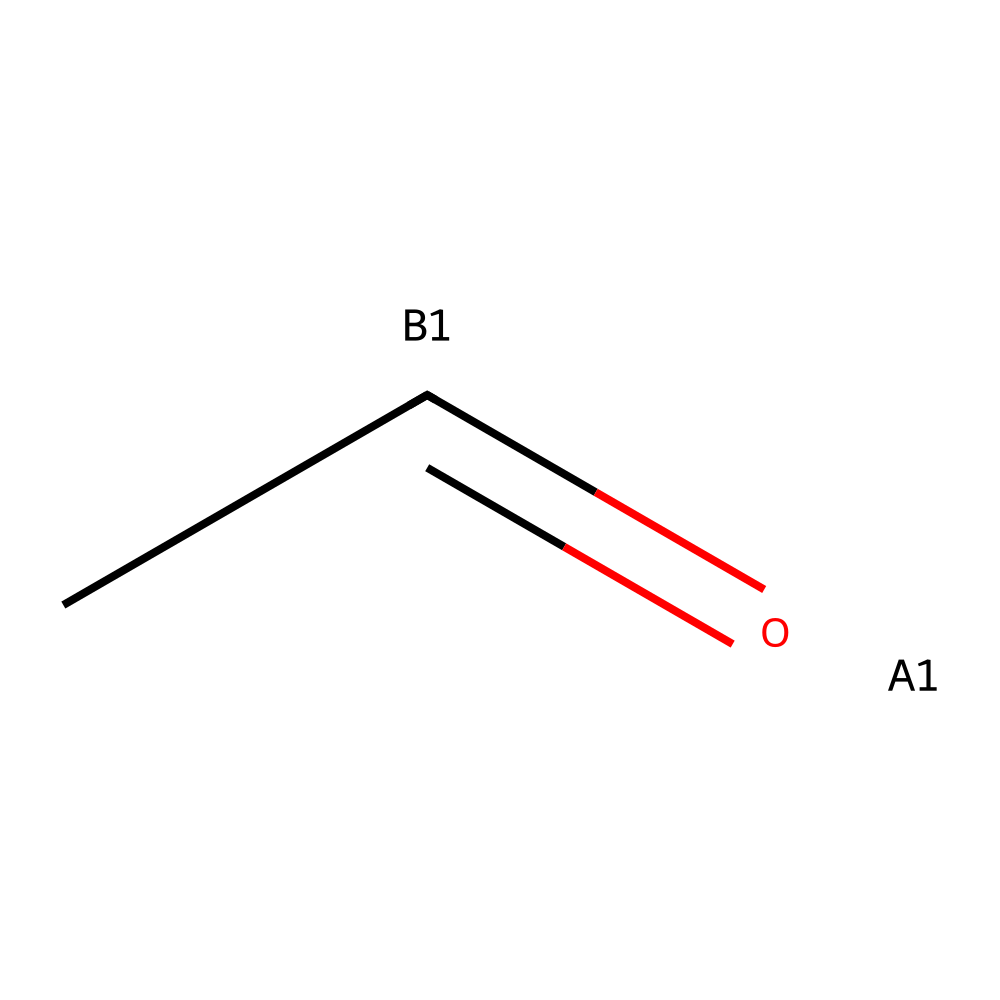}
    \end{center}
\end{minipage}

\paragraph{Layer 2.} The second layer aggregates low-level chemical features into higher-level ones. 

\begin{minipage}{0.55\linewidth}
\begin{enumerate}
    \item \textbf{A2:} B1 with a double bond to A1.
\end{enumerate}
\end{minipage}%
\begin{minipage}{0.1\linewidth}
    \begin{center}
        \rule{1pt}{100pt}
    \end{center}
\end{minipage}%
\begin{minipage}{0.35\linewidth}
    \begin{center}
    \includegraphics[width=\linewidth]{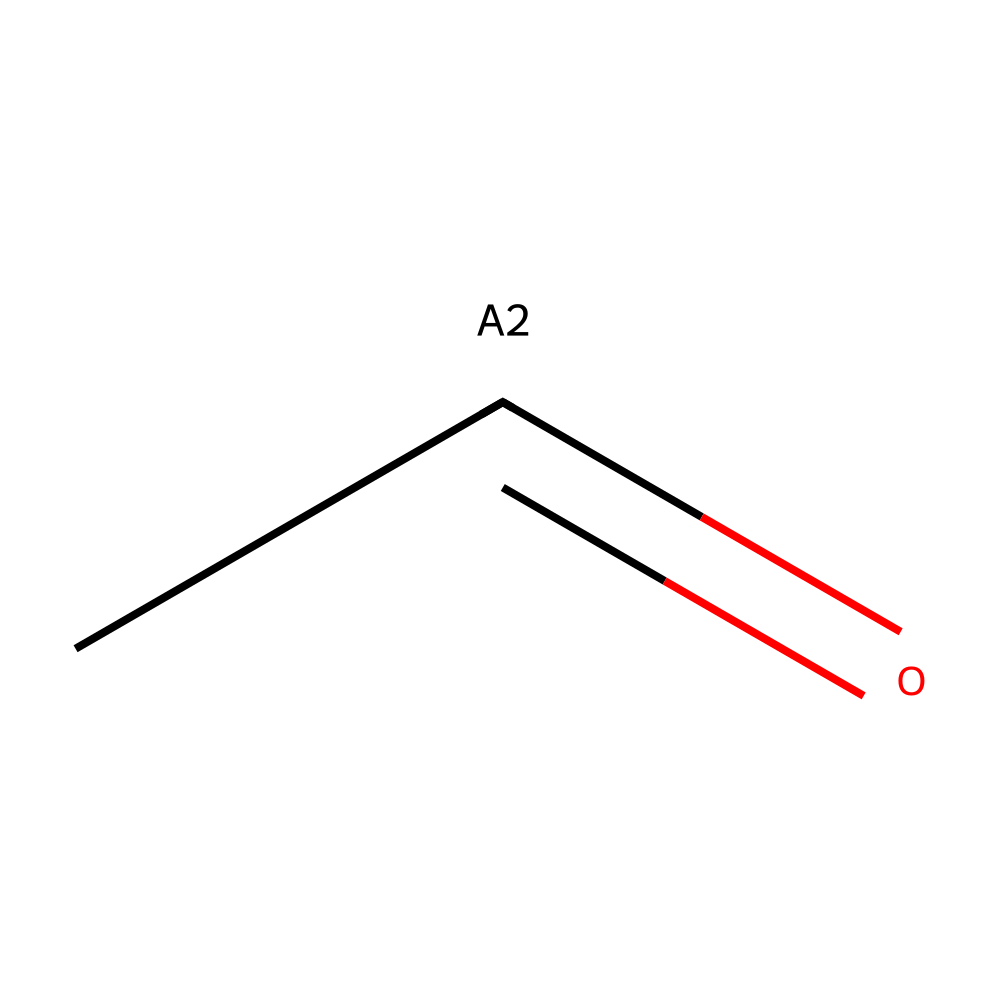}
    \end{center}
\end{minipage}

\paragraph{Readout.} The final graph-level readout function simply checks for the existence of at least one node that is A2-active.

\paragraph{Observations.} This whitebox changes the implementation so that PGExplainer cannot overfit the way it did for the first implementation. Here, the only rule for PGExplainer to ``learn'' is that if an edge has a terminal in A2-active node, it needs to be attributed, which is generally trivial for an MLP to learn\footnote{This is in opposition to the \(\mathrm{A1} \land \mathrm{B1}\) rule PGExplainer should have learned for the original implementation, but failed because of the heuristic \(\mathrm{A1} \implies \mathrm{B1}\) which was true for \textit{most} but not all molecules}. Indeed, PGExplainer always scored AUROC=1.0 for that whitebox. 

\subsubsection{The third whitebox (anti-PGExplainer)}

This whitebox consists of 3 layers.

\paragraph{Layer 1.} The first layer aggregates basic chemical features and is identical to the first layer from the pro-PGExplainer whitebox.

\begin{minipage}{0.55\linewidth}
\begin{enumerate}
    \item \textbf{A1:} Oxygen that is connected via a double bond to at least one carbon.
    \item \textbf{B1:} Carbon with exactly 1 hydrogen, that is connected via a single bond to at least one other atom.
\end{enumerate}
\end{minipage}%
\begin{minipage}{0.1\linewidth}
    \begin{center}
        \rule{1pt}{100pt}
    \end{center}
\end{minipage}%
\begin{minipage}{0.35\linewidth}
    \begin{center}
    \includegraphics[width=\linewidth]{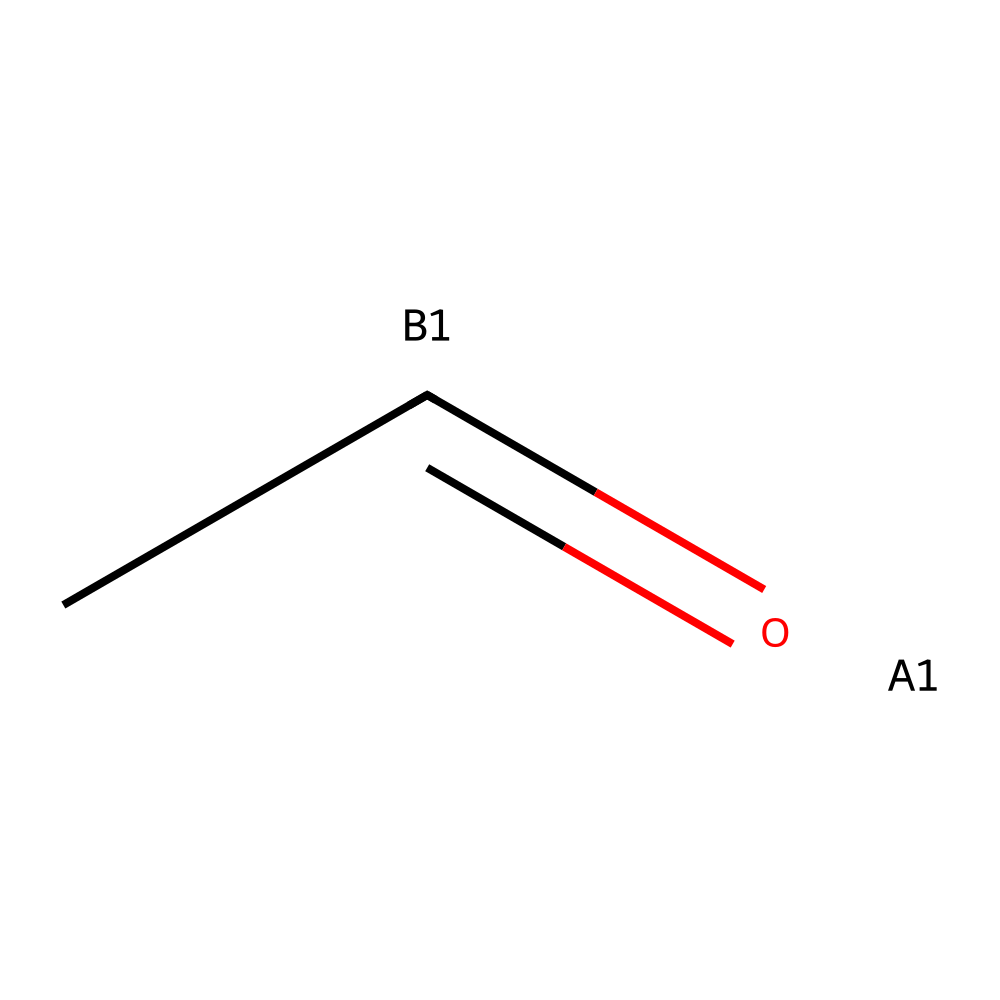}
    \end{center}
\end{minipage}

\paragraph{Layer 2.} The second layer aggregates low-level chemical features into higher-level ones. 

\begin{minipage}{0.55\linewidth}
\begin{enumerate}
    \item \textbf{A2:} B1 with a double bond to A1.
    \item \textbf{B2:} Activation in \textit{any} node (this is done by setting the bias term to 1 and was not done in any other whitebox).
\end{enumerate}
\end{minipage}%
\begin{minipage}{0.1\linewidth}
    \begin{center}
        \rule{1pt}{100pt}
    \end{center}
\end{minipage}%
\begin{minipage}{0.35\linewidth}
    \begin{center}
    \includegraphics[width=\linewidth]{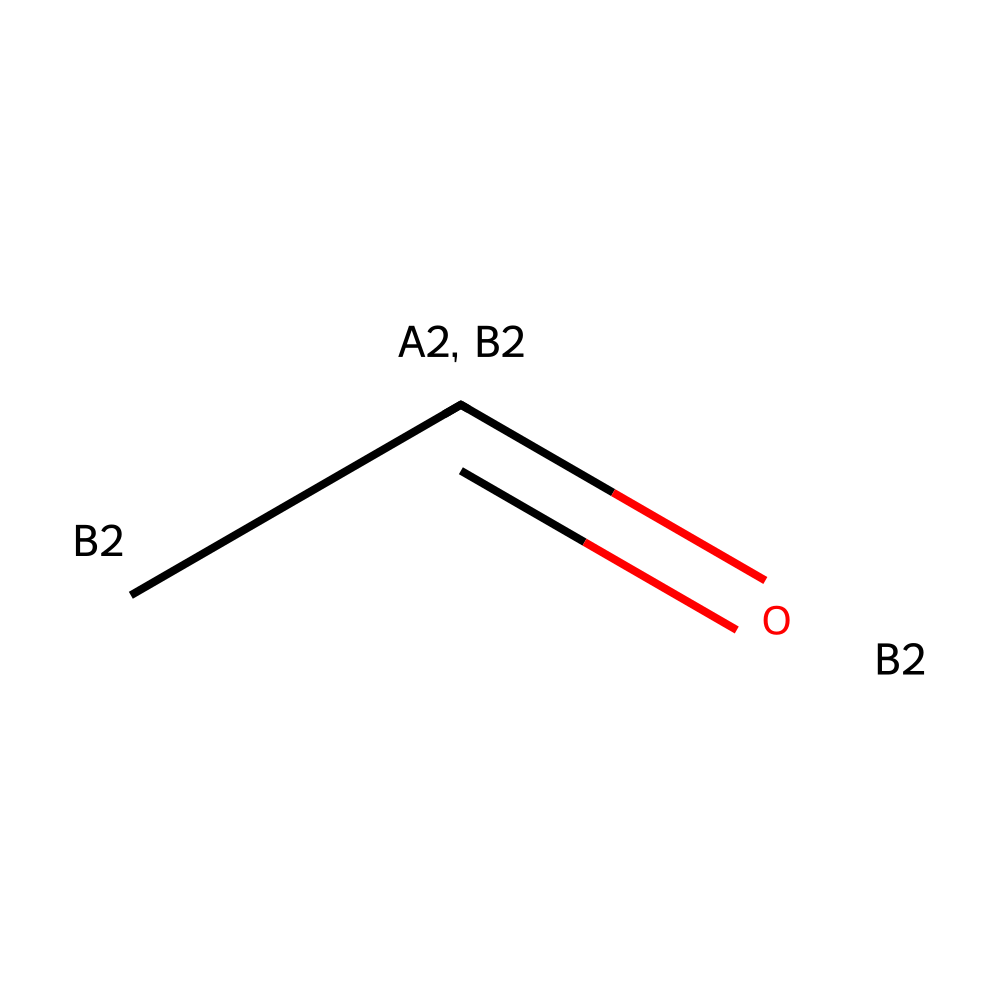}
    \end{center}
\end{minipage}

\paragraph{Layer 3.} The third layer ``moves`` the A2 activation to the nodes bordering the carbon atom.

\begin{minipage}{0.55\linewidth}
\begin{enumerate}
    \item \textbf{A3:} B2 connected with A2 via any bond.
\end{enumerate}
\end{minipage}%
\begin{minipage}{0.1\linewidth}
    \begin{center}
        \rule{1pt}{100pt}
    \end{center}
\end{minipage}%
\begin{minipage}{0.35\linewidth}
    \begin{center}
    \includegraphics[width=\linewidth]{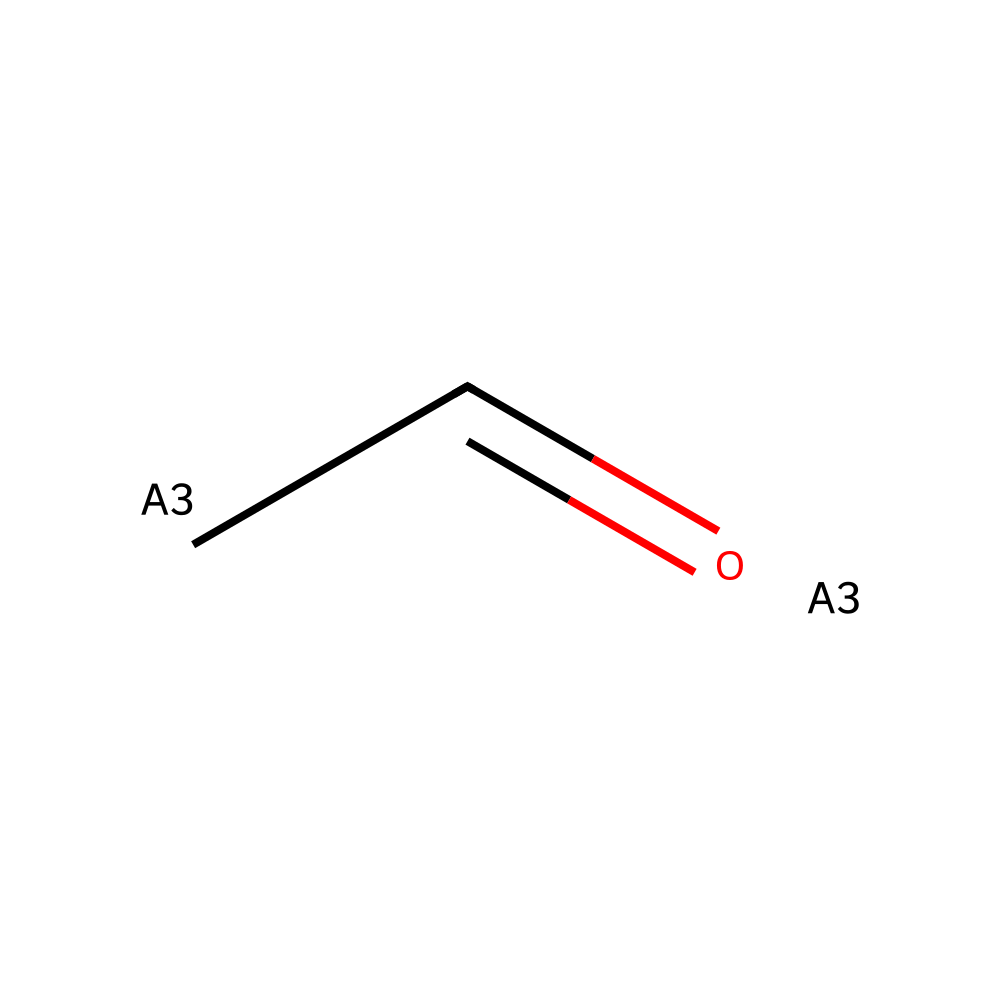}
    \end{center}
\end{minipage}

\paragraph{Readout.} The final graph-level readout function simply checks for the existence of at least one node that is A3-active.

\paragraph{Observations.} This whitebox is meant to further show PGExplainer’s implementation reliance, as it effectively ``spreads around'' the activation that essentially tells that the motif is present to adjacent nodes, making PGE’s reliance on last-layer activations troublesome. That is because -- this time -- merely attributing all edges which have at least one A3-activated terminal can yield false positives (as the carbon's neighbour will be A3-activated, meaning edges connecting it to its irrelevant neighbours would get attributions). Indeed, under that implementation scheme, PGE’s failures are not isolated cases (although it does not return null attributions; it achieves a mean AUROC of 0.983, but ideal attributions only for 5\% of molecules). Note that -- in this case -- the final representations are too ambiguous from the perspective of PGExplainer (i.e. its imperfect performance is \textit{not} an artifact of optimisation process causing it to learn some heuristic). 

\subsubsection{The fourth whitebox (anti-IG)}

This whitebox consists of 6 layers.

\paragraph{Layer 1.} The first layer aggregates basic chemical features. Note the presence of C1 activation, which carries almost no information but will be computationally used to ``smear'' attributions. 

\begin{minipage}{0.55\linewidth}
\begin{enumerate}
    \item \textbf{A1:} Carbon  with exactly 1 hydrogen, that is connected via a single bond to at least one other atom.
    \item \textbf{B1:} Oxygen, that is connected via any bond to at least one other atom.
    \item \textbf{C1:} Any atom connected to any atom via any bond (note that this is possible without setting positive bias)
\end{enumerate}
\end{minipage}%
\begin{minipage}{0.1\linewidth}
    \begin{center}
        \rule{1pt}{100pt}
    \end{center}
\end{minipage}%
\begin{minipage}{0.35\linewidth}
    \begin{center}
    \includegraphics[width=\linewidth]{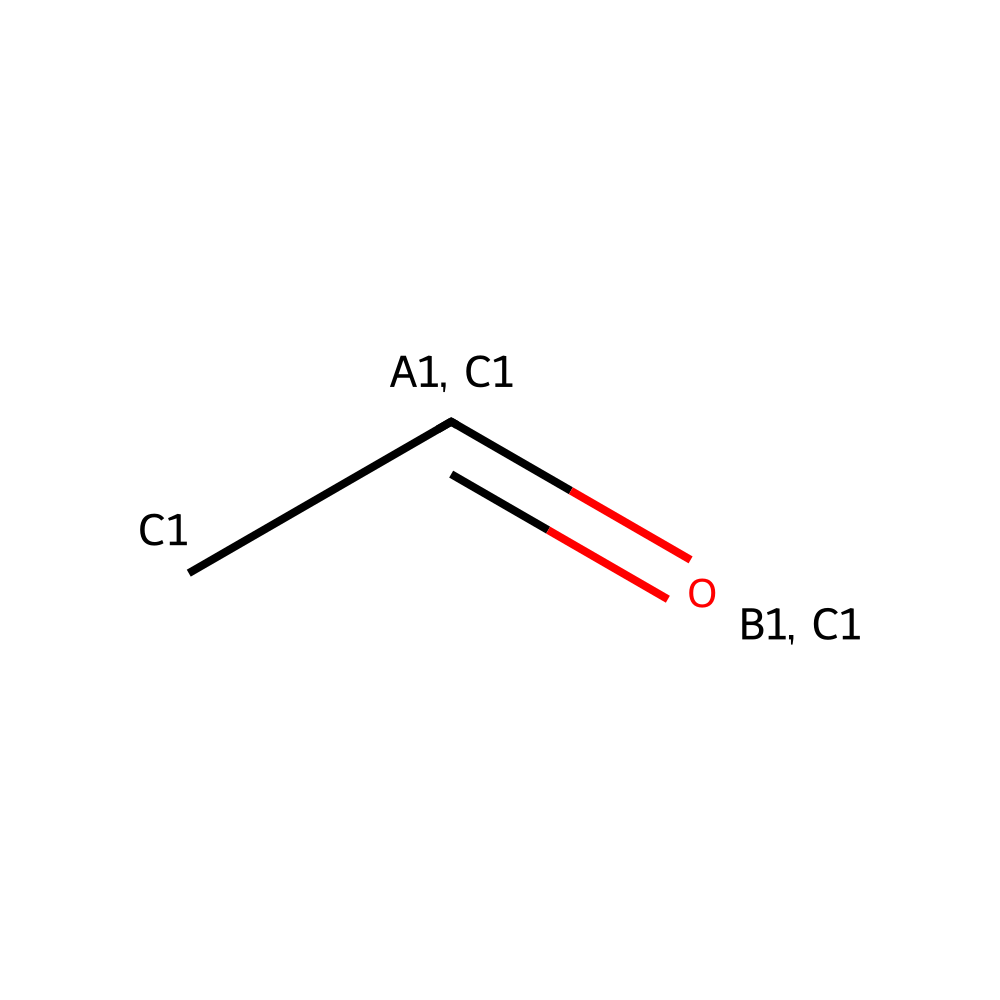}
    \end{center}
\end{minipage}

\paragraph{Layer 2.} The second layer does nothing, except preserving the A1, B1 and C1 activations further. To preserve A1 and B1 activations, we use predicate checking for connection to C1 (every non-isolated atom is C1-activated).

\begin{minipage}{0.55\linewidth}
\begin{enumerate}
    \item \textbf{A2:} A1-activated node connected via any bond to at least one C1-activated atom.
    \item \textbf{B2:} B1-activated node connected via any bond to at least one C1-activated atom.
    \item \textbf{C2:} C1-activated node connected via any bond to at least one C1-activated atom.
\end{enumerate}
\end{minipage}%
\begin{minipage}{0.1\linewidth}
    \begin{center}
        \rule{1pt}{100pt}
    \end{center}
\end{minipage}%
\begin{minipage}{0.35\linewidth}
    \begin{center}
    \includegraphics[width=\linewidth]{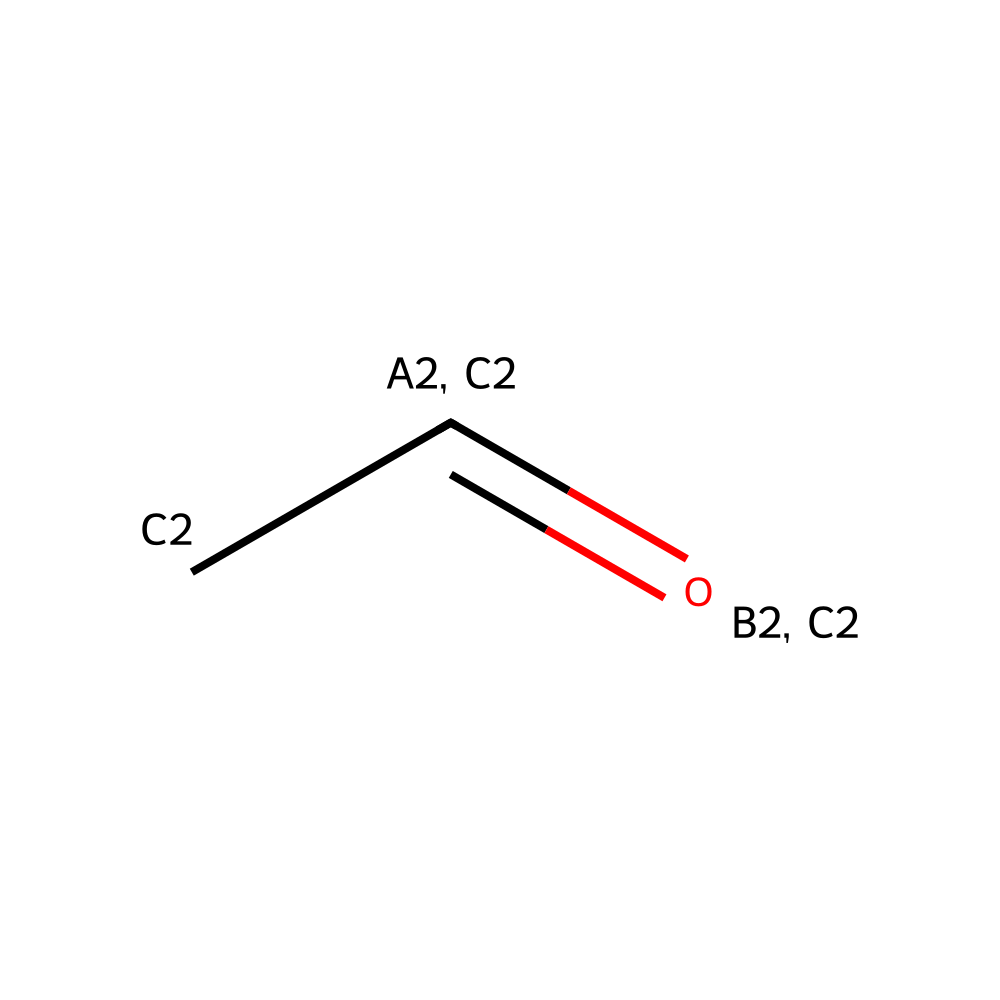}
    \end{center}
\end{minipage}

\paragraph{Layer 3.} The third layer also does nothing, except further activation preservation.

\begin{minipage}{0.55\linewidth}
\begin{enumerate}
    \item \textbf{A3:} A2-activated node connected via any bond to at least one C2-activated atom.
    \item \textbf{B3:} B2-activated node connected via any bond to at least one C2-activated atom.
    \item \textbf{C3:} C2-activated node connected via any bond to at least one C2-activated atom.
\end{enumerate}
\end{minipage}%
\begin{minipage}{0.1\linewidth}
    \begin{center}
        \rule{1pt}{100pt}
    \end{center}
\end{minipage}%
\begin{minipage}{0.35\linewidth}
    \begin{center}
    \includegraphics[width=\linewidth]{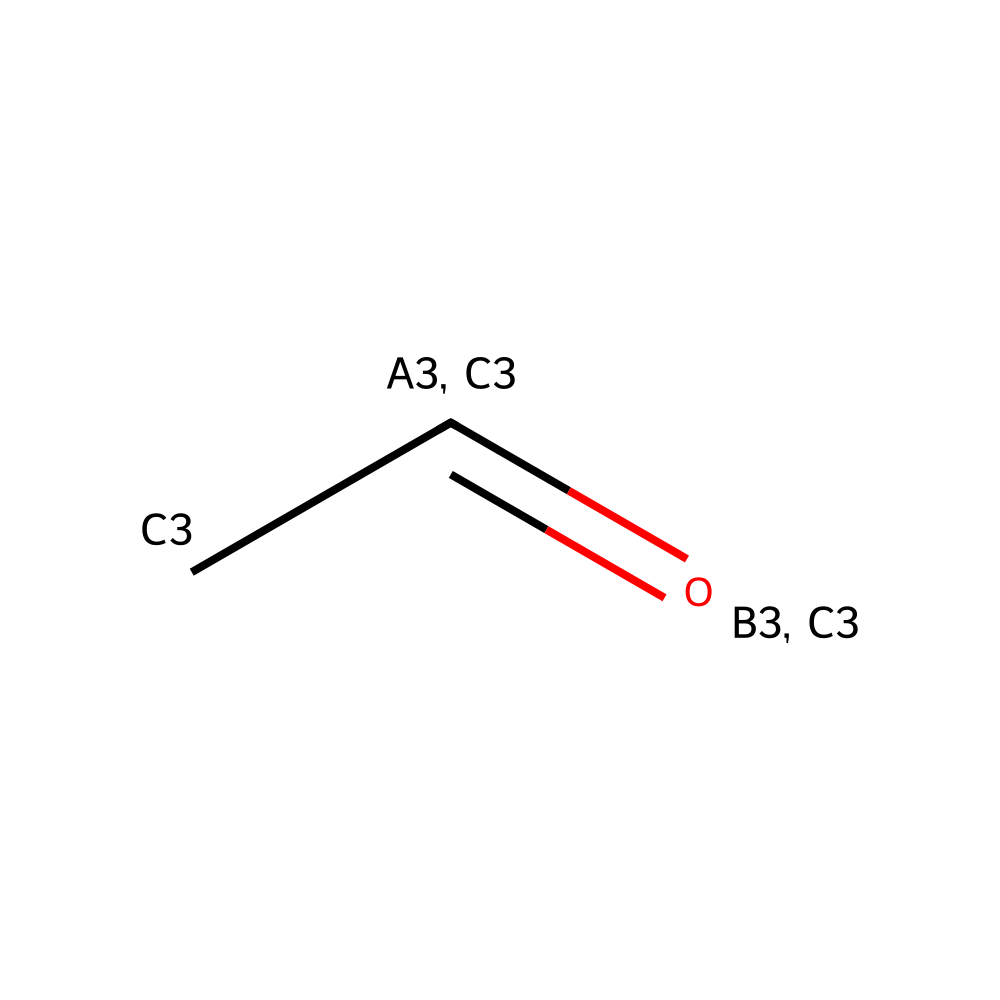}
    \end{center}
\end{minipage}

\paragraph{Layer 4.} The fourth layer likewise only preserves activations.

\begin{minipage}{0.55\linewidth}
\begin{enumerate}
    \item \textbf{A4:} A3-activated node connected via any bond to at least one C3-activated atom.
    \item \textbf{B4:} B3-activated node connected via any bond to at least one C3-activated atom.
    \item \textbf{C4:} C3-activated node connected via any bond to at least one C3-activated atom.
\end{enumerate}
\end{minipage}%
\begin{minipage}{0.1\linewidth}
    \begin{center}
        \rule{1pt}{100pt}
    \end{center}
\end{minipage}%
\begin{minipage}{0.35\linewidth}
    \begin{center}
    \includegraphics[width=\linewidth]{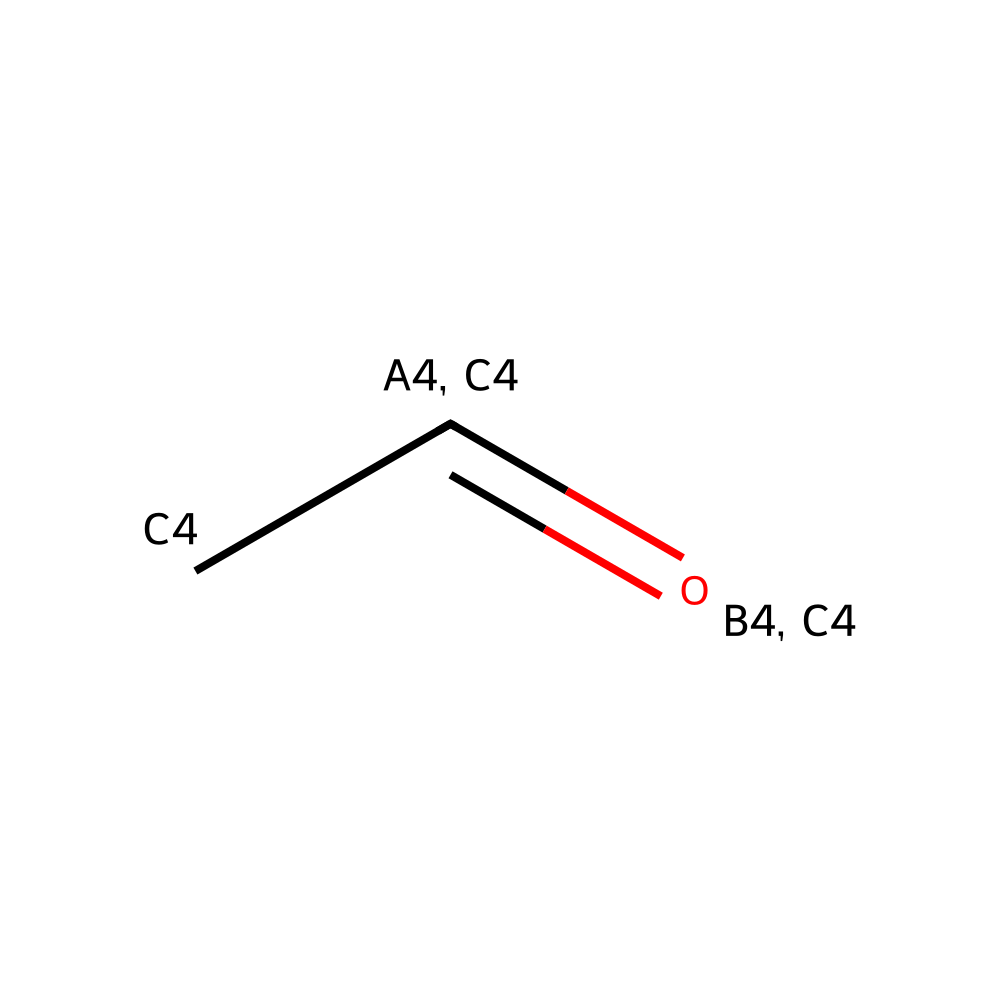}
    \end{center}
\end{minipage}

\paragraph{Layer 5.} The fifth layer likewise only preserves activations.

\begin{minipage}{0.55\linewidth}
\begin{enumerate}
    \item \textbf{A5:} A4-activated node connected via any bond to at least one C4-activated atom.
    \item \textbf{B5:} B4-activated node connected via any bond to at least one C4-activated atom.
    \item \textbf{C5:} C4-activated node connected via any bond to at least one C4-activated atom.
\end{enumerate}
\end{minipage}%
\begin{minipage}{0.1\linewidth}
    \begin{center}
        \rule{1pt}{100pt}
    \end{center}
\end{minipage}%
\begin{minipage}{0.35\linewidth}
    \begin{center}
    \includegraphics[width=\linewidth]{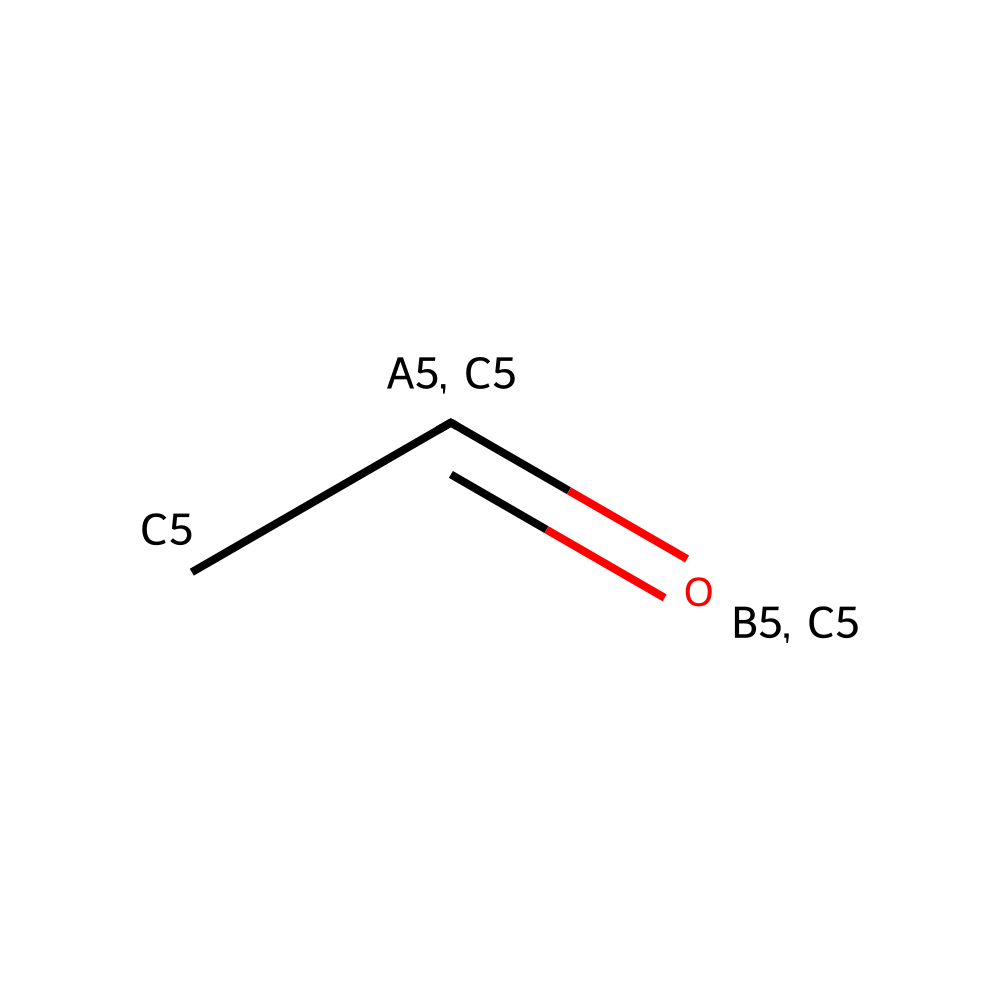}
    \end{center}
\end{minipage}

\paragraph{Layer 6.} The sixth layer combines chemical features into a higher-level one.

\begin{minipage}{0.55\linewidth}
\begin{enumerate}
    \item \textbf{A6:} A5-activated node connected via a double bond to at least one B5-activated atom.
\end{enumerate}
\end{minipage}%
\begin{minipage}{0.1\linewidth}
    \begin{center}
        \rule{1pt}{100pt}
    \end{center}
\end{minipage}%
\begin{minipage}{0.35\linewidth}
    \begin{center}
    \includegraphics[width=\linewidth]{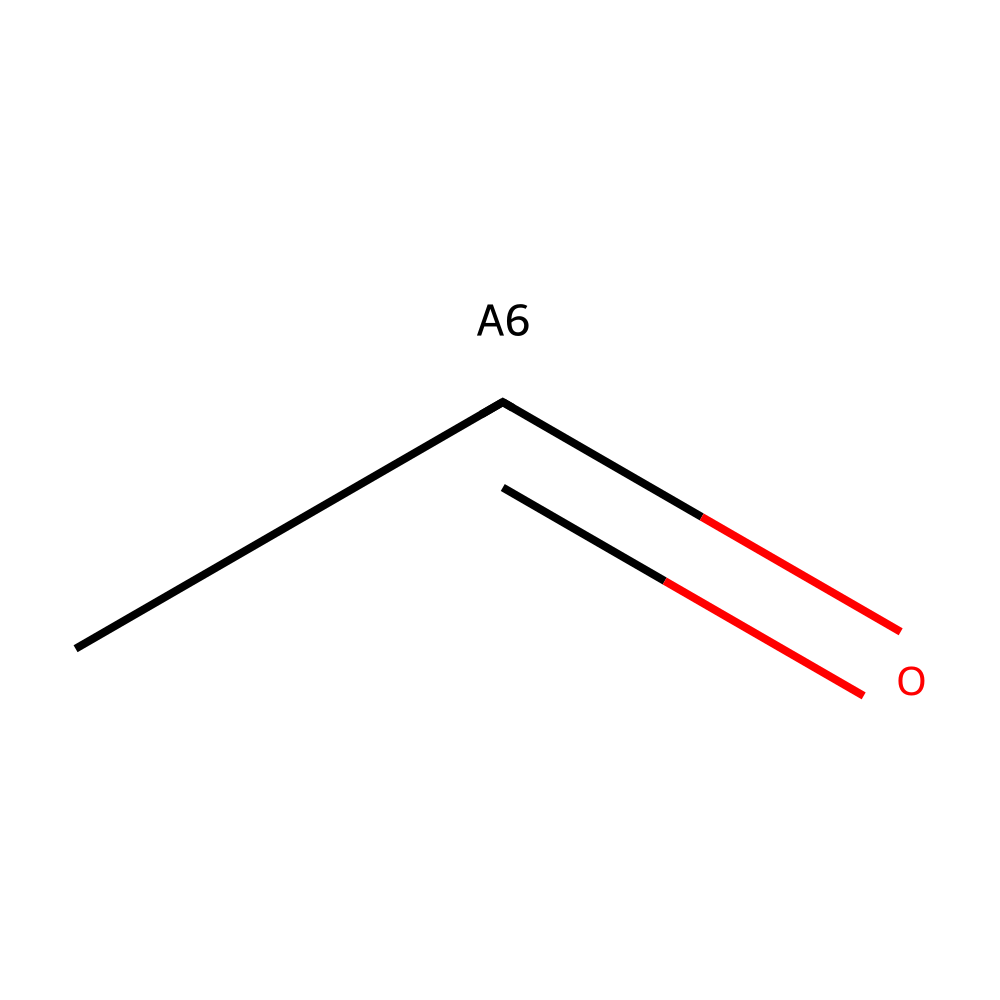}
    \end{center}
\end{minipage}

\paragraph{Readout.} The final graph-level readout function simply checks for the existence of at least one node that is A6-active.

\paragraph{Observations.} The additional layers preserve the motif recognition rule on one-hot-encoded molecular inputs: every non-isolated node is C1-active, so the propagation steps preserve which nodes satisfy the A1 and B1 conditions. However, they introduce additional computational dependencies on neighbouring features through the \(\mathrm{C}_i\) activations, allowing non-motif atoms to contribute gradients along IG's interpolation path. This effectively ``smears'' the resulting attributions: despite a mean AUROC of \(0.97\), IG achieves perfect attributions on only \(7.60\%\) of the evaluated molecules. We illustrate the network's layer-by-layer activations on a larger molecule in Figure~\ref{fig:appendix_smoothing_pattern_12_activations_big} and the corresponding IG attribution in Figure~\ref{fig:leaky_ig}. Note that attributed non-motif nodes are \textit{not} logically needed, as the substructure alone is both necessary and sufficient for the network to predict its existence.

\begin{figure}
    \begin{subfigure}[t]{0.5\linewidth}
        \includegraphics[width=\linewidth]{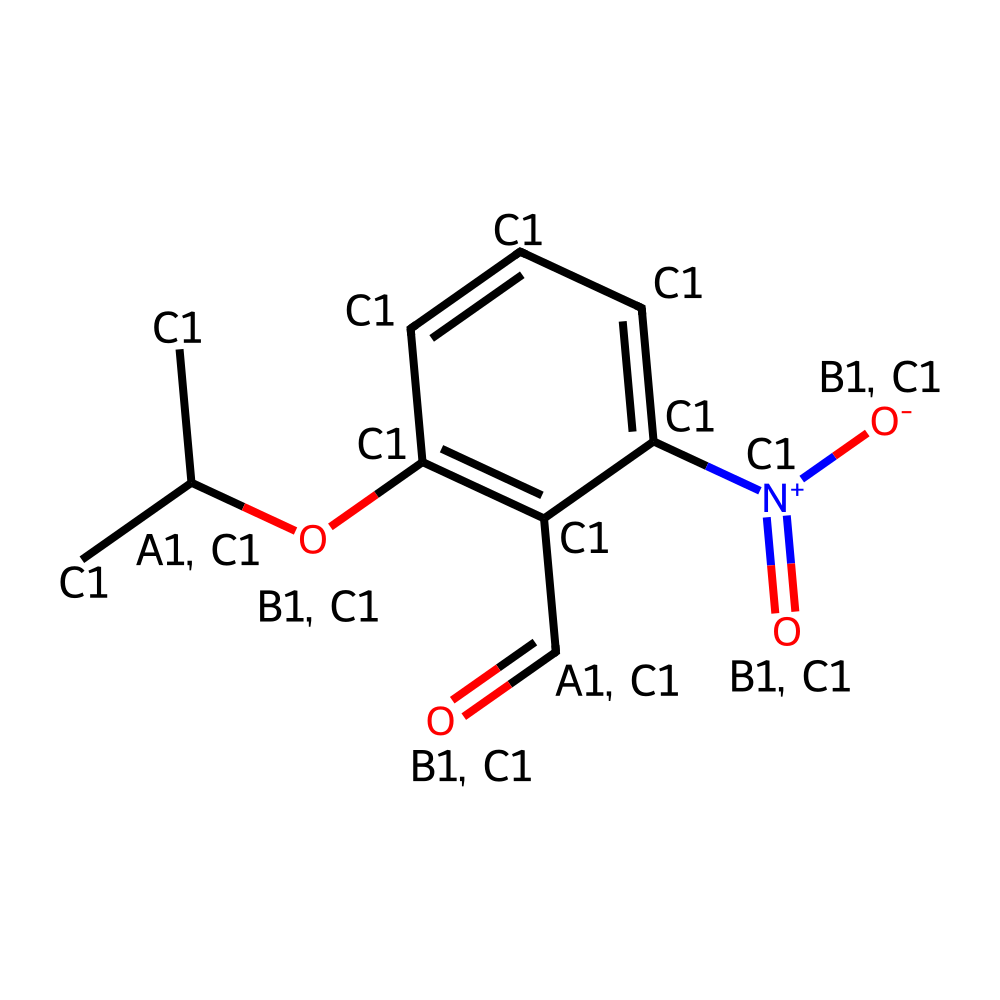}
    \end{subfigure}
    \hfill 
    \begin{subfigure}[t]{0.5\linewidth}
        \includegraphics[width=\linewidth]{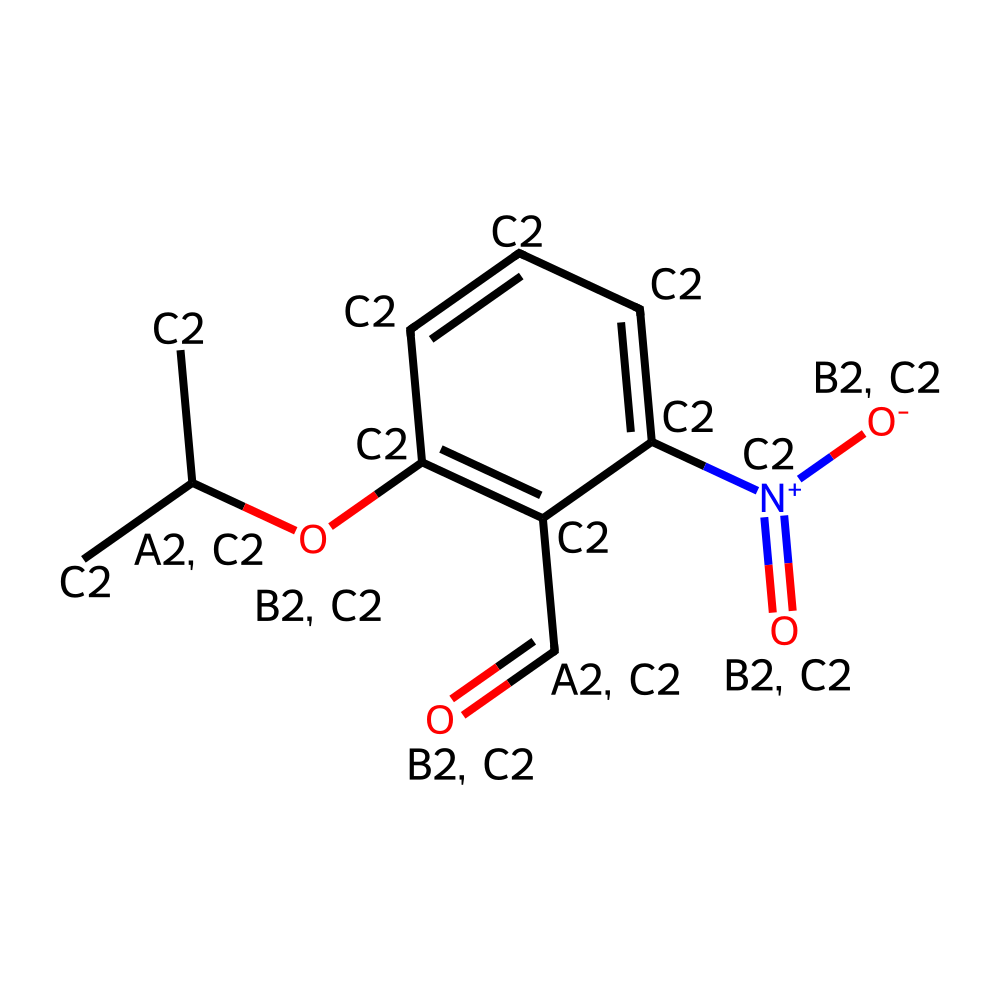}
    \end{subfigure} \\ 
    \begin{subfigure}[t]{0.5\linewidth}
        \includegraphics[width=\linewidth]{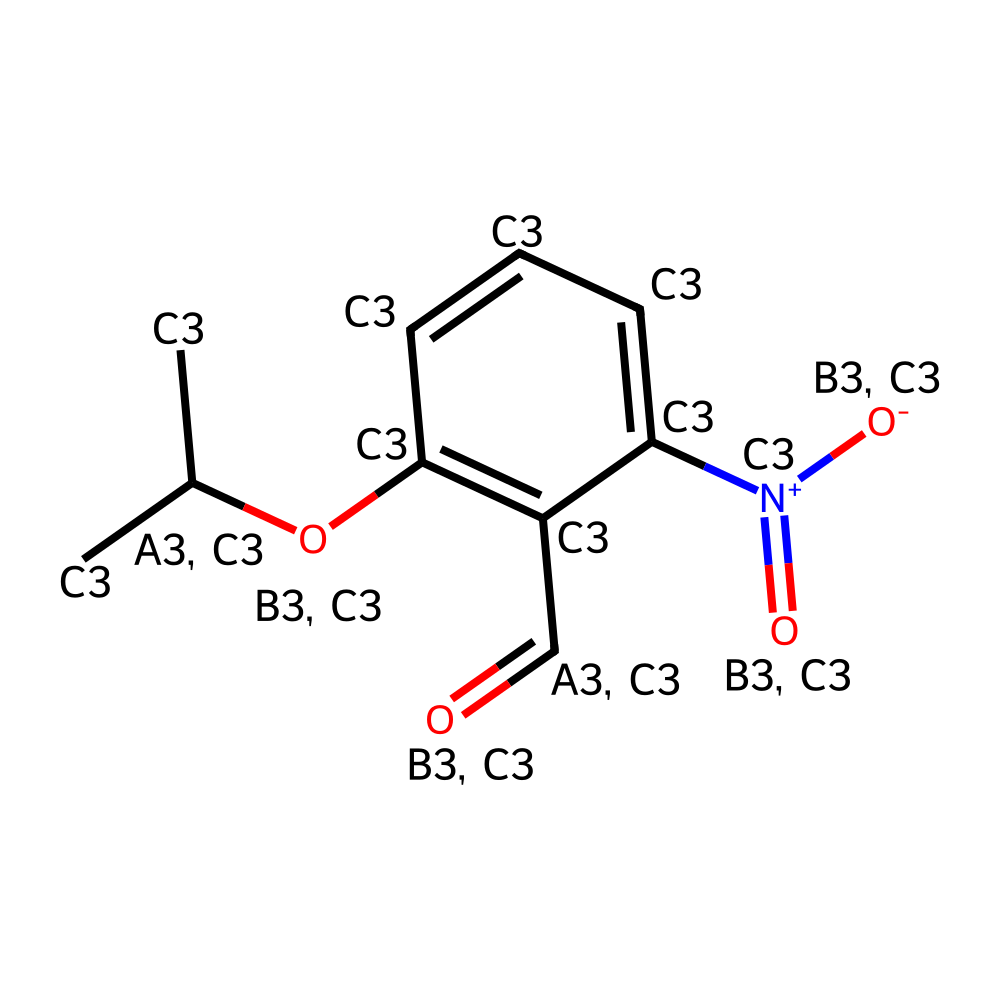}
    \end{subfigure}
    \hfill 
    \begin{subfigure}[t]{0.5\linewidth}
        \includegraphics[width=\linewidth]{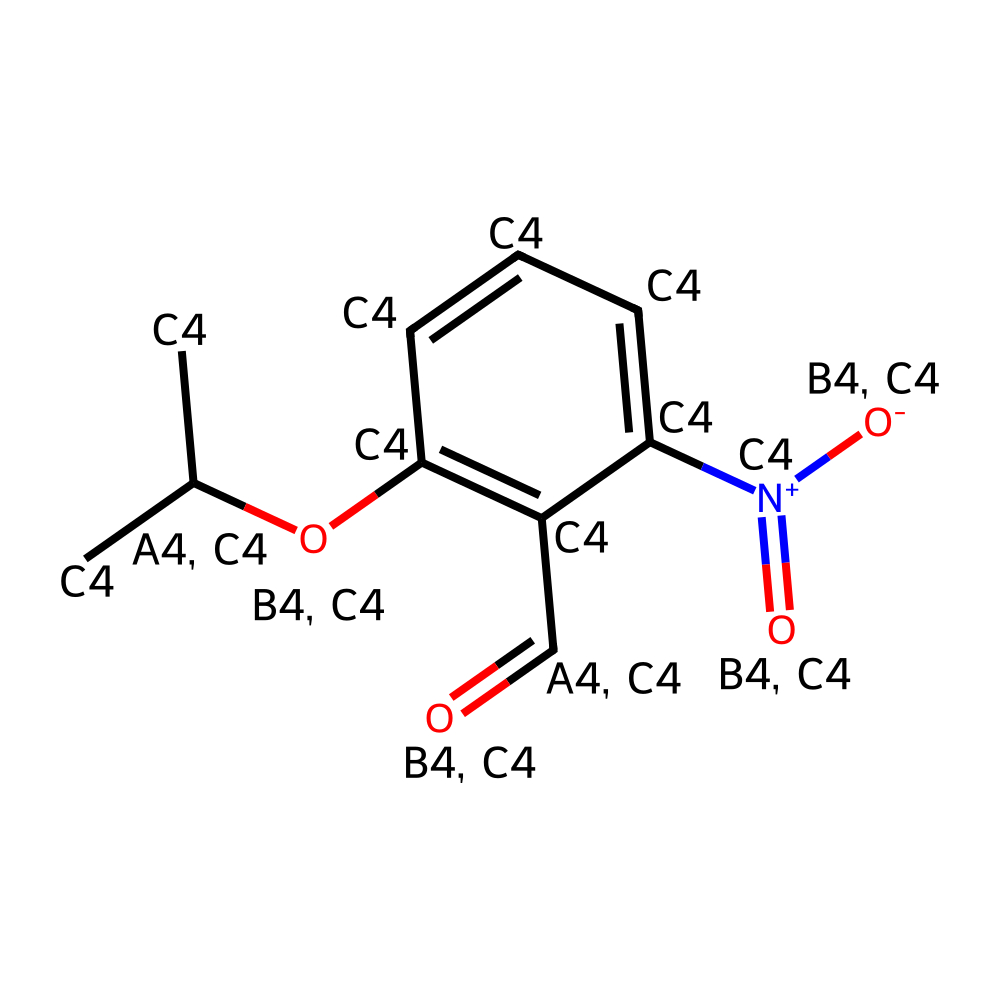}
    \end{subfigure} \\ 
        \begin{subfigure}[t]{0.5\linewidth}
        \includegraphics[width=\linewidth]{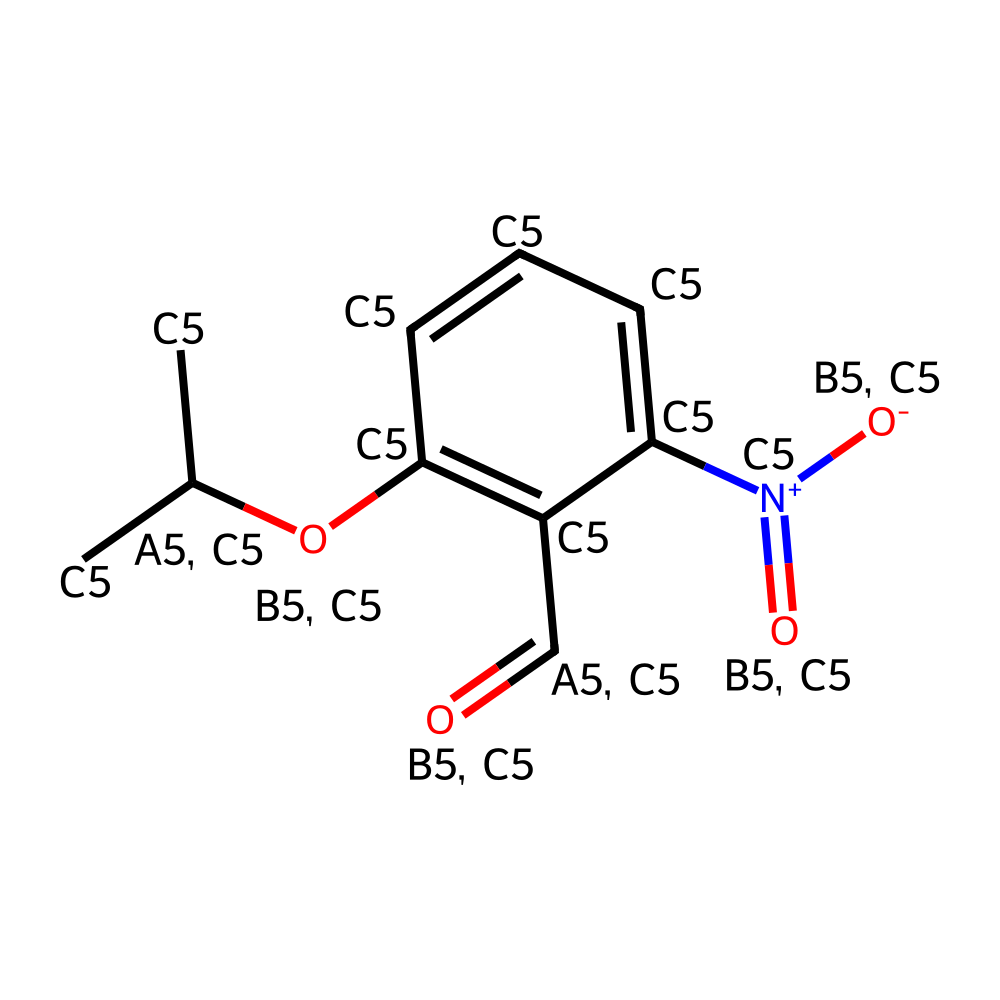}
    \end{subfigure}
    \hfill 
    \begin{subfigure}[t]{0.5\linewidth}
        \includegraphics[width=\linewidth]{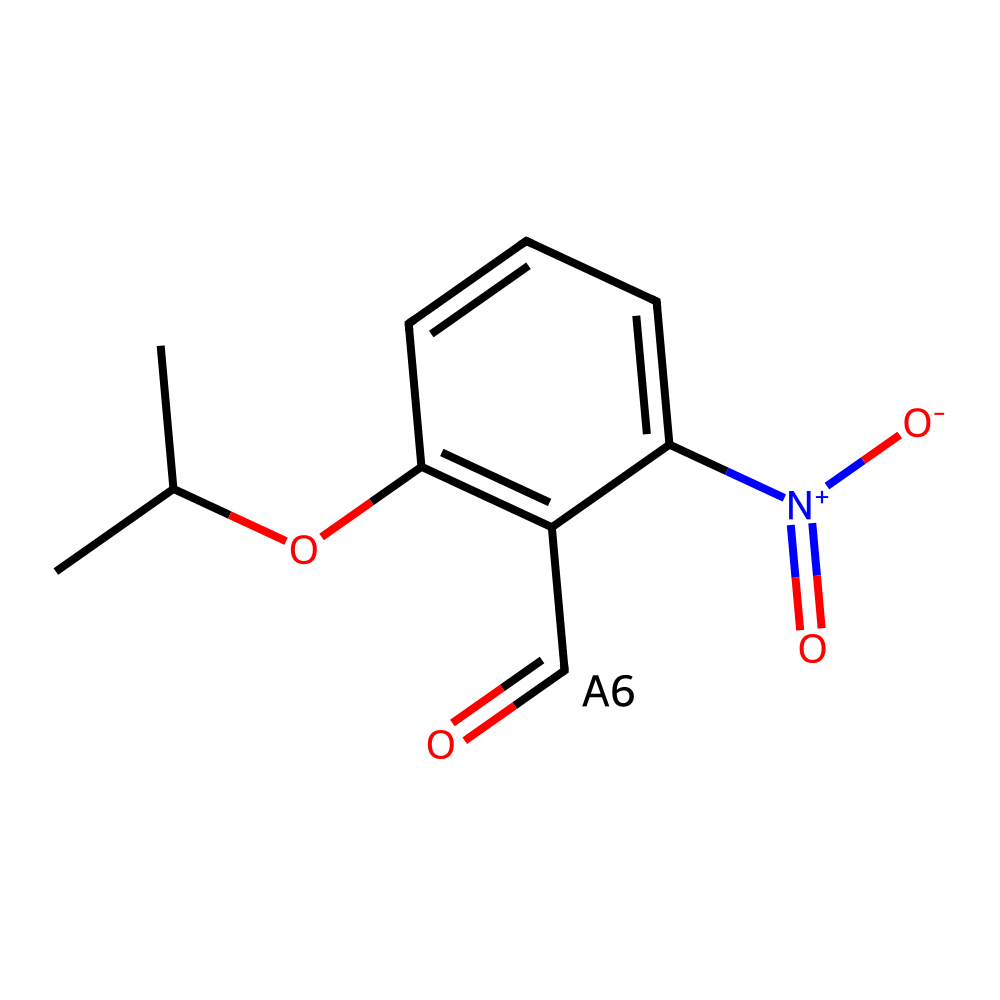}
    \end{subfigure} \\ 
    \caption{Layer-by-layer activations of the anti-IG Pattern 12 on an example molecule.}
    \label{fig:appendix_smoothing_pattern_12_activations_big}
\end{figure}

%% file: appendix_additional_results.tex
\section{Additional experimental results}

In this section we present results presenting AUROC for readout $r_1$ in Table~\ref{tab:model_eval_auroc_r1} and for readout $r_2$ in Table~\ref{tab:model_eval_auroc_r2}, and combined results colorcoded in Table~\ref{tab:model_eval_auroc_continuous_custom_palette}. Then we present results for mean Average Precision for readout $r_1$ in Table~\ref{tab:model_eval_ap_r1}, for readout $r_2$ in Table~\ref{tab:model_eval_ap_r2}, and combined in Table~\ref{tab:model_eval_ap_continuous_custom_palette}.

\input{tables/big_table_aurocs_readout_1}
\input{tables/big_table_aurocs_readout_2}
\input{tables/big_table_aps_readout_1}
\input{tables/big_table_aps_readout_2}
\input{tables/big_table_aurocs_pos_combined}
\input{tables/big_table_asp_pos_combined}

%% file: tables/big_table_aurocs_readout_1.tex
\begin{table}
\setlength{\tabcolsep}{3pt}

\caption{Mean AUROCs of all explainers across different whiteboxes (using readout function from Equation \ref{eqn:readout_1}) when evaluated on positive examples.}
\label{tab:model_eval_auroc_r1}
\begin{tiny}
\centering
\begin{tabular}{cccccccc}
 & \multicolumn{7}{c}{\textbf{mean AUROC \(\uparrow\)}} \\
\toprule
\textbf{No.}& \textbf{GNN Expl}. & \textbf{Input x Grad.} & \textbf{IG} & \textbf{PG Expl.} & \textbf{SHAP Sampl.} & \textbf{Saliency }& \textbf{SubgraphX} \\
\toprule
Pattern 1 & $0.500 \pm 0.000$ & $0.994 \pm 0.000$ & $0.994 \pm 0.000$ & $0.500 \pm 0.000$ & $0.989 \pm 0.002$ & $0.994 \pm 0.000$ & $0.935 \pm 0.011$ \\
Pattern 2 & $0.500 \pm 0.000$ & $0.997 \pm 0.002$ & $0.997 \pm 0.002$ & $0.500 \pm 0.000$ & $0.975 \pm 0.003$ & $0.997 \pm 0.002$ & $0.960 \pm 0.012$ \\
Pattern 3 & $0.500 \pm 0.000$ & $1.000 \pm 0.000$ & $1.000 \pm 0.000$ & $0.500 \pm 0.000$ & $0.966 \pm 0.006$ & $1.000 \pm 0.000$ & $0.916 \pm 0.026$ \\
Pattern 4 & $0.500 \pm 0.000$ & $0.982 \pm 0.002$ & $0.982 \pm 0.002$ & $0.500 \pm 0.000$ & $0.980 \pm 0.003$ & $0.986 \pm 0.002$ & $0.798 \pm 0.019$ \\
Pattern 5 & $0.500 \pm 0.000$ & $0.987 \pm 0.006$ & $0.989 \pm 0.005$ & $0.500 \pm 0.000$ & $0.939 \pm 0.030$ & $0.991 \pm 0.004$ & $0.857 \pm 0.032$ \\
Pattern 6 & $0.500 \pm 0.000$ & $0.942 \pm 0.001$ & $0.942 \pm 0.001$ & $0.500 \pm 0.000$ & $0.993 \pm 0.003$ & $0.942 \pm 0.001$ & $0.914 \pm 0.016$ \\
Pattern 7 & $0.500 \pm 0.000$ & $0.982 \pm 0.001$ & $0.982 \pm 0.001$ & $0.500 \pm 0.000$ & $0.995 \pm 0.003$ & $0.982 \pm 0.001$ & $0.939 \pm 0.014$ \\
Pattern 8 & $0.500 \pm 0.000$ & $0.973 \pm 0.001$ & $0.973 \pm 0.001$ & $0.500 \pm 0.000$ & $0.990 \pm 0.004$ & $0.973 \pm 0.001$ & $0.888 \pm 0.019$ \\
Pattern 9 & $0.500 \pm 0.000$ & $0.985 \pm 0.001$ & $0.985 \pm 0.001$ & $0.500 \pm 0.000$ & $0.997 \pm 0.001$ & $0.985 \pm 0.001$ & $0.941 \pm 0.013$ \\
Pattern 10 & $0.500 \pm 0.000$ & $0.870 \pm 0.001$ & $0.870 \pm 0.001$ & $0.500 \pm 0.000$ & $0.996 \pm 0.002$ & $0.870 \pm 0.001$ & $0.874 \pm 0.016$ \\
Pattern 11 & $0.500 \pm 0.000$ & $0.939 \pm 0.001$ & $0.939 \pm 0.001$ & $0.500 \pm 0.000$ & $0.997 \pm 0.003$ & $0.939 \pm 0.001$ & $0.888 \pm 0.018$ \\
Pattern 12 & $0.980 \pm 0.001$ & $0.980 \pm 0.001$ & $0.980 \pm 0.001$ & $1.000 \pm 0.000$ & $1.000 \pm 0.000$ & $0.980 \pm 0.001$ & $0.824 \pm 0.024$ \\
Pattern 13 & $0.500 \pm 0.000$ & $0.985 \pm 0.001$ & $0.985 \pm 0.001$ & $0.500 \pm 0.000$ & $1.000 \pm 0.001$ & $0.985 \pm 0.001$ & $0.905 \pm 0.017$ \\
Pattern 14 & $0.500 \pm 0.000$ & $0.980 \pm 0.001$ & $0.980 \pm 0.001$ & $0.500 \pm 0.000$ & $1.000 \pm 0.000$ & $0.980 \pm 0.001$ & $0.943 \pm 0.014$ \\
\midrule
\textbf{Mean} & $0.534$ & $0.971$ & $0.971$ & $0.536$ & $0.987$ & $0.972$ & $0.899$\\
\bottomrule
\end{tabular}

\end{tiny}
\end{table}

%GNN Explainer_auroc: 0.534 ± 0.034
%Input x Gradient_auroc: 0.971 ± 0.009
%Integrated Gradients_auroc: 0.971 ± 0.009
%PG Explainer_auroc: 0.536 ± 0.036
%SHAP Sampling_auroc: 0.987 ± 0.005
%Saliency_auroc: 0.972 ± 0.009
%SubgraphX_auroc: 0.899 ± 0.013

%% file: tables/big_table_aurocs_readout_2.tex
\begin{table}
\setlength{\tabcolsep}{3pt}
\caption{Mean AUROCs of all explainers across different whiteboxes (using readout function from Equation \ref{eqn:readout_2}) when evaluated on positive examples.}
\label{tab:model_eval_auroc_r2}
\begin{tiny}
\centering
\begin{tabular}{cccccccc}
 & \multicolumn{7}{c}{\textbf{mean AUROC \(\uparrow\)}} \\
\toprule
\textbf{No.}& \textbf{GNN Expl}. & \textbf{Input x Grad.} & \textbf{IG} & \textbf{PG Expl.} & \textbf{SHAP Sampl.} & \textbf{Saliency }& \textbf{SubgraphX} \\
\toprule
1 & $0.500 \pm 0.000$ & $0.500 \pm 0.000$ & $1.000 \pm 0.000$ & $0.500 \pm 0.000$ & $0.992 \pm 0.002$ & $0.500 \pm 0.000$ & $0.929 \pm 0.011$ \\
2 & $0.500 \pm 0.000$ & $0.500 \pm 0.000$ & $1.000 \pm 0.000$ & $0.500 \pm 0.000$ & $0.980 \pm 0.003$ & $0.500 \pm 0.000$ & $0.957 \pm 0.012$ \\
3 & $0.500 \pm 0.000$ & $0.500 \pm 0.000$ & $1.000 \pm 0.000$ & $0.500 \pm 0.000$ & $0.964 \pm 0.004$ & $0.500 \pm 0.000$ & $0.895 \pm 0.027$ \\
4 & $0.500 \pm 0.000$ & $0.500 \pm 0.000$ & $0.992 \pm 0.001$ & $0.500 \pm 0.000$ & $0.975 \pm 0.004$ & $0.500 \pm 0.000$ & $0.793 \pm 0.019$ \\
 5 & $0.500 \pm 0.000$ & $0.500 \pm 0.000$ & $0.995 \pm 0.003$ & $0.500 \pm 0.000$ & $0.976 \pm 0.007$ & $0.500 \pm 0.000$ & $0.854 \pm 0.031$ \\
 6 & $0.500 \pm 0.000$ & $0.500 \pm 0.000$ & $1.000 \pm 0.000$ & $0.500 \pm 0.000$ & $0.982 \pm 0.005$ & $0.500 \pm 0.000$ & $0.898 \pm 0.016$ \\
 7 & $0.500 \pm 0.000$ & $0.500 \pm 0.000$ & $1.000 \pm 0.000$ & $0.500 \pm 0.000$ & $0.997 \pm 0.002$ & $0.500 \pm 0.000$ & $0.938 \pm 0.014$ \\
 8 & $0.500 \pm 0.000$ & $0.500 \pm 0.000$ & $1.000 \pm 0.000$ & $0.500 \pm 0.000$ & $0.988 \pm 0.004$ & $0.500 \pm 0.000$ & $0.889 \pm 0.019$ \\
 9 & $0.500 \pm 0.000$ & $0.500 \pm 0.000$ & $1.000 \pm 0.000$ & $0.500 \pm 0.000$ & $0.999 \pm 0.001$ & $0.500 \pm 0.000$ & $0.935 \pm 0.013$ \\
 10 & $0.500 \pm 0.000$ & $0.500 \pm 0.000$ & $1.000 \pm 0.000$ & $0.500 \pm 0.000$ & $0.997 \pm 0.002$ & $0.500 \pm 0.000$ & $0.874 \pm 0.016$ \\
 11 & $0.500 \pm 0.000$ & $0.500 \pm 0.000$ & $1.000 \pm 0.000$ & $0.500 \pm 0.000$ & $0.997 \pm 0.002$ & $0.500 \pm 0.000$ & $0.889 \pm 0.018$ \\
 12 & $1.000 \pm 0.000$ & $0.500 \pm 0.000$ & $1.000 \pm 0.000$ & $1.000 \pm 0.000$ & $1.000 \pm 0.000$ & $0.500 \pm 0.000$ & $0.824 \pm 0.024$ \\
 13 & $0.500 \pm 0.000$ & $0.500 \pm 0.000$ & $1.000 \pm 0.000$ & $0.500 \pm 0.000$ & $1.000 \pm 0.000$ & $0.500 \pm 0.000$ & $0.903 \pm 0.018$ \\
 14 & $0.500 \pm 0.000$ & $0.500 \pm 0.000$ & $1.000 \pm 0.000$ & $0.500 \pm 0.000$ & $1.000 \pm 0.000$ & $0.500 \pm 0.000$ & $0.942 \pm 0.015$ \\
\midrule
\textbf{Mean} & $0.536$ & $0.500$ & $0.999$ & $0.536$ & $0.989$ & $0.500$ & $0.894$\\
\bottomrule
\end{tabular}

\end{tiny}
\end{table}

%PG Explainer_auroc: 0.536 ± 0.036
%Integrated Gradients_auroc: 0.999 ± 0.001
%Saliency_auroc: 0.500 ± 0.000
%SHAP Sampling_auroc: 0.989 ± 0.003
%Input x Gradient_auroc: 0.500 ± 0.000
%GNN Explainer_auroc: 0.536 ± 0.036
%SubgraphX_auroc: 0.894 ± 0.012

%% file: tables/big_table_aps_readout_1.tex
\begin{table}
\setlength{\tabcolsep}{2.1pt}
\caption{Mean APs of all explainers across different whiteboxes (using readout function from Equation \ref{eqn:readout_1}) when evaluated on positive examples.}
\label{tab:model_eval_ap_r1}
\begin{tiny}
\centering
\begin{tabular}{ccccccccc}
 & \multicolumn{8}{c}{\textbf{mean AP \(\uparrow\)}} \\
\toprule
\textbf{No.}& \textbf{Baseline} & \textbf{GNN Expl}. & \textbf{Input x Grad.} & \textbf{IG} & \textbf{PG Expl.} & \textbf{SHAP Sampl.} & \textbf{Saliency }& \textbf{SubgraphX} \\
\toprule
 1 & $0.406 \pm 0.001$ & $0.406 \pm 0.001$ & $0.995 \pm 0.000$ & $0.995 \pm 0.000$ & $0.406 \pm 0.001$ & $0.987 \pm 0.003$ & $0.995 \pm 0.000$ & $0.886 \pm 0.018$ \\
 2 & $0.581 \pm 0.012$ & $0.581 \pm 0.012$ & $0.998 \pm 0.001$ & $0.998 \pm 0.001$ & $0.581 \pm 0.012$ & $0.979 \pm 0.003$ & $0.998 \pm 0.001$ & $0.935 \pm 0.018$ \\
 3 & $0.678 \pm 0.020$ & $0.678 \pm 0.020$ & $1.000 \pm 0.000$ & $1.000 \pm 0.000$ & $0.678 \pm 0.020$ & $0.980 \pm 0.003$ & $1.000 \pm 0.000$ & $0.929 \pm 0.022$ \\
 4 & $0.462 \pm 0.005$ & $0.462 \pm 0.005$ & $0.980 \pm 0.002$ & $0.980 \pm 0.002$ & $0.462 \pm 0.005$ & $0.978 \pm 0.004$ & $0.984 \pm 0.002$ & $0.687 \pm 0.031$ \\
 5 & $0.587 \pm 0.041$ & $0.587 \pm 0.041$ & $0.995 \pm 0.003$ & $0.995 \pm 0.002$ & $0.587 \pm 0.041$ & $0.972 \pm 0.008$ & $0.995 \pm 0.002$ & $0.786 \pm 0.048$ \\
 6 & $0.294 \pm 0.001$ & $0.294 \pm 0.001$ & $0.928 \pm 0.001$ & $0.928 \pm 0.001$ & $0.294 \pm 0.001$ & $0.991 \pm 0.004$ & $0.928 \pm 0.001$ & $0.857 \pm 0.026$ \\
 7 & $0.297 \pm 0.001$ & $0.297 \pm 0.001$ & $0.975 \pm 0.001$ & $0.975 \pm 0.001$ & $0.297 \pm 0.001$ & $0.993 \pm 0.004$ & $0.975 \pm 0.001$ & $0.873 \pm 0.028$ \\
 8 & $0.270 \pm 0.001$ & $0.270 \pm 0.001$ & $0.959 \pm 0.001$ & $0.959 \pm 0.001$ & $0.270 \pm 0.001$ & $0.983 \pm 0.006$ & $0.959 \pm 0.001$ & $0.765 \pm 0.035$ \\
 9 & $0.296 \pm 0.001$ & $0.296 \pm 0.001$ & $0.977 \pm 0.001$ & $0.977 \pm 0.001$ & $0.296 \pm 0.001$ & $0.997 \pm 0.001$ & $0.977 \pm 0.001$ & $0.872 \pm 0.024$ \\
 10 & $0.266 \pm 0.001$ & $0.266 \pm 0.001$ & $0.837 \pm 0.002$ & $0.837 \pm 0.002$ & $0.266 \pm 0.001$ & $0.994 \pm 0.003$ & $0.837 \pm 0.002$ & $0.814 \pm 0.025$ \\
 11 & $0.210 \pm 0.001$ & $0.210 \pm 0.001$ & $0.915 \pm 0.002$ & $0.915 \pm 0.002$ & $0.210 \pm 0.001$ & $0.996 \pm 0.003$ & $0.915 \pm 0.002$ & $0.780 \pm 0.032$ \\
 12 & $0.145 \pm 0.001$ & $0.967 \pm 0.001$ & $0.967 \pm 0.001$ & $0.967 \pm 0.001$ & $1.000 \pm 0.000$ & $1.000 \pm 0.000$ & $0.967 \pm 0.001$ & $0.657 \pm 0.042$ \\
 13 & $0.136 \pm 0.001$ & $0.136 \pm 0.001$ & $0.979 \pm 0.001$ & $0.979 \pm 0.001$ & $0.136 \pm 0.001$ & $0.999 \pm 0.001$ & $0.979 \pm 0.001$ & $0.720 \pm 0.034$ \\
 14 & $0.173 \pm 0.001$ & $0.173 \pm 0.001$ & $0.969 \pm 0.001$ & $0.969 \pm 0.001$ & $0.173 \pm 0.001$ & $1.000 \pm 0.000$ & $0.969 \pm 0.001$ & $0.868 \pm 0.030$ \\
\midrule
\textbf{Mean} & $0.343$ & $0.402$ & $0.962$ & $0.962$ & $0.404$ & $0.989$ & $0.963$ & $0.816$ \\
\bottomrule
\end{tabular}

\end{tiny}
\end{table}

%% file: tables/big_table_aps_readout_2.tex
\begin{table}
\setlength{\tabcolsep}{2.1pt}
\caption{Mean APs of all explainers across different whiteboxes (using readout function from Equation \ref{eqn:readout_2}).}
\label{tab:model_eval_ap_r2}
\begin{tiny}
\centering
\begin{tabular}{ccccccccc}
 & \multicolumn{8}{c}{\textbf{mean AP \(\uparrow\)}} \\
\toprule
\textbf{No.}& \textbf{Baseline} & \textbf{GNN Expl}. & \textbf{Input x Grad.} & \textbf{IG} & \textbf{PG Expl.} & \textbf{SHAP Sampl.} & \textbf{Saliency }& \textbf{SubgraphX} \\
\toprule
 1 & $0.406 \pm 0.001$ & $0.406 \pm 0.001$ & $0.406 \pm 0.001$ & $1.000 \pm 0.000$ & $0.406 \pm 0.001$ & $0.991 \pm 0.002$ & $0.406 \pm 0.001$ & $0.872 \pm 0.018$ \\
 2 & $0.581 \pm 0.012$ & $0.581 \pm 0.012$ & $0.581 \pm 0.012$ & $1.000 \pm 0.000$ & $0.581 \pm 0.012$ & $0.983 \pm 0.003$ & $0.581 \pm 0.012$ & $0.933 \pm 0.018$ \\
 3 & $0.678 \pm 0.020$ & $0.678 \pm 0.020$ & $0.678 \pm 0.020$ & $1.000 \pm 0.000$ & $0.678 \pm 0.020$ & $0.978 \pm 0.003$ & $0.678 \pm 0.020$ & $0.911 \pm 0.022$ \\
 4 & $0.462 \pm 0.005$ & $0.462 \pm 0.005$ & $0.462 \pm 0.005$ & $0.991 \pm 0.001$ & $0.462 \pm 0.005$ & $0.967 \pm 0.006$ & $0.462 \pm 0.005$ & $0.682 \pm 0.031$ \\
 5 & $0.587 \pm 0.041$ & $0.587 \pm 0.041$ & $0.587 \pm 0.041$ & $1.000 \pm 0.000$ & $0.587 \pm 0.041$ & $0.984 \pm 0.005$ & $0.587 \pm 0.041$ & $0.783 \pm 0.047$ \\
 6 & $0.294 \pm 0.001$ & $0.294 \pm 0.001$ & $0.294 \pm 0.001$ & $1.000 \pm 0.000$ & $0.294 \pm 0.001$ & $0.976 \pm 0.006$ & $0.294 \pm 0.001$ & $0.846 \pm 0.025$ \\
 7 & $0.297 \pm 0.001$ & $0.297 \pm 0.001$ & $0.297 \pm 0.001$ & $1.000 \pm 0.000$ & $0.297 \pm 0.001$ & $0.996 \pm 0.003$ & $0.297 \pm 0.001$ & $0.870 \pm 0.028$ \\
 8 & $0.270 \pm 0.001$ & $0.270 \pm 0.001$ & $0.270 \pm 0.001$ & $1.000 \pm 0.000$ & $0.270 \pm 0.001$ & $0.981 \pm 0.006$ & $0.270 \pm 0.001$ & $0.764 \pm 0.035$ \\
 9 & $0.296 \pm 0.001$ & $0.296 \pm 0.001$ & $0.296 \pm 0.001$ & $1.000 \pm 0.000$ & $0.296 \pm 0.001$ & $0.999 \pm 0.001$ & $0.296 \pm 0.001$ & $0.869 \pm 0.025$ \\
 10 & $0.266 \pm 0.001$ & $0.266 \pm 0.001$ & $0.266 \pm 0.001$ & $1.000 \pm 0.000$ & $0.266 \pm 0.001$ & $0.996 \pm 0.002$ & $0.266 \pm 0.001$ & $0.814 \pm 0.025$ \\
 11 & $0.210 \pm 0.001$ & $0.210 \pm 0.001$ & $0.210 \pm 0.001$ & $1.000 \pm 0.000$ & $0.210 \pm 0.001$ & $0.996 \pm 0.003$ & $0.210 \pm 0.001$ & $0.784 \pm 0.032$ \\
 12 & $0.145 \pm 0.001$ & $1.000 \pm 0.000$ & $0.145 \pm 0.001$ & $1.000 \pm 0.000$ & $1.000 \pm 0.000$ & $1.000 \pm 0.000$ & $0.145 \pm 0.001$ & $0.658 \pm 0.042$ \\
 13 & $0.136 \pm 0.001$ & $0.136 \pm 0.001$ & $0.136 \pm 0.001$ & $1.000 \pm 0.000$ & $0.136 \pm 0.001$ & $1.000 \pm 0.000$ & $0.136 \pm 0.001$ & $0.717 \pm 0.035$ \\
 14 & $0.173 \pm 0.001$ & $0.173 \pm 0.001$ & $0.173 \pm 0.001$ & $1.000 \pm 0.000$ & $0.173 \pm 0.001$ & $1.000 \pm 0.000$ & $0.173 \pm 0.001$ & $0.867 \pm 0.030$ \\
\midrule
\textbf{Mean} & $0.343$ & $0.404$ & $0.343$ & $0.999$ & $0.404$ & $0.989$ & $0.343$ & $0.812$ \\
\bottomrule
\end{tabular}

\end{tiny}
\end{table}

%GNN Explainer_ap: 0.404 ± 0.063
%GNN Explainer_ap_baseline: 0.343 ± 0.046

%Input x Gradient_ap: 0.343 ± 0.046
%Input x Gradient_ap_baseline: 0.343 ± 0.046

%Integrated Gradients_ap: 0.999 ± 0.001
%Integrated Gradients_ap_baseline: 0.343 ± 0.046

%PG Explainer_ap: 0.404 ± 0.063
%PG Explainer_ap_baseline: 0.343 ± 0.046

%SHAP Sampling_ap: 0.989 ± 0.003
%SHAP Sampling_ap_baseline: 0.336 ± 0.046

%Saliency_ap: 0.343 ± 0.046
%Saliency_ap_baseline: 0.343 ± 0.046

%SubgraphX_ap: 0.812 ± 0.022
%SubgraphX_ap_baseline: 0.336 ± 0.046

%% file: tables/big_table_aurocs_pos_combined.tex
\begin{table}[htpb]
\setlength{\tabcolsep}{4.5pt}
\caption{Mean AUROCs of explainers across different whiteboxes evaluated on positive examples. Sub-columns \(r_1\) and \(r_2\) denote the readout functions utilized by the whitebox models (Equation \ref{eqn:readout_1} and Equation \ref{eqn:readout_2}, respectively).}
\label{tab:model_eval_auroc_continuous_custom_palette}
\centering
\begin{tiny}
% User-requested landmark colors
\definecolor{user_c1}{HTML}{E57373} % Red (Target: 0.50)
\definecolor{user_c2}{HTML}{FFCDD2} % Pale Red (Target: ~0.65)
\definecolor{user_c3}{HTML}{C5E1A5} % Light Green (Target: ~0.85)
\definecolor{user_c4}{HTML}{81C784} % Green (Target: 1.00)

\begin{tabular}{c cccccccccccccc}
\toprule
\multirow{2}{*}{\textbf{No.}} & \multicolumn{2}{c}{\textbf{GNN Expl.}} & \multicolumn{2}{c}{\textbf{Input $\times$ Grad.}} & \multicolumn{2}{c}{\textbf{IG}} & \multicolumn{2}{c}{\textbf{PG Expl.}} & \multicolumn{2}{c}{\textbf{SHAP Sampl.}} & \multicolumn{2}{c}{\textbf{Saliency}} & \multicolumn{2}{c}{\textbf{SubgraphX}} \\
\cmidrule(lr){2-3} \cmidrule(lr){4-5} \cmidrule(lr){6-7} \cmidrule(lr){8-9} \cmidrule(lr){10-11} \cmidrule(lr){12-13} \cmidrule(lr){14-15}
& \(r_1\) & \(r_2\) & \(r_1\) & \(r_2\) & \(r_1\) & \(r_2\) & \(r_1\) & \(r_2\) & \(r_1\) & \(r_2\) & \(r_1\) & \(r_2\) & \(r_1\) & \(r_2\) \\
\midrule
1  & \cellcolor{user_c1} 0.500 & \cellcolor{user_c1} 0.500 & \cellcolor{user_c4!96!user_c3} 0.994 & \cellcolor{user_c1} 0.500 & \cellcolor{user_c4!96!user_c3} 0.994 & \cellcolor{user_c4} 1.000 & \cellcolor{user_c1} 0.500 & \cellcolor{user_c1} 0.500 & \cellcolor{user_c4!93!user_c3} 0.989 & \cellcolor{user_c4!95!user_c3} 0.992 & \cellcolor{user_c4!96!user_c3} 0.994 & \cellcolor{user_c1} 0.500 & \cellcolor{user_c4!57!user_c3} 0.935 & \cellcolor{user_c4!53!user_c3} 0.929 \\
2  & \cellcolor{user_c1} 0.500 & \cellcolor{user_c1} 0.500 & \cellcolor{user_c4!98!user_c3} 0.997 & \cellcolor{user_c1} 0.500 & \cellcolor{user_c4!98!user_c3} 0.997 & \cellcolor{user_c4} 1.000 & \cellcolor{user_c1} 0.500 & \cellcolor{user_c1} 0.500 & \cellcolor{user_c4!83!user_c3} 0.975 & \cellcolor{user_c4!87!user_c3} 0.980 & \cellcolor{user_c4!98!user_c3} 0.997 & \cellcolor{user_c1} 0.500 & \cellcolor{user_c4!73!user_c3} 0.960 & \cellcolor{user_c4!71!user_c3} 0.957 \\
3  & \cellcolor{user_c1} 0.500 & \cellcolor{user_c1} 0.500 & \cellcolor{user_c4} 1.000 & \cellcolor{user_c1} 0.500 & \cellcolor{user_c4} 1.000 & \cellcolor{user_c4} 1.000 & \cellcolor{user_c1} 0.500 & \cellcolor{user_c1} 0.500 & \cellcolor{user_c4!77!user_c3} 0.966 & \cellcolor{user_c4!76!user_c3} 0.964 & \cellcolor{user_c4} 1.000 & \cellcolor{user_c1} 0.500 & \cellcolor{user_c4!44!user_c3} 0.916 & \cellcolor{user_c4!30!user_c3} 0.895 \\
4  & \cellcolor{user_c1} 0.500 & \cellcolor{user_c1} 0.500 & \cellcolor{user_c4!88!user_c3} 0.982 & \cellcolor{user_c1} 0.500 & \cellcolor{user_c4!88!user_c3} 0.982 & \cellcolor{user_c4!95!user_c3} 0.992 & \cellcolor{user_c1} 0.500 & \cellcolor{user_c1} 0.500 & \cellcolor{user_c4!87!user_c3} 0.980 & \cellcolor{user_c4!83!user_c3} 0.975 & \cellcolor{user_c4!91!user_c3} 0.986 & \cellcolor{user_c1} 0.500 & \cellcolor{user_c3!74!user_c2} 0.798 & \cellcolor{user_c3!72!user_c2} 0.793 \\
5  & \cellcolor{user_c1} 0.500 & \cellcolor{user_c1} 0.500 & \cellcolor{user_c4!91!user_c3} 0.987 & \cellcolor{user_c1} 0.500 & \cellcolor{user_c4!93!user_c3} 0.989 & \cellcolor{user_c4!97!user_c3} 0.995 & \cellcolor{user_c1} 0.500 & \cellcolor{user_c1} 0.500 & \cellcolor{user_c4!59!user_c3} 0.939 & \cellcolor{user_c4!84!user_c3} 0.976 & \cellcolor{user_c4!94!user_c3} 0.991 & \cellcolor{user_c1} 0.500 & \cellcolor{user_c4!5!user_c3} 0.857 & \cellcolor{user_c3!100!user_c2} 0.854 \\
6  & \cellcolor{user_c1} 0.500 & \cellcolor{user_c1} 0.500 & \cellcolor{user_c4!61!user_c3} 0.942 & \cellcolor{user_c1} 0.500 & \cellcolor{user_c4!61!user_c3} 0.942 & \cellcolor{user_c4} 1.000 & \cellcolor{user_c1} 0.500 & \cellcolor{user_c1} 0.500 & \cellcolor{user_c4!95!user_c3} 0.993 & \cellcolor{user_c4!88!user_c3} 0.982 & \cellcolor{user_c4!61!user_c3} 0.942 & \cellcolor{user_c1} 0.500 & \cellcolor{user_c4!43!user_c3} 0.914 & \cellcolor{user_c4!32!user_c3} 0.898 \\
7  & \cellcolor{user_c1} 0.500 & \cellcolor{user_c1} 0.500 & \cellcolor{user_c4!88!user_c3} 0.982 & \cellcolor{user_c1} 0.500 & \cellcolor{user_c4!88!user_c3} 0.982 & \cellcolor{user_c4} 1.000 & \cellcolor{user_c1} 0.500 & \cellcolor{user_c1} 0.500 & \cellcolor{user_c4!97!user_c3} 0.995 & \cellcolor{user_c4!98!user_c3} 0.997 & \cellcolor{user_c4!88!user_c3} 0.982 & \cellcolor{user_c1} 0.500 & \cellcolor{user_c4!59!user_c3} 0.939 & \cellcolor{user_c4!59!user_c3} 0.938 \\
8  & \cellcolor{user_c1} 0.500 & \cellcolor{user_c1} 0.500 & \cellcolor{user_c4!82!user_c3} 0.973 & \cellcolor{user_c1} 0.500 & \cellcolor{user_c4!82!user_c3} 0.973 & \cellcolor{user_c4} 1.000 & \cellcolor{user_c1} 0.500 & \cellcolor{user_c1} 0.500 & \cellcolor{user_c4!93!user_c3} 0.990 & \cellcolor{user_c4!92!user_c3} 0.988 & \cellcolor{user_c4!82!user_c3} 0.973 & \cellcolor{user_c1} 0.500 & \cellcolor{user_c4!25!user_c3} 0.888 & \cellcolor{user_c4!26!user_c3} 0.889 \\
9  & \cellcolor{user_c1} 0.500 & \cellcolor{user_c1} 0.500 & \cellcolor{user_c4!90!user_c3} 0.985 & \cellcolor{user_c1} 0.500 & \cellcolor{user_c4!90!user_c3} 0.985 & \cellcolor{user_c4} 1.000 & \cellcolor{user_c1} 0.500 & \cellcolor{user_c1} 0.500 & \cellcolor{user_c4!98!user_c3} 0.997 & \cellcolor{user_c4!99!user_c3} 0.999 & \cellcolor{user_c4!90!user_c3} 0.985 & \cellcolor{user_c1} 0.500 & \cellcolor{user_c4!61!user_c3} 0.941 & \cellcolor{user_c4!57!user_c3} 0.935 \\
10 & \cellcolor{user_c1} 0.500 & \cellcolor{user_c1} 0.500 & \cellcolor{user_c4!13!user_c3} 0.870 & \cellcolor{user_c1} 0.500 & \cellcolor{user_c4!13!user_c3} 0.870 & \cellcolor{user_c4} 1.000 & \cellcolor{user_c1} 0.500 & \cellcolor{user_c1} 0.500 & \cellcolor{user_c4!97!user_c3} 0.996 & \cellcolor{user_c4!98!user_c3} 0.997 & \cellcolor{user_c4!13!user_c3} 0.870 & \cellcolor{user_c1} 0.500 & \cellcolor{user_c4!16!user_c3} 0.874 & \cellcolor{user_c4!16!user_c3} 0.874 \\
11 & \cellcolor{user_c1} 0.500 & \cellcolor{user_c1} 0.500 & \cellcolor{user_c4!59!user_c3} 0.939 & \cellcolor{user_c1} 0.500 & \cellcolor{user_c4!59!user_c3} 0.939 & \cellcolor{user_c4} 1.000 & \cellcolor{user_c1} 0.500 & \cellcolor{user_c1} 0.500 & \cellcolor{user_c4!98!user_c3} 0.997 & \cellcolor{user_c4!98!user_c3} 0.997 & \cellcolor{user_c4!59!user_c3} 0.939 & \cellcolor{user_c1} 0.500 & \cellcolor{user_c4!25!user_c3} 0.888 & \cellcolor{user_c4!26!user_c3} 0.889 \\
12 & \cellcolor{user_c4!87!user_c3} 0.980 & \cellcolor{user_c4} 1.000 & \cellcolor{user_c4!87!user_c3} 0.980 & \cellcolor{user_c1} 0.500 & \cellcolor{user_c4!87!user_c3} 0.980 & \cellcolor{user_c4} 1.000 & \cellcolor{user_c4} 1.000 & \cellcolor{user_c4} 1.000 & \cellcolor{user_c4} 1.000 & \cellcolor{user_c4} 1.000 & \cellcolor{user_c4!87!user_c3} 0.980 & \cellcolor{user_c1} 0.500 & \cellcolor{user_c3!87!user_c2} 0.824 & \cellcolor{user_c3!87!user_c2} 0.824 \\
13 & \cellcolor{user_c1} 0.500 & \cellcolor{user_c1} 0.500 & \cellcolor{user_c4!90!user_c3} 0.985 & \cellcolor{user_c1} 0.500 & \cellcolor{user_c4!90!user_c3} 0.985 & \cellcolor{user_c4} 1.000 & \cellcolor{user_c1} 0.500 & \cellcolor{user_c1} 0.500 & \cellcolor{user_c4} 1.000 & \cellcolor{user_c4} 1.000 & \cellcolor{user_c4!90!user_c3} 0.985 & \cellcolor{user_c1} 0.500 & \cellcolor{user_c4!37!user_c3} 0.905 & \cellcolor{user_c4!35!user_c3} 0.903 \\
14 & \cellcolor{user_c1} 0.500 & \cellcolor{user_c1} 0.500 & \cellcolor{user_c4!87!user_c3} 0.980 & \cellcolor{user_c1} 0.500 & \cellcolor{user_c4!87!user_c3} 0.980 & \cellcolor{user_c4} 1.000 & \cellcolor{user_c1} 0.500 & \cellcolor{user_c1} 0.500 & \cellcolor{user_c4} 1.000 & \cellcolor{user_c4} 1.000 & \cellcolor{user_c4!87!user_c3} 0.980 & \cellcolor{user_c1} 0.500 & \cellcolor{user_c4!62!user_c3} 0.943 & \cellcolor{user_c4!61!user_c3} 0.942 \\
\midrule
\textbf{Mean} & \cellcolor{user_c2!23!user_c1} 0.534 & \cellcolor{user_c2!24!user_c1} 0.536 & \cellcolor{user_c4!81!user_c3} 0.971 & \cellcolor{user_c1} 0.500 & \cellcolor{user_c4!81!user_c3} 0.971 & \cellcolor{user_c4!99!user_c3} 0.999 & \cellcolor{user_c2!24!user_c1} 0.536 & \cellcolor{user_c2!24!user_c1} 0.536 & \cellcolor{user_c4!91!user_c3} 0.987 & \cellcolor{user_c4!93!user_c3} 0.989 & \cellcolor{user_c4!81!user_c3} 0.972 & \cellcolor{user_c1} 0.500 & \cellcolor{user_c4!33!user_c3} 0.899 & \cellcolor{user_c4!29!user_c3} 0.894\\
\bottomrule
\end{tabular}
\end{tiny}
\end{table}

%% file: tables/big_table_asp_pos_combined.tex
\begin{table}[htpb]
\setlength{\tabcolsep}{4pt}
\caption{Mean Average Precision (AP) of explainers across different whiteboxes evaluated on positive examples. Sub-columns \(r_1\) and \(r_2\) denote the readout functions utilized by the whitebox models (Equation \ref{eqn:readout_1} and Equation \ref{eqn:readout_2}, respectively). Standard errors are omitted for brevity. Cell colors represent a piecewise continuous gradient based on landmark colors: \textbf{Red} (\(\le 0.50\)), \textbf{Pale Red} (\(\approx 0.65\)), \textbf{Light Green} (\(\approx 0.85\)), and \textbf{Green} (\(1.00\)).}
\label{tab:model_eval_ap_continuous_custom_palette}
\centering
\begin{tiny}
% User-requested landmark colors
\definecolor{user_c1}{HTML}{E57373} % Red (Target: <= 0.50)
\definecolor{user_c2}{HTML}{FFCDD2} % Pale Red (Target: ~0.65)
\definecolor{user_c3}{HTML}{C5E1A5} % Light Green (Target: ~0.85)
\definecolor{user_c4}{HTML}{81C784} % Green (Target: 1.00)

\begin{tabular}{c c cccccccccccccc}
\toprule
\multirow{2}{*}{\textbf{No.}} & \multirow{2}{*}{\textbf{Baseline}} & \multicolumn{2}{c}{\textbf{GNN Expl.}} & \multicolumn{2}{c}{\textbf{Input $\times$ Grad.}} & \multicolumn{2}{c}{\textbf{IG}} & \multicolumn{2}{c}{\textbf{PG Expl.}} & \multicolumn{2}{c}{\textbf{SHAP Sampl.}} & \multicolumn{2}{c}{\textbf{Saliency}} & \multicolumn{2}{c}{\textbf{SubgraphX}} \\
\cmidrule(lr){3-4} \cmidrule(lr){5-6} \cmidrule(lr){7-8} \cmidrule(lr){9-10} \cmidrule(lr){11-12} \cmidrule(lr){13-14} \cmidrule(lr){15-16}
& & \(r_1\) & \(r_2\) & \(r_1\) & \(r_2\) & \(r_1\) & \(r_2\) & \(r_1\) & \(r_2\) & \(r_1\) & \(r_2\) & \(r_1\) & \(r_2\) & \(r_1\) & \(r_2\) \\
\midrule
1  & \cellcolor{user_c1} 0.406 & \cellcolor{user_c1} 0.406 & \cellcolor{user_c1} 0.406 & \cellcolor{user_c4!97!user_c3} 0.995 & \cellcolor{user_c1} 0.406 & \cellcolor{user_c4!97!user_c3} 0.995 & \cellcolor{user_c4} 1.000 & \cellcolor{user_c1} 0.406 & \cellcolor{user_c1} 0.406 & \cellcolor{user_c4!91!user_c3} 0.987 & \cellcolor{user_c4!94!user_c3} 0.991 & \cellcolor{user_c4!97!user_c3} 0.995 & \cellcolor{user_c1} 0.406 & \cellcolor{user_c4!24!user_c3} 0.886 & \cellcolor{user_c4!15!user_c3} 0.872 \\
2  & \cellcolor{user_c2!54!user_c1} 0.581 & \cellcolor{user_c2!54!user_c1} 0.581 & \cellcolor{user_c2!54!user_c1} 0.581 & \cellcolor{user_c4!99!user_c3} 0.998 & \cellcolor{user_c2!54!user_c1} 0.581 & \cellcolor{user_c4!99!user_c3} 0.998 & \cellcolor{user_c4} 1.000 & \cellcolor{user_c2!54!user_c1} 0.581 & \cellcolor{user_c2!54!user_c1} 0.581 & \cellcolor{user_c4!86!user_c3} 0.979 & \cellcolor{user_c4!89!user_c3} 0.983 & \cellcolor{user_c4!99!user_c3} 0.998 & \cellcolor{user_c2!54!user_c1} 0.581 & \cellcolor{user_c4!57!user_c3} 0.935 & \cellcolor{user_c4!55!user_c3} 0.933 \\
3  & \cellcolor{user_c3!14!user_c2} 0.678 & \cellcolor{user_c3!14!user_c2} 0.678 & \cellcolor{user_c3!14!user_c2} 0.678 & \cellcolor{user_c4} 1.000 & \cellcolor{user_c3!14!user_c2} 0.678 & \cellcolor{user_c4} 1.000 & \cellcolor{user_c4} 1.000 & \cellcolor{user_c3!14!user_c2} 0.678 & \cellcolor{user_c3!14!user_c2} 0.678 & \cellcolor{user_c4!87!user_c3} 0.980 & \cellcolor{user_c4!85!user_c3} 0.978 & \cellcolor{user_c4} 1.000 & \cellcolor{user_c3!14!user_c2} 0.678 & \cellcolor{user_c4!53!user_c3} 0.929 & \cellcolor{user_c4!41!user_c3} 0.911 \\
4  & \cellcolor{user_c1} 0.462 & \cellcolor{user_c1} 0.462 & \cellcolor{user_c1} 0.462 & \cellcolor{user_c4!87!user_c3} 0.980 & \cellcolor{user_c1} 0.462 & \cellcolor{user_c4!87!user_c3} 0.980 & \cellcolor{user_c4!94!user_c3} 0.991 & \cellcolor{user_c1} 0.462 & \cellcolor{user_c1} 0.462 & \cellcolor{user_c4!85!user_c3} 0.978 & \cellcolor{user_c4!78!user_c3} 0.967 & \cellcolor{user_c4!89!user_c3} 0.984 & \cellcolor{user_c1} 0.462 & \cellcolor{user_c3!19!user_c2} 0.687 & \cellcolor{user_c3!16!user_c2} 0.682 \\
5  & \cellcolor{user_c2!58!user_c1} 0.587 & \cellcolor{user_c2!58!user_c1} 0.587 & \cellcolor{user_c2!58!user_c1} 0.587 & \cellcolor{user_c4!97!user_c3} 0.995 & \cellcolor{user_c2!58!user_c1} 0.587 & \cellcolor{user_c4!97!user_c3} 0.995 & \cellcolor{user_c4} 1.000 & \cellcolor{user_c2!58!user_c1} 0.587 & \cellcolor{user_c2!58!user_c1} 0.587 & \cellcolor{user_c4!81!user_c3} 0.972 & \cellcolor{user_c4!89!user_c3} 0.984 & \cellcolor{user_c4!97!user_c3} 0.995 & \cellcolor{user_c2!58!user_c1} 0.587 & \cellcolor{user_c3!68!user_c2} 0.786 & \cellcolor{user_c3!67!user_c2} 0.783 \\
6  & \cellcolor{user_c1} 0.294 & \cellcolor{user_c1} 0.294 & \cellcolor{user_c1} 0.294 & \cellcolor{user_c4!52!user_c3} 0.928 & \cellcolor{user_c1} 0.294 & \cellcolor{user_c4!52!user_c3} 0.928 & \cellcolor{user_c4} 1.000 & \cellcolor{user_c1} 0.294 & \cellcolor{user_c1} 0.294 & \cellcolor{user_c4!94!user_c3} 0.991 & \cellcolor{user_c4!84!user_c3} 0.976 & \cellcolor{user_c4!52!user_c3} 0.928 & \cellcolor{user_c1} 0.294 & \cellcolor{user_c4!5!user_c3} 0.857 & \cellcolor{user_c3!98!user_c2} 0.846 \\
7  & \cellcolor{user_c1} 0.297 & \cellcolor{user_c1} 0.297 & \cellcolor{user_c1} 0.297 & \cellcolor{user_c4!83!user_c3} 0.975 & \cellcolor{user_c1} 0.297 & \cellcolor{user_c4!83!user_c3} 0.975 & \cellcolor{user_c4} 1.000 & \cellcolor{user_c1} 0.297 & \cellcolor{user_c1} 0.297 & \cellcolor{user_c4!95!user_c3} 0.993 & \cellcolor{user_c4!97!user_c3} 0.996 & \cellcolor{user_c4!83!user_c3} 0.975 & \cellcolor{user_c1} 0.297 & \cellcolor{user_c4!15!user_c3} 0.873 & \cellcolor{user_c4!13!user_c3} 0.870 \\
8  & \cellcolor{user_c1} 0.270 & \cellcolor{user_c1} 0.270 & \cellcolor{user_c1} 0.270 & \cellcolor{user_c4!73!user_c3} 0.959 & \cellcolor{user_c1} 0.270 & \cellcolor{user_c4!73!user_c3} 0.959 & \cellcolor{user_c4} 1.000 & \cellcolor{user_c1} 0.270 & \cellcolor{user_c1} 0.270 & \cellcolor{user_c4!89!user_c3} 0.983 & \cellcolor{user_c4!87!user_c3} 0.981 & \cellcolor{user_c4!73!user_c3} 0.959 & \cellcolor{user_c1} 0.270 & \cellcolor{user_c3!58!user_c2} 0.765 & \cellcolor{user_c3!57!user_c2} 0.764 \\
9  & \cellcolor{user_c1} 0.296 & \cellcolor{user_c1} 0.296 & \cellcolor{user_c1} 0.296 & \cellcolor{user_c4!85!user_c3} 0.977 & \cellcolor{user_c1} 0.296 & \cellcolor{user_c4!85!user_c3} 0.977 & \cellcolor{user_c4} 1.000 & \cellcolor{user_c1} 0.296 & \cellcolor{user_c1} 0.296 & \cellcolor{user_c4!98!user_c3} 0.997 & \cellcolor{user_c4!99!user_c3} 0.999 & \cellcolor{user_c4!85!user_c3} 0.977 & \cellcolor{user_c1} 0.296 & \cellcolor{user_c4!15!user_c3} 0.872 & \cellcolor{user_c4!13!user_c3} 0.869 \\
10 & \cellcolor{user_c1} 0.266 & \cellcolor{user_c1} 0.266 & \cellcolor{user_c1} 0.266 & \cellcolor{user_c3!94!user_c2} 0.837 & \cellcolor{user_c1} 0.266 & \cellcolor{user_c3!94!user_c2} 0.837 & \cellcolor{user_c4} 1.000 & \cellcolor{user_c1} 0.266 & \cellcolor{user_c1} 0.266 & \cellcolor{user_c4!96!user_c3} 0.994 & \cellcolor{user_c4!97!user_c3} 0.996 & \cellcolor{user_c3!94!user_c2} 0.837 & \cellcolor{user_c1} 0.266 & \cellcolor{user_c3!82!user_c2} 0.814 & \cellcolor{user_c3!82!user_c2} 0.814 \\
11 & \cellcolor{user_c1} 0.210 & \cellcolor{user_c1} 0.210 & \cellcolor{user_c1} 0.210 & \cellcolor{user_c4!43!user_c3} 0.915 & \cellcolor{user_c1} 0.210 & \cellcolor{user_c4!43!user_c3} 0.915 & \cellcolor{user_c4} 1.000 & \cellcolor{user_c1} 0.210 & \cellcolor{user_c1} 0.210 & \cellcolor{user_c4!97!user_c3} 0.996 & \cellcolor{user_c4!97!user_c3} 0.996 & \cellcolor{user_c4!43!user_c3} 0.915 & \cellcolor{user_c1} 0.210 & \cellcolor{user_c3!65!user_c2} 0.780 & \cellcolor{user_c3!67!user_c2} 0.784 \\
12 & \cellcolor{user_c1} 0.145 & \cellcolor{user_c4!78!user_c3} 0.967 & \cellcolor{user_c4} 1.000 & \cellcolor{user_c4!78!user_c3} 0.967 & \cellcolor{user_c1} 0.145 & \cellcolor{user_c4!78!user_c3} 0.967 & \cellcolor{user_c4} 1.000 & \cellcolor{user_c4} 1.000 & \cellcolor{user_c4} 1.000 & \cellcolor{user_c4} 1.000 & \cellcolor{user_c4} 1.000 & \cellcolor{user_c4!78!user_c3} 0.967 & \cellcolor{user_c1} 0.145 & \cellcolor{user_c3!4!user_c2} 0.657 & \cellcolor{user_c3!4!user_c2} 0.658 \\
13 & \cellcolor{user_c1} 0.136 & \cellcolor{user_c1} 0.136 & \cellcolor{user_c1} 0.136 & \cellcolor{user_c4!86!user_c3} 0.979 & \cellcolor{user_c1} 0.136 & \cellcolor{user_c4!86!user_c3} 0.979 & \cellcolor{user_c4} 1.000 & \cellcolor{user_c1} 0.136 & \cellcolor{user_c1} 0.136 & \cellcolor{user_c4!99!user_c3} 0.999 & \cellcolor{user_c4} 1.000 & \cellcolor{user_c4!86!user_c3} 0.979 & \cellcolor{user_c1} 0.136 & \cellcolor{user_c3!35!user_c2} 0.720 & \cellcolor{user_c3!34!user_c2} 0.717 \\
14 & \cellcolor{user_c1} 0.173 & \cellcolor{user_c1} 0.173 & \cellcolor{user_c1} 0.173 & \cellcolor{user_c4!79!user_c3} 0.969 & \cellcolor{user_c1} 0.173 & \cellcolor{user_c4!79!user_c3} 0.969 & \cellcolor{user_c4} 1.000 & \cellcolor{user_c1} 0.173 & \cellcolor{user_c1} 0.173 & \cellcolor{user_c4} 1.000 & \cellcolor{user_c4} 1.000 & \cellcolor{user_c4!79!user_c3} 0.969 & \cellcolor{user_c1} 0.173 & \cellcolor{user_c4!12!user_c3} 0.868 & \cellcolor{user_c4!11!user_c3} 0.867 \\
\midrule
\textbf{Mean} & \cellcolor{user_c1} 0.343 & \cellcolor{user_c1} 0.402 & \cellcolor{user_c1} 0.404 & \cellcolor{user_c4!75!user_c3} 0.962 & \cellcolor{user_c1} 0.343 & \cellcolor{user_c4!75!user_c3} 0.962 & \cellcolor{user_c4!99!user_c3} 0.999 & \cellcolor{user_c1} 0.404 & \cellcolor{user_c1} 0.404 & \cellcolor{user_c4!93!user_c3} 0.989 & \cellcolor{user_c4!93!user_c3} 0.989 & \cellcolor{user_c4!75!user_c3} 0.963 & \cellcolor{user_c1} 0.343 & \cellcolor{user_c3!83!user_c2} 0.816 & \cellcolor{user_c3!81!user_c2} 0.812 \\
\bottomrule
\end{tabular}
\end{tiny}
\end{table}

%% file: appendix_additional_info.tex
\section{Additional dataset information}

\subsection{Patterns}
\label{app:pattern_stats}

\paragraph{Pattern statistics.} The KRFP fingerprint of~\cite{klekota2008chemical} contains 4,860 patterns, drawn from substructures enriched in biologically active compounds. 1,655 of them (34.1\%) don't contain rings, and therefore are eligible for inclusion. Each whitebox requires manual weight construction and validation against more than five million negatives. This construction and validation effort limited our selection to 14 motifs. We did not sample the admissible set for diversity; we took patterns that qualified and varied their size deliberately. This coverage is not complete: no motif contains an alcohol, a primary amine, a nitrile or a urea. Ring-containing motifs are excluded by construction. %Commonness of the motifs varies: Pattern 3 matches 24 molecules in the filtered corpus and Pattern 5 matches 34, while Pattern 10 is present in more than half of it (Table \ref{fig:app_prevalence}).

%\begin{table}
%\centering
%\caption{Prevalence of the 14 WOMBAT motifs in the filtered PubChem corpus. For each pattern, we report the number and percentage of molecules matching the relaxed SMARTS formulation.} %\label{fig:app_prevalence}
%\begin{tabular}{ccc}
%\toprule
%Pattern No. & Number of positives & \% \\
%\toprule

%1 & 7,674 & 0.0073 \\
%2 & 248 & 0.0002 \\ 
%3 & 24 & \(<\)0.0001 \\ 
%4 & 1,150 & 0.0011 \\ 
%5 & 34 & \(<\)0.0001 \\ 
%6 & 13,520,614 & 12.7743 \\ 
%7 & 887,033 & 0.8381 \\ 
%8 & 3,380,223 & 3.1936 \\ 
%9 & 660,305 & 0.6239 \\ 
%10 & 59,713,182 & 56.4172 \\ 
%11 & 22,246,330 & 21.0184 \\ 
%12 & 1,972,142 & 1.8633 \\ 
%13 & 286,324 & 0.2705 \\ 
%14 & 894,777 & 0.8454 \\ 
% \bottomrule
%\end{tabular}
%\end{table}

\begin{table}
\centering
\caption{Fractions of molecules for which the original KRFP SMARTS and our relaxed SMARTS disagree either in motif detection or in the identified atom set.} \label{tab:app_disagreements}
\begin{tabular}{cccc}
\toprule
Pattern No. & Disagreements & Total & \% \\
\toprule

1 &	1 &	7,674 &	0.01 \\
2 &	184 &	248 & 74.19 \\
3 &	0 &	24 & 0.00 \\ 
4 &	2 &1,146 & 0.17 \\
5 &	10 &	34 & 29.41 \\
6 &	591 &	10,000 & 5.91 \\
7 &	0 &	10,000 & 0.00 \\
8 &	106 & 10,000 & 1.06 \\
9 &	5 &10,000 &	0.05 \\
10 &	0 &	10,000 & 0.00 \\
11 &	0 &	10,000 & 0.00 \\
12 &	0 &	10,000 & 0.00 \\
13 &	0 &	10,000 & 0.00 \\
14 &	0 &	10,000 & 0.00 \\
\midrule 
Total &	899 & 99,126 &	0.91 \\
 \bottomrule
\end{tabular}
\end{table}

\paragraph{Pattern SMARTS.} \label{par:smarts_comp} Here, we present KRFP SMARTS and their relaxed versions. Note that some cases do not require relaxation, as in simple cases node identity requirements can be fulfilled by neighbour counting. This relaxation allows our whiteboxes to be fully reliable, while recognizing complex patterns. See Table \ref{tab:app_disagreements} for analysis of disagreements between the original and relaxed SMARTS.

{
\NewDocumentCommand{\smarts}{v}{%
  \texttt{#1}%
}

\begin{tiny}
    \begin{enumerate}
    % Pattern 1
    \item \smarts{[!#1][CH]([NH]C(=O)[!#1])C(Cl)(Cl)Cl} \(\rightarrow\) \smarts{[CHD3]([NH][CD3](=O))C(Cl)(Cl)Cl} 
    % Pattern 2 
    \item \smarts{[!#1][CH]([!#1])C(=O)[CH2]C([CH3])([CH3])[CH2]C([!#1])([!#1])[!#1]} \(\rightarrow\) \smarts{[CHD3]C(=O)[CH2]C([CH3])([CH3])[CH2][CD4]} 

    % Pattern 3
    \item \smarts{[!#1][CH]([!#1])C(=O)[NH][NH]C(=O)[CH2][CH2][CH2][CH2][CH2][CH2][CH3]} \(\rightarrow\) \smarts{[CHD3]C(=O)[N;H1,H2][N;H1,H2]C(=O)[CH2][CH2][CH2][CH2][CH2][CH2][CH3]}

    % Pattern 4
    \item \smarts{[!#1][CH]([!#1])C(=O)O[CH2][CH2]N([CH2][CH3])[CH2][CH3]} \(\rightarrow\) \smarts{[CHD3]C(=O)O[CH2][CH2]N([CH2][CH3])[CH2][CH3]}  

    % Pattern 5 
    \item \smarts{[!#1][CH]([CH2]C(=N[NH]C(=O)[!#1])[!#1])C(=O)[OH]} \(\rightarrow\) \smarts{[CHD3]([CH2][CD3](=N[N;H1,H2][CD3](=O)))C(=O)[OH]}   

    % Pattern 6
    \item \smarts{[!#1][CH]([!#1])[CH]([!#1])[!#1]} \(\rightarrow\) \smarts{[CHD3][CHD3]}  

    % Pattern 7
    \item \smarts{[!#1][CH]([!#1])[CH2][CH2][CH3]} \(\rightarrow\) \smarts{[CHD3][CH2][CH2][CH3]} 

    % Pattern 8 
    \item \smarts{[!#1][CH]([!#1])[CH2]C(=O)[!#1]} \(\rightarrow\) \smarts{[CHD3][CH2][CD3](=O)} 

    % Pattern 9
    \item \smarts{[!#1][CH]([!#1])S(=O)(=O)[!#1]} \(\rightarrow\) \smarts{[CHD3][SD4](=O)(=O)}

    % Pattern 10
    \item \smarts{[!#1][CH]([!#1])[!#1]} \(\rightarrow\) \smarts{[!#1][CH]([!#1])[!#1]} 

    % Pattern 11
    \item \smarts{[!#1][CH]([!#1])[CH3]} \(\rightarrow\) \smarts{[!#1][CH]([!#1])[CH3]} 

    % Pattern 12
    \item \smarts{[!#1][CH]=O} \(\rightarrow\) \smarts{[!#1][CH]=O} 

    % Pattern 13
    \item \smarts{[!#1][N+]([CH3])([CH3])[CH3]} \(\rightarrow\) \smarts{[!#1][N+]([CH3])([CH3])[CH3]} 

    % Pattern 14 
    \item \smarts{[!#1][CH2]Cl} \(\rightarrow\) \smarts{[!#1][CH2]Cl} 
\end{enumerate}

\end{tiny}
}

\subsection{Validation molecules}

In this section, we present additional statistics and details regarding our validation datasets. Figure \ref{fig:tversky_big} illustrates the structural similarity distributions relative to ECFP4 motif fingerprints, while Table \ref{tab:ds_sizes} details the validation dataset sizes for each motif.

\input{figures/tanimoto_tversky_comp_big}
\input{tables/validation_ds_sizes}

%% file: figures/tanimoto_tversky_comp_big.tex
\begin{figure}
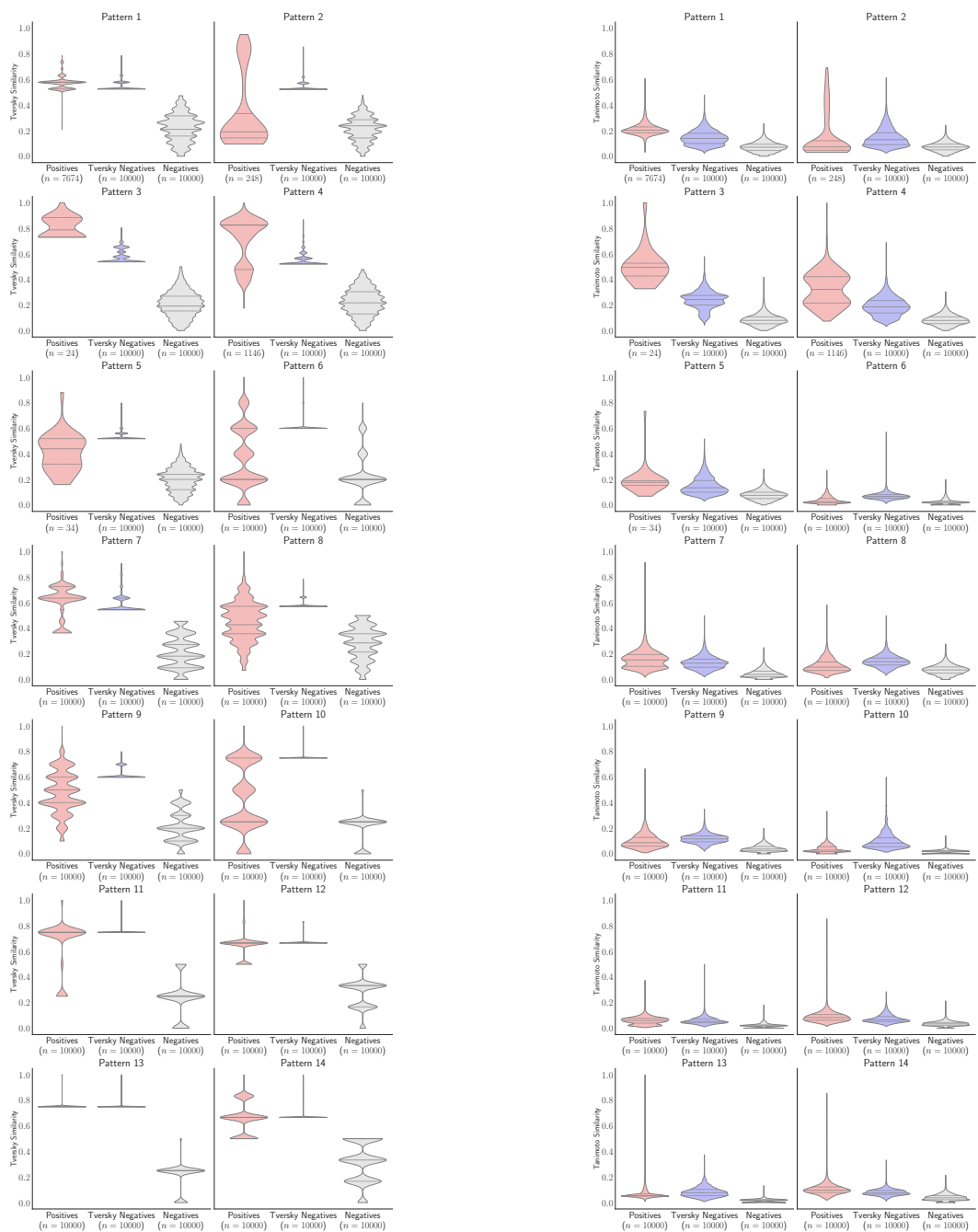

    \centering
    \begin{minipage}{0.4\linewidth}
    \includegraphics[width=\linewidth]{files/big_distributions/tversky_similarity_distributions.pdf}
    \end{minipage}%
    \hfill
    \begin{minipage}{0.4\linewidth}
    \includegraphics[width=\linewidth]{files/big_distributions/tanimoto_similarity_distributions.pdf}
    \end{minipage}%
    \caption{Structural similarity distributions relative to ECFP4 motif fingerprints. Similarity scores (Tversky and Tanimoto) for molecules used to all of our patterns.}
    \label{fig:tversky_big} 
\end{figure}

%% file: tables/validation_ds_sizes.tex
\begin{table}[t]
\scriptsize
\centering
\caption{In this table, we showcase the sizes of the validation sets for each of our patterns. Molecule counts under the ''All molecules'' field show the factual size of our validation dataset that was prepared by us by filtering PubChem. Unfortunately, a small subset of SMILES strings was corrupted during preprocessing; as such, the ''Valid molecules'' field shows on how many molecules each of the whiteboxes was actually tested. Note that there is very little difference between the two subsets; however, we've chosen to compare these numbers for transparency.}
\label{tab:ds_sizes}
\begin{tabular}{ccccccc}
\toprule
\multirow{2}{*}{\textbf{Pattern No.}} & \multicolumn{3}{c}{\textbf{All molecules}} & \multicolumn{3}{c}{\textbf{Valid molecules}}  \\
\cmidrule(lr){2-4} \cmidrule(lr){5-7}
& \textbf{Positives} & \textbf{Tversky} & \textbf{Negatives} & \textbf{Positives} & \textbf{Tversky} & \textbf{Negatives} \\
\toprule 
1 & 7,674 & 491,447 & 4,976,465 & 7,674 & 491,447 & 4,976,452 \\
2 & 248 & 118,672 & 4,994,354 & 248 & 118,672 & 4,994,352 \\
3 & 24 & 1,568,016 & 4,925,632 & 24 & 1,568,016 & 4,925,626 \\
4 & 1,146 & 569,455 & 4,972,881 & 1,146 & 569,455 & 4,972,877 \\
5 & 34 & 70,609 & 4,996,581 & 34 & 70,609 & 4,996,571 \\
6 & 5,000,000 & 5,000,000 & 4,729,159 & 4,999,997 & 4,999,997 & 4,729,155 \\
7 & 886,073 & 4,346,189 & 4,792,735 & 886,073 & 4,346,189 & 4,792,729 \\
8 & 3,376,760 & 2,866,888 & 4,859,902 & 3,376,759 & 2,866,887 & 4,859,894 \\
9 & 659,042 & 1,685,436 & 4,919,363 & 659,042 & 1,685,436 & 4,919,355 \\
10 & 5,000,000 & 13,491 & 4,998,512 & 4,999,998 & 13,491  & 4,998,501 \\
11 & 5,000,000 & 241,794 & 4,985,482 & 4,999,995 & 241,794 & 4,985,478 \\
12 & 1,970,537 & 717,424 & 4,965,399 & 1,970,527 & 717,424 & 4,965,389 \\
13 & 286,098 & 127,085 & 4,993,947 & 286,098 & 127,085 & 4,993,942 \\
14 & 894,517 & 3,950,564 & 4,812,155 & 894,517 & 3,950,564 & 4,812,149 \\
 \bottomrule   
\end{tabular}    

\end{table}